\documentclass[11pt]{article}
\usepackage{amsmath}
\usepackage{amssymb}
\usepackage{mathrsfs}
\usepackage{algorithm,algorithmic}
\usepackage{amsthm}
\usepackage{xcolor}
\definecolor{mydarkblue}{rgb}{0,0.08,0.45}
\usepackage[hidelinks,colorlinks]{hyperref}
\hypersetup{
	colorlinks=true,
	linkcolor=mydarkblue,
	citecolor=mydarkblue,
	filecolor=mydarkblue,
	urlcolor=mydarkblue,
}
\usepackage[margin=1in]{geometry}
\usepackage{graphicx}
\usepackage[outdir=./]{epstopdf}
\usepackage{caption}
\usepackage[framemethod=TikZ]{mdframed}
\usepackage{lipsum}
\usepackage{subcaption}
\usepackage{multirow}
\usepackage{framed}
\usepackage{color}
\usepackage{thmtools,thm-restate}
\usepackage{booktabs}
\usepackage[round]{natbib}
\setcitestyle{authoryear,round,citesep={;},aysep={,},yysep={;}}
\mdfdefinestyle{MyFrame}{%
	linecolor=black,
	outerlinewidth=1pt,
	roundcorner=20pt,
	innertopmargin=\baselineskip,
	innerbottommargin=\baselineskip,
	innerrightmargin=20pt,
	innerleftmargin=20pt,
	backgroundcolor=white!50!white}

\newtheorem{lemma}{Lemma}

\theoremstyle{definition}

\newtheorem*{theorem*}{Theorem}

\numberwithin{equation}{section}
\DeclareMathOperator*{\argmin}{arg\,min}
\DeclareMathOperator*{\argmax}{arg\,max}
\let\Pr\relax
\newcommand*{\Pr}{\mathbb{P}}
\newcommand*{\Oh}{{O}}
\newcommand*{\h}{h}
\newcommand*{\Real}{\mathbb{R}}

\newcommand{\polylog}{\text{polylog}}
\newcommand{\normsq}[1]{\parallel #1 \parallel_2^2}
\newcommand{\norm}[1]{{\parallel #1 \parallel_{2}}}

\newcommand{\lpnorm}[2]{\parallel #1 \parallel_{#2}}
\newcommand*{\E}{\mathbb{E}}

\renewcommand{\cite}[1]{\citep{#1}}
\title{Orthogonalized ALS: A Theoretically Principled Tensor Decomposition Algorithm for Practical Use}

\author{
  Vatsal Sharan \\
  Stanford University \\
  {\tt vsharan@stanford.edu}
\And
	Gregory Valiant \\
  Stanford University \\
  {\tt gvaliant@stanford.edu}
}

\begin{document}
	

%\maketitle
\newcommand{\theTitle}{Guaranteed Tensor Decomposition via Orthogonalized Alternating Least Squares}
%\author{Moses Charikar, Jacob Steinhardt, and Gregory Valiant}
\author{
 \fontsize{11}{13}\selectfont {\bf Vatsal Sharan} \\
%\fontsize{11}{13}\selectfont  Computer Science Department \\
 \fontsize{11}{13}\selectfont Stanford University \\
 \fontsize{11}{13}\selectfont {\tt vsharan@stanford.edu}
\and
\fontsize{11}{13}\selectfont {\bf Gregory Valiant} \\
%\fontsize{11}{13}\selectfont  Computer Science Department \\
\fontsize{11}{13}\selectfont Stanford University \\
\fontsize{11}{13}\selectfont {\tt valiant@stanford.edu}
}

\title{\theTitle}
\date{}

\clearpage
\maketitle

\begin{abstract}
The popular Alternating Least Squares (ALS) algorithm for tensor decomposition is efficient and easy to implement, but often converges to poor local optima---particularly when the weights of the factors are non-uniform. We propose a modification of the ALS approach that is as efficient as standard ALS, but provably recovers the true factors with random initialization under standard incoherence assumptions on the factors of the tensor. We demonstrate the significant practical superiority of our approach over traditional ALS for a variety of tasks on synthetic data---including tensor factorization on exact, noisy and over-complete tensors, as well as tensor completion---and for computing word embeddings from a third-order word tri-occurrence tensor.
\end{abstract}

 % These provable guarantees also hold for the practically relevant regime where the factor weights are highly non-uniform. 
\thispagestyle{empty}
%\newpage
%\setcounter{page}{1}

%\input{abstract}
\section{Introduction}

From a theoretical perspective, tensor methods have become an incredibly useful and versatile tool for learning a wide array of popular models, including topic modeling \cite{anandkumar2012spectral}, mixtures of Gaussians \cite{ge2015learning}, community detection \cite{anandkumar2014atensor}, learning graphical models with guarantees via the method of moments \cite{anandkumar2014btensor,chaganty2014estimating} and reinforcement learning \cite{azizzadenesheli2016reinforcement}. The key property of tensors that enables these applications is that tensors have a unique decomposition (decomposition here refers to the most commonly used CANDECOMP/PARAFAC or CP decomposition), under mild conditions on the factor matrices \cite{kruskal1977three}; for example, tensors have a unique decomposition whenever the factor matrices are full rank. As tensor methods naturally model three-way (or higher-order) relationships, it is not too optimistic to hope that their practical utility will only increase, with the rise of multi-modal measurements (e.g. measurements taken by ``Internet of Things'' devices) and the numerous practical applications involving high order dependencies, such as those encountered in natural language processing or genomic settings. In fact, we are already seeing exciting applications of tensor methods for analysis of high-order spatiotemporal data \cite{yu2016learning}, health data analysis \cite{wang2015rubik} and bioinformatics \cite{colombo2015fastmotif}.   Nevertheless, to truly realize the practical impact that the current theory of tensor methods portends, we require better algorithms for computing decompositions---practically efficient algorithms that are both capable of scaling to large (and possibly sparse) tensors, and are robust to noise and deviations from the idealized ``low-rank'' assumptions.

%In this paper we focus on the most commonly used CANDECOMP/PARAFAC (CP) decomposition for tensors, which is the tensor analog of matrix factorization. 

%In this work, we introduce a new algorithm for computing tensor decompositions; this approach has strong theoretical guarantees, and seems to significantly outperform the most commonly used existing approaches in practice on both real and synthetic data, for a number of tasks related to tensors decomposition.  We are optimistic that our approach may lead to 

As tensor decomposition is NP-Hard in the worst-case \cite{shitov2016hard,hillar2013most,haastad1990tensor}, one cannot hope for algorithms which always produce the correct factorization.
%but to harvest the benefits of the big data revolution we do need computationally efficient tensor decomposition algorithms which work well in practice. 
Despite this worst-case impossibility, accurate decompositions can be efficiently computed in many practical settings. Early work from the 1970's~\cite{leurgans1993decomposition,harshman1970foundations} established a simple algorithm for computing the tensor decomposition (in the noiseless setting) provided that the factor matrices are full rank.   This approach, based on an eigendecomposition, is very sensitive to noise in the tensor (as we also show in our experiments), and does not scale well for large, sparse tensors.   

Since this early work, much progress has been made.  Nevertheless,  many of the tensor decomposition algorithms hitherto proposed and employed have strong provable success guarantees but are computationally expensive (though still polynomial time)---either requiring an expensive initialization phase, being unable to leverage the sparsity of the input tensor, or not being efficiently parallelizable. On the other hand, there are also approaches which are  efficient to implement, but which fail to compute an accurate decomposition in many natural settings. The Alternating Least Squares (ALS) algorithm (either with random initialization or more complicated initializations) falls in this latter category and is, by far, the most widely employed decomposition algorithm despite its often poor performance and propensity for getting stuck in local optima (which we demonstrate on both synthetic data and real NLP data).

In this paper we propose an alternative decomposition algorithm, ``Orthogonalized Alternating Least Squares'' (Orth-ALS) which has strong theoretical guarantees, and seems to significantly outperform the most commonly used existing approaches in practice on both real and synthetic data, for a number of tasks related to tensor decomposition. %that is both extremely simple and efficient, as well as has strong provable recovery guarantees.  this approach has strong theoretical guarantees, and seems to significantly outperform the most commonly used existing approaches in practice on both real and synthetic data, for a number of tasks related to tensors decomposition.  We are optimistic that our approach may lead to 
 This algorithm is a simple modification of the ALS algorithm to periodically ``orthogonalize'' the estimates of the factors.  Intuitively, this periodic orthogonalization prevents multiple recovered factors from ``chasing after'' the same true factors, allowing for the avoidance of local optima and more rapid convergence to the true factors.  

From the practical side, our algorithm enjoys all the benefits of standard ALS, namely simplicity and computational efficiency/scalability, particularly for very large yet sparse tensors, and noise robustness.  Additionally, the speed of convergence and quality of the recovered factors are \emph{substantially} better than standard ALS, even when ALS is initialized using the more expensive SVD initialization. As we show, on synthetic low-rank tensors, our algorithm consistently recovers the true factors, while standard ALS often falters in local optima and fails both in recovering the true factors and in recovering an accurate low-rank approximation to the original tensor. We also applied Orth-ALS to a  large 3-tensor of word co-occurrences to compute ``word embeddings''.\footnote{Word embeddings are vector representations of words, which can then be used as features for higher-level machine learning.  Word embeddings have rapidly become the backbone of many downstream natural language processing tasks (see e.g.~\cite{mikolov2013distributed}).} The embedding produced by our Orth-ALS algorithm is \emph{significantly} better than that produced by standard ALS, as we quantify via a near 30\% better performance of the resulting word embeddings across standard NLP datasets that test the ability of the embeddings to answer basic analogy tasks (i.e. ``puppy is to dog as kitten is to $\rule{0.4cm}{0.15mm}$?'') and semantic word-similarity tasks.   Together, these results support our optimism that with better decomposition algorithms, tensor methods will become an indispensable, widely-used data analysis tool in the near future. 
%tensor methods may achieve the practical success that their theory portends.  

Beyond the practical benefits of Orth-ALS, we also consider its theoretical properties.  We show that Orth-ALS provably recovers all factors under \emph{random} initialization for worst-case tensors as long as the tensor satisfies an {incoherence} property (which translates to the factors of the tensors having small correlation with each other), which is satisfied by random tensors with rank $k=o(d^{0.25})$ where $d$ is the dimension of the tensor.  This requirement that $k=o(d^{0.25})$ is significantly worse than the best known provable recovery guarantees for polynomial-time algorithms on random tensors---the recent work \citet{ma2016polynomial} succeeds even in the over-complete setting with $k=o(d^{1.5})$.  Nevertheless, our experiments support our belief that this shortcoming is more a property of our analysis than the algorithm itself.  Additionally, for many practical settings, particularly natural language tasks, the rank of the recovered tensor is typically significantly sublinear in the dimensionality of the space, and the benefits of an extremely efficient and simple algorithm might outweigh limitations on the required rank for provable recovery.

Finally, as a consequence of our analysis technique for proving convergence of Orth-ALS, we also improve the known guarantees for another popular tensor decomposition algorithm---the {tensor power method}. We show that the tensor power method with \emph{random} initialization converges to one of the factors with small residual error for rank $k=o(d)$, where $d$ is the dimension. We also show that the convergence rate is \emph{quadratic} in the dimension.  \citet{anandkumar2014guaranteed} had previously shown local convergence of the tensor power method with a linear convergence rate (and also showed global convergence via a SVD-based initialization scheme, obtaining the first guarantees for the tensor power method in non-orthogonal settings).  Our new results, particularly global convergence from random initialization, provide some deeper insights into the behavior of this popular algorithm.

The rest of the paper is organized as follows-- in Section \ref{sec:back} we discuss related work, describe the ALS algorithm and tensor power method, and discuss the shortcomings of both algorithms, particularly for tensors with non-uniform factor weights. Section \ref{sec:notation} states the notation. Section \ref{sec:orthals} introduces and motivates Orth-ALS, and states the convergence guarantees. We state our convergence results for the tensor power method in Section \ref{sec:tpm}. The experimental results, on both synthetic data and the NLP tasks are discussed in Section \ref{sec:exp}. In Section \ref{sec:orthogonal_tensors} we illustrate our proof techniques for the special case of orthogonal tensors. Proof details have been deferred to the Appendix. Our code is available at \url{http://web.stanford.edu/~vsharan/orth-als.html}. 
\vspace{-10pt}

\section{Background and Related Work}\label{sec:back}

We begin the section with a brief discussion of related work on tensor decomposition. We then review the ALS algorithm and the tensor power method and discuss their basic properties. Our proposed tensor decomposition algorithm, Orth-ALS, builds on these algorithms.  %We conclude the section with a brief discussion of other related work on tensor decomposition. %and the work from the theory community on tensor decomposition algorithms  with provable guarantees---including the exciting recent algorithms that can be applied to the ``over-complete'' setting, where $k > d$.  
\vspace{-10pt}
\subsection{Related Work on Tensor Decomposition}

Though it is not possible for us to do justice to the substantial body of work on tensor decomposition, we will review three families of algorithms which are distinct from alternating minimization approaches such as ALS and the tensor power method. Many algorithms have been proposed for guaranteed decomposition of \emph{orthogonal} tensors, we refer the reader to \citet{anandkumar2014btensor,kolda2011shifted,comon2009tensor,zhang2001rank}. However, obtaining guaranteed recovery of {non-orthogonal} tensors using algorithms for orthogonal tensors requires converting the tensor into an orthogonal form (known as \emph{whitening}) which is ill conditioned in high dimensions \cite{le2011ica,souloumiac2009joint}, and is computationally the most expensive step \cite{huang2013fast}. Another very interesting line of work on tensor decompositions is to use simultaneous diagonalization and higher order SVD \cite{colombo2016tensor,kuleshov2015tensor,de2006link} but these are not as computationally efficient as alternating minimization\footnote{\citet{de2006link} prove unique recovery under very general conditions, but their algorithm is quite complex and requires solving a linear system of size $\Oh(d^4)$, which is prohibitive for large tensors. We ran the simultaneous diagonalization algorithm of \citet{kuleshov2015tensor} on a dimension 100, rank 30 tensor; and the algorithm needed around 30 minutes to run, whereas Orth-ALS converges in less than 5 seconds.}. Recently, there has been intriguing work on provably decomposing random tensors using the \emph{sum-of-squares} approach \cite{ma2016polynomial,hopkins2016fast,tang2015guaranteed,ge2015decomposing}. \citet{ma2016polynomial} show that a sum-of-squares based relaxation can decompose highly overcomplete random tensors of rank up to $o(d^{1.5})$. Though these results establish the polynomial learnability of the problem, they are unfortunately not practical.

Very recently, there has been exciting work on scalable tensor decomposition algorithms using ideas such as sketching \cite{song2016sublinear,wang2015fast} and contraction of tensor problems to matrix problems \cite{shah2015sparse}. Also worth noting are recent approaches to speedup ALS via sampling and randomized least squares \cite{battaglino2017practical,cheng2016spals,papalexakis2012parcube}.
\vspace{-10pt}
\subsection{Alternating Least Squares (ALS)}

ALS is the most widely used algorithm for tensor decomposition and has been described as the ``workhorse'' for tensor decomposition \cite{kolda2009tensor}. The algorithm is conceptually very simple:  if the goal is to
recover a rank-$k$ tensor, ALS maintains a rank-$k$ decomposition specified by three sets of $d\times k$ dimensional matrices $\{\hat{A}, \hat{B}, \hat{C}\}$ corresponding to the three modes of the tensor.  ALS will iteratively fix two of the three modes, say $\hat{A}$ and $\hat{B}$, and then update $\hat{C}$ by solving a least-squared regression problem to find the best approximation to the underlying tensor $T$ having factors $\hat{A}$ and $\hat{B}$ in the first two modes, namely $\hat{C}_{new} = \argmin_{C'} \|T - \hat{A} \otimes \hat{B} \otimes \hat{C}'\|_2.$   ALS will then continue to iteratively fix two of the three modes, and update the other mode via solving the associated least-squares regression problem. 
%fixing two of the modes, the optimization problem of finding the value of the third mode that minimizes the squared error of the resulting tensor can be expressed as a linear least-squares regression problem (and hence can be efficiently solved).  As its name suggests, ALS iteratively fixes two of the three modes, and solves the least squares problem on the remaining mode.  
These updates continue until  some stopping condition is satisfied---typically when the squared error of the approximation is no longer decreasing, or when a fixed number of iterations have elapsed. The factors used in ALS are either chosen uniformly at random, or via a more expensive initialization scheme such as SVD based initialization \cite{anandkumar2014guaranteed}. In the SVD based scheme, the factors are initialized to be the singular vectors of a random projection of the tensor onto a matrix.
%The initial factors used in ALS are usually chosen uniformly at random; as we discuss below, more sophisticated and time-intensive initializations have also been suggested~\cite{de2006link}. %We note that in some implementations the newly computed mode is used immediately when computing the next least-squared update, as opposed to implementations (such as Algorithm \ref{als}) where the modes are only updated after all three updated modes have been computed.   %In this work we'll be focusing on ALS with random initialization, as it's more efficient and is commonly used in practice.% More sophisticated initialization schemes have been proposed via simultaneous diagonalization and higher-order SVD \cite{de2006link}. However, this requires solving a linear system of size $O(d^4)$, where $d$ is the dimension of the tensor, and hence is too expensive for large tensors.  

The main advantages of the ALS approach, which have led to its widespread use in practice are its conceptual simplicity, noise robustness and computational efficiency given its graceful handling of sparse tensors and ease of parallelization. There are several publicly available optimized packages implementing ALS, such as \citet{kossaifi2016tensorly,tensorlab3.0,TTB_Software,TTB_Sparse,splattsoftware,huang2014distributed,kang2012gigatensor}.

Despite the advantages, ALS does not have any global convergence guarantees and can get stuck in local optima \cite{comon2009tensor,kolda2009tensor}, even under very realistic settings. For example, consider a setting where the weights $w_i$ for the factors $\{A_i,B_i,C_i\}$ decay according to a power-law, hence the first few factors have much larger weight than the others. As we show in the experiments (see Fig. \ref{fig:random_rec}), ALS fails to recover the low-weight factors. Intuitively, this is because multiple recovered factors will be chasing after the \emph{same} high weight factor, leading to a bad local optima.  %Such decaying weights are common in practical settings,  is indeed the case for the tensor of word co-occurrences derived from the Wikipedia corpus (see the Experiments section).

%\begin{algorithm}\caption{Standard ALS for Tensor Decomposition}\label{als}
%	\begin{flushleft}
%		\textbf{Input:} Tensor $T\in \Real^{d\times d\times d}$, number of iterations $N$.
%	\end{flushleft}
%	\begin{algorithmic}[1]
%		\STATE Initialize each column of $\hat{A}, \hat{B}$ and $\hat{C} \in \Real^{d\times k}$ uniformly from the unit sphere
%		\FOR {$t=1:N$}
%		\STATE $X = T_{(1)}(\hat{C} \odot \hat{B})({\hat{C}}^T \hat{C} \hat{B}^T \hat{B})^{-1}$
%		\STATE $Y = T_{(2)}(\hat{C} \odot \hat{A})(\hat{C}^T \hat{C} \hat{A}^T \hat{A})^{-1}$
%		\STATE $Z = T_{(3)}(\hat{B} \odot \hat{A})(\hat{B}^T \hat{B} \hat{A}^T \hat{A})^{-1}$
%		\STATE Normalize $X,Y,Z$ and store results in $\hat{A},\hat{B},\hat{C}$
%%		\STATE Normalize columns of $X$ and store result in $\hat{A}$.	
%%		\STATE Normalize columns of $Y$ and store result in $\hat{B}$.
%%		\STATE Normalize columns of $Z$ and store result in $\hat{C}$.			
%		\ENDFOR
%		\STATE Estimate weights $\hat{w}_i  = T(\hat{A}_i , \hat{B}_i, \hat{C}_i ), \forall\; i \in [k]$.
%		\STATE \textbf{return} $\hat{A}, \hat{B}, \hat{C}, \hat{w}$
%	\end{algorithmic}
%\end{algorithm}
\vspace{-10pt}
\subsection{Tensor Power Method}

The tensor power method is a special case of ALS that only computes a rank-1 approximation. The procedure is then repeated multiple times to recover different factors. The factors recovered in different iterations of the algorithm are then clustered to determine the set of unique factors.  Different initialization strategies have been proposed for the tensor power method. \citet{anandkumar2014guaranteed} showed that the tensor power method converges locally (i.e. for a suitably chosen initialization) for random tensors with rank $o(d^{1.5})$. They also showed that a SVD based initialization strategy gives good starting points and used this to prove global convergence for random tensors with rank $O(d)$. However, the SVD based initialization strategy can be computationally expensive, and our experiments suggest that even SVD initialization fails in the setting where the weights decay according to a power-law (see Fig. \ref{fig:random_rec}).

In this work, we prove global convergence guarantees with random initializations for the tensor power method for random and worst-case incoherent tensors. Our results also demonstrate how, with random initialization, the tensor power method converges to the factor having the largest product of weight times the correlation of the factor with the random initialization vector. This explains the difficulty of using random initialization to recover factors with small weight.  For example, if one factor has weight less than a $1/c$ fraction of the weight of, say, the heaviest $k/2$ factors, then with high probability we would require at least $k^{\Theta(c^2)}$ random initializations to recover this factor. This is because the correlation between random vectors in high dimensions is approximately distributed as a Normal random variable and if $k/2+1$ samples are drawn from the standard Normal distribution, the probability that one particular sample is at least a factor of $c$ larger than the other $k/2$ other samples scales as roughly  $k^{-\Theta(c^2)}$.

\section{Notation}\label{sec:notation}

\vspace{-.1cm}We state our algorithm and results for 3rd order tensors, and  believe the algorithm and analysis techniques should extend easily to higher dimensions. Given a 3rd order tensor $T\in \Real^{d\times d \times d}$ our task is to decompose the tensor into its factor matrices $A,B$ and $C$:
$
T = \sum_{i\in [k]}^{}w_i A_i \otimes B_i \otimes C_i, %\label{tensor_def}
$
where $A_i$ denotes the $i$th column of a matrix $A$. Here $w_i \in \Real, A_i, B_i, C_i \in \Real^d$ and $\otimes$ denotes the tensor product: if $a,b,c \in \mathbb{R}^d$ then $a \otimes b \otimes c \in \mathbb{R}^{d\times d \times d}$ and $(a \otimes b \otimes c)_{ijk}=a_i b_j c_k$. We will refer to $w_i$ as the weight of the factor $\{A_i, B_i, C_i\}$. This is also known as CP decomposition. We refer to the dimension of the tensor by $d$ and denote its rank by $k$. We refer to different dimensions of a tensor as the {modes} of the tensor.

We denote $T_{(n)}$ as the mode $n$ matricization of the tensor, which is the flattening of the tensor along the $n$th direction obtained by stacking all the matrix slices together. For example $T_{(1)}$ denotes flattening of a tensor $T\in \mathbb{R}^{n \times m \times p}$ to a $(n \times mp)$ matrix. We denote the Khatri-Rao product of two matrices $A$ and $B$ as $(A \odot B )_i = (A_i \otimes B_i)_{(1)}$, where $(A_i \otimes B_i)_{(1)}$ denotes the flattening of the matrix $A_i \otimes B_i$ into a row vector. For any tensor $T$ and vectors $a,b,c$, we also define $T(a,b,c)= \sum_{i,j,k}^{}T_{ijk}a_i b_j c_k$. Throughout, we say $f(n)=\tilde{\Oh}(g(n))$ if $f(n)=O(g(n))$ up to poly-logarithmic factors. %We denote the set $\{1,2,\dotsb,k\}=[k]$

%We also define $T(I,I,x)$ as the projection of a vector $x$ by the tensor along the third dimension. This equals $T(I,I,x)_{uv} = \sum_{w=1}^{d}T_{uvw}x_w$. 

Though all algorithms in the paper extend to asymmetric tensors, we prove convergence results under the symmetric setting where $A=B=C$. Similar to other works \cite{tang2015guaranteed,anandkumar2014guaranteed,ma2016polynomial}, our guarantees depend on the {incoherence} of the factor matrices ($c_{\max}$), defined to be the maximum correlation in absolute value between any two factors, i.e. ${c_{\max} = \max_{i \ne j} |A_i^T A_j|}$. This serves as a natural assumption to simplify the problem as it is NP-Hard in the worst case. Also, tensors with randomly drawn factors satisfy $c_{\max}\le \tilde{O}(1/\sqrt{d})$, and our results hold for such tensors. %Tensor decomposition becomes more difficult as the factors become more coherent; the easiest case is when $c_{\max}=0$ and all the factors are orthogonal to each other (this is known as orthogonal tensor decomposition). %We also define %$\eta=\max\{ c_{\max},1/d\}$ and $\gamma=\frac{w_{\max}}{w_{\min}}$ as the ratio between the largest and smallest factor weights.

\vspace{-10pt}
\section{The Algorithm: Orthogonalized Alternating Least Squares (Orth-ALS)}\label{sec:orthals}

%\begin{comment}
%\begin{figure}
%	\centering
%	\includegraphics[height = 2.5 in]{orth_als_chart}
%	\caption{Orthogonalized ALS (Orth-ALS). The orthogonalization block in gray is the only difference from vanilla ALS}
%	\label{orth_als_chart}
%\end{figure}
%\end{comment}

In this section we introduce Orth-ALS, which combines the computational benefits of standard ALS and the provable recovery of the tensor power method, while avoiding the difficulties faced by both when factors have different weights. Orth-ALS is a simple modification of standard ALS that adds an orthogonalization step before each set of ALS steps. We describe the algorithm in Algorithm \ref{orth}. Note that steps 4-6 are just the solution to the least squares problem expressed in compact tensor notation, for instance step 4 can be equivalently stated as $X=\argmin_{C'} \|T - \hat{A} \otimes \hat{B} \otimes \hat{C}'\|_2$. Similarly, step 9 is the least squares estimate of the weight $w_i$ of each rank-1 component $\hat{A}_i\otimes \hat{B}_i \otimes \hat{C}_i$.

%\begin{algorithm}[h]\caption{Orthogonalized ALS (Orth-ALS) for tensor decomposition}\label{orth}
%	\textbf{Input:} Tensor $T\in \Real^{d\times d\times d}$, number of iterations $N$.
%	\begin{algorithmic}[1]
%		\State Initialize each column of $\hat{A}, \hat{B}$ and $\hat{C} \in \Real^{d\times k}$ uniformly from the unit sphere
%		\For {$t=1:N$}
%		\State Find $QR$ decomposition of $\hat{A}$ and set $\hat{A}$ to be the $Q$ matrix
%		\State Find $QR$ decomposition of $\hat{B}$ and set $\hat{B}$ to be the $Q$ matrix
%		\State Find $QR$ decomposition of $\hat{C}$ and set $\hat{C}$ to be the $Q$ matrix
%		\State $X = T_{(1)}(\hat{C} \odot \hat{B})({\hat{C}}^T \hat{C} \hat{B}^T \hat{B})^{-1}$
%		\State $Y = T_{(1)}(\hat{C} \odot \hat{A})(\hat{C}^T \hat{C} \hat{A}^T \hat{A})^{-1}$
%		\State $Z = T_{(1)}(\hat{B} \odot \hat{A})(\hat{B}^T \hat{B} \hat{A}^T \hat{A})^{-1}$
%		\State Normalize columns of $X$ and store result in $\hat{A}$.	
%		\State Normalize columns of $Y$ and store result in $\hat{B}$.
%		\State Normalize columns of $Z$ and store result in $\hat{C}$.			
%		\EndFor
%		\State Estimate weights $\hat{w}_i  = T(\hat{A}_i , \hat{B}_i, \hat{C}_i ), \forall i \in [k]$.
%		\State \textbf{return} $\hat{A}, \hat{B}, \hat{C}, \hat{w}$
%	\end{algorithmic}
%\end{algorithm}

\begin{algorithm}[h]\caption{Orthogonalized ALS (Orth-ALS)}\label{orth}
	\begin{flushleft}
		\textbf{Input:} Tensor $T\in \Real^{d\times d\times d}$, number of iterations $N$.
	\end{flushleft}
	\begin{algorithmic}[1]
		\STATE Initialize each column of $\hat{A}, \hat{B}$ and $\hat{C} \in \Real^{d\times k}$ uniformly from the unit sphere
		\FOR {$t=1:N$}
		\STATE Find $QR$ decomposition of $\hat{A}$, set $\hat{A}=Q.$  Orthogonalize $\hat{B}$ and $\hat{C}$ analogously.
%		\STATE Find $QR$ decomposition of $\hat{B}$, set $\hat{B}=Q$
%		\STATE Find $QR$ decomposition of $\hat{C}$, set $\hat{C}=Q$
		\STATE $X = T_{(1)}(\hat{C} \odot \hat{B})$
		\STATE $Y = T_{(2)}(\hat{C} \odot \hat{A})$
		\STATE $Z = T_{(3)}(\hat{B} \odot \hat{A})$
		\STATE Normalize $X,Y,Z$ and store results in $\hat{A},\hat{B},\hat{C}$
		%\STATE Normalize columns of $X$ and store result in $\hat{A}$.	
		%\STATE Normalize columns of $Y$ and store result in $\hat{B}$.
		%\STATE Normalize columns of $Z$ and store result in $\hat{C}$.			
		\ENDFOR
		\STATE Estimate weights $\hat{w}_i  = T(\hat{A}_i , \hat{B}_i, \hat{C}_i ), \forall\; i \in [k]$.
		\STATE \textbf{return} $\hat{A}, \hat{B}, \hat{C}, \hat{w}$
	\end{algorithmic}
\end{algorithm}

To get some intuition for why the orthogonalization makes sense, let us consider the more intuitive \emph{matrix} factorization problem, where the goal is to compute the eigenvectors of a matrix. Subspace iteration is a straightforward extension of the {matrix} power method to recover all eigenvectors at once. In subspace iteration, the matrix of eigenvector estimates is orthogonalized before each power method step (by projecting the second eigenvector estimate orthogonal to the first one and so on), because otherwise all the vectors would converge to the dominant eigenvector. For the case of tensors, the vectors would not all necessarily converge to the dominant factor if the initialization is good, but with high probability a random initialization would drive many factors towards the larger weight factors. The orthogonalization step is a natural modification which forces the estimates to converge to different factors, even if some factors are much larger than the others.  It is worth stressing that the orthogonalization step does \emph{not} force the final recovered factors to be orthogonal (because the ALS step follows the orthogonalization step) and in general the factors output will not be orthogonal (which is essential for accurately recovering the factors).

From a computational perspective, adding the orthogonalization step does not add to the computational cost as the least squares updates in step 4-6 of Algorithm \ref{orth} involve an extra pseudoinverse term for standard ALS, which evaluates to identity for Orth-ALS and does not have to be computed. The cost of orthogonalization is $\Oh(k^2d)$, while the cost of computing the pseudoinverse is also $\Oh(k^2d)$. We also observe significant speedups in terms of the number of iterations required for convergence for Orth-ALS as compared to standard ALS in our simulations on random tensors (see the experiments in Section~\ref{sec:exp}).%The overall time complexity of Orth-ALS is $\Oh(kd^3)$ for dense tensors.
 %$({\hat{C}}^T \hat{C} \hat{B}^T \hat{B})^{-1}$ term in the ALS algorithm now evaluates to the identity and does not have to be computed. The cost of orthogonalization is $\Oh(dk^2)$, while the cost of computing $({\hat{C}}^T \hat{C} \hat{B}^T \hat{B})^{-1}$ is also $\Oh(dk^2+k^3)$. The overall time complexity is $\Oh(kd^3)$.
\vspace{-10pt}
\paragraph{Variants of Orthogonalized ALS. }Several other modifications to the simple orthogonalization step also seem natural.  Particularly for low-dimensional settings, in practice we found that it is useful to carry out orthogonalization for a few steps and then continue with standard ALS updates until convergence (we call this variant \emph{Hybrid-ALS}). Hybrid-ALS also gracefully reverts to standard ALS in settings where the factors are highly correlated and orthogonalization is not helpful. Our advice to practitioners would be try Hybrid-ALS first before the fully orthogonalized Orth-ALS, and then tune the number of steps for which orthogonalization takes place to get the best results.

%In such settings, it may also useful to \emph{partially} orthogonalize the factors instead of computing a fully orthogonal basis. Hybrid-ALS in particular seems very robust, as it switches to standard ALS updates after some steps and the few extra orthogonalization steps never hurt convergence.
\vspace{-8pt}
\subsection{Performance Guarantees}

We now state the formal guarantees on the performance of Orthogonalized ALS.  The specific variant of Orthogonalized ALS that our theorems apply to is a slight modification of Algorithm~\ref{orth}, and differs in that there is a periodic (every $\log k$ steps) re-randomization of the factors for which our analysis has not yet guaranteed convergence.   In our practical implementations, we observe that \emph{all} factors seem to converge within this first $\log k$ steps, and hence the subsequent re-randomization is unnecessary. 

\begin{restatable}{theorem}{orthalsconvergence}
	\label{orth_als_convergence}
	Consider a $d$-dimensional rank $k$ tensor $T=\sum_{i=1}^{k}w_i A_i \otimes A_i \otimes A_i$. Let $c_{\max}=\max_{i \ne j} |A_i^T A_j|$ be the incoherence between the true factors and $\gamma=\frac{w_{\max}}{w_{\min}}$ be the ratio of the largest and smallest weight.  Assume $\gamma c_{\max} \le o(k^{-2})$, and the estimates of the factors are initialized randomly from the unit sphere.  Provided that, at the $i (\log k+\log\log d)$th step of the algorithm the estimates for all but the first $i$ factors are re-randomized, then with high probability  the orthogonalized ALS updates converge to the true factors in $\Oh( k(\log k +\log \log d))$ steps, and the error at convergence satisfies (up to relabelling) $\normsq{A_i - \hat{A}_i} \le \Oh(\gamma k\max\{c_{\max}^2,1/d^2\})$ and $|1- \frac{\hat{w}_i}{w_i}| \le  \Oh(\max\{c_{\max},1/d\}),$ for all $i$.
\end{restatable}

Theorem \ref{orth_als_convergence} immediately gives convergence guarantees for random low rank tensors. For random $d$ dimensional tensors, $c_{\max} = \Oh(1/\sqrt{d})$; therefore Orth-ALS converges globally with random initialization whenever $k=o({d}^{0.25})$. If the tensor has rank much smaller than the dimension, then our analysis can tolerate significantly higher correlation between the factors. In the Appendix, we also prove Theorem \ref{orth_als_convergence} for the special and easy case of orthogonal tensors, which nevertheless highlights the key proof ideas.
 
%As \citet{anandkumar2014guaranteed} point out, alternating methods such as ALS, Orth-ALS and the tensor power method have a small residual error even after convergence as the true global minimum is not a fixed point for the updates. \citet{anandkumar2014guaranteed} also describe a coordinate descent based algorithm to remove this residual error. In practice though, this residual error term does not seem to be significant.

%\textbf{Remark: } It is natural to ask whether the orthgonalization needs to be done at every step, or if it is sufficient to only orthogonalize the random initialization. The orthogonalized tensor power method (Orth-TPM, Algorithm \ref{tpm_orth}) projects the initial random estimate orthogonal to the previously extracted factors, in contrast to orthogonalizing at each step. However, experiments suggest that this milder form of orthogonalization is not as effective as Orth-ALS when the distribution of the weights becomes highly skewed (see Section~\ref{sec:exp}).
	
\vspace{-8pt}
\subsection{New Guarantees for the Tensor Power Method}\label{sec:tpm}

As a consequence of our analysis of the orthogonalized ALS algorithm, we also prove new guarantees on the tensor power method. As these may be of independent interest because of the wide use of the tensor power method, we summarize them in this section. We show a quadratic rate of convergence (in $\Oh(\log \log d)$ steps) with random initialization for random tensors having rank $k=o(d)$. This contrasts with the analysis of \citet{anandkumar2014guaranteed} who showed a linear rate of convergence ($\Oh(\log d)$ steps) for random tensors, provided an SVD based initialization is employed.% We perform a new analysis of the tensor power method updates and show global convergence of the updates with random initialization.  %Theorem \ref{global_convergence} shows that the tensor power method converges with random initialization as long as the maximum correlation between any two factors is smaller than $\tilde{O}(1/k)$.

\begin{restatable}{theorem}{randomtensor}
	\label{random_tensor}
	Consider a $d$-dimensional rank $k$ tensor $T=\sum_{i=1}^{k}w_i A_i \otimes A_i \otimes A_i$ with the factors $A_i$ sampled uniformly from the $d$-dimensional sphere. Define $\gamma=\frac{w_{\max}}{w_{\min}}$ to be the ratio of the largest and smallest weight. Assume $ k  \le o(d)$ and $\gamma \le \polylog(d)$. If the initialization $x_0\in \Real^d$ is chosen uniformly from the unit sphere, then with high probability the tensor power method updates converge to one of the true factors (say $A_1$) in $\Oh(\log \log d )$ steps, and the error at convergence satisfies $\norm{A_1 - \hat{A}_{1}} \le \tilde{\Oh}(1/\sqrt{d})$. Also, the estimate of the weight $\hat{w}_1$ satisfies  $|1- \frac{\hat{w}_1}{w_1}| \le \tilde{\Oh}(1/\sqrt{d})$.
\end{restatable}

Theorem \ref{random_tensor} provides guarantees for random tensors, but it is natural to ask if there are deterministic conditions on the tensors which guarantee global convergence of the tensor power method. Our analysis also allows us to obtain a clean characterization for global convergence of the tensor power method updates for worst-case tensors in terms of the {incoherence} of the factor matrix---

\begin{restatable}{theorem}{globalconvergence}
	\label{global_convergence}
	Consider a $d$-dimensional rank $k$ tensor $T=\sum_{i=1}^{k}w_i A_i \otimes A_i \otimes A_i$. Let $c_{\max}=\max_{i \ne j} |A_i^T A_j|$ and  $\gamma=\frac{w_{\max}}{w_{\min}}$ be the ratio of the largest and smallest weight, and assume $\gamma c_{\max} \le o(k^{-2})$. If the initialization  $x_0\in \Real^d$ is chosen uniformly from the unit sphere, then with high probability the tensor power method updates converge to one of the true factors (say $A_1$) in $O(\log k +\log\log d )$ steps, and the error at convergence satisfies ${\normsq{A_1 - \hat{A}_1}} \le O(\gamma k\max\{c_{\max}^2,1/d^2\})$ and $|1- \frac{\hat{w}_1}{w_1}| \le  O(\max\{c_{\max},1/d\})$.
\end{restatable}
\vspace{-8pt}
\section{Experiments}\label{sec:exp}

%\vspace{-.2cm}
We compare the performance of Orth-ALS, standard ALS (with random and SVD initialization), the tensor power method, and the classical eigendecomposition approach, through experiments on low rank tensor recovery in a few different parameter regimes, on a overcomplete tensor decomposition task and a tensor completion task. We also compare the factorization of Orth-ALS and standard ALS on  a large real-world tensor of word tri-occurrence based on the 1.5 billion word English Wikipedia corpus.\footnote{MATLAB, Python and C code for Orth-ALS and Hybrid-ALS is available at \url{http://web.stanford.edu/~vsharan/orth-als.html}}

\subsection{Experiments on Random Tensors}
 \noindent \textbf{Recovering low rank tensors:} We explore the abilities of Orth-ALS, standard ALS, and the tensor power method (TPM), to recover a low rank (rank $k$) tensor that has been constructed by independently drawing each of the $k$ factors independently and uniformly at random from the $d$ dimensional unit spherical shell. We consider several different combinations of the dimension, $d$, and rank, $k$.  We also consider both the setting where all of the factors are equally weighted, as well as the practically relevant setting where the factor weights
 decay geometrically, and consider the setting where independent Gaussian noise has been added to the low-rank tensor.

In addition to random initialization for standard ALS and the TPM, we also explore SVD based initialization \cite{anandkumar2014guaranteed} where the factors are initialized via SVD of a projection of the tensor onto a matrix. We also test the classical technique for tensor decomposition via simultaneous diagonalization \cite{leurgans1993decomposition,harshman1970foundations} (also known as Jennrich's algorithm, we refer to it as Sim-Diag), which first performs two random projections of the tensor, and then recovers the factors by an eigenvalue decomposition of the projected matrices. This gives guaranteed recovery when the tensors are noiseless and factors are linearly independent, but is extremely unstable to perturbations.

We evaluate the performance in two respects: 1) the ability of the algorithms to recover a low-rank tensor that is close to the input tensor, and 2) the ability of the algorithms to recover accurate approximations of many of the true factors.  Fig. \ref{fig:mse_iter} depicts the performance via the first metric.  We evaluate the performance in terms of the discrepancy between the input low-rank tensor, and the low-rank tensor recovered by the algorithms, quantified via the ratio of the Frobenius norm of the residual, to the Frobenius norm of the actual tensor: $\frac{\lpnorm{T-\hat{T}}{\text{F}}}{\lpnorm{T}{\text{F}}}$, where $\hat{T}$ is the recovered tensor. Since the true tensor has rank $k$, the inability of an algorithm to drive this error to zero indicates the presence of local optima.  Fig. \ref{fig:mse_iter} depicts the performance of Orth-ALS, standard ALS with random initialization and the hybrid algorithm that performs Orth-ALS for the first five iterations before reverting to standard ALS (Hybrid-ALS). Tests are conducted in both the setting where factor weights are uniform, as well as a geometric spacing, where the ratio of the largest factor weight to the smallest is 100. Fig. \ref{fig:mse_iter} shows that Hybrid ALS and Orth-ALS have much faster convergence and find a significantly better fit than standard ALS.

Fig. \ref{fig:random_rec} quantifies the performance of the algorithms in terms of the number of the original factors that the algorithms accurately recover.
%\footnote{We recognize a factor $\{A_i,B_i,C_i\}$ of the tensor $T$ as being ``recovered'' if there exists at least one recovered factor $\{\hat{A}_j, \hat{B}_j, \hat{C}_j\}$ having correlation at least 0.9 with the factor $\{A_i,B_i,C_i\}$, i.e. there exists $j$ s.t. $|A_i^T\hat{A}_j|>0.9, |B_i^T\hat{B}_j|>0.9$ and $ |C_i^T\hat{C}_j|>0.9$.} 
We use standard ALS, Orth-ALS (Algorithm \ref{orth}), Hybrid-ALS, TPM with random initialization (TPM), ALS with SVD initialization (ALS-SVD), TPM with SVD initialization (TPM-SVD) and the simultaneous diagonalization approach (Sim-Diag). We run TPM and SVD-TPM with 100 different initializations and find a rank $k=30$ decomposition for ALS, ALS-SVD, Orth-ALS, Hybrid-ALS and Sim-Diag. We repeat the experiment (by sampling a new tensor) 10 times. % The stopping criterion for the updates of all four algorithms is when the algorithm completes 30 iterations or if the fit changes by less than $10^{-4}$ between two iterations (we observed that having a smaller threshold and using more iterations did not improve performance). 
We perform this evaluation in both the setting where we receive an actual low-rank tensor as input, as well as the setting where each entry $T_{ijk}$ of the low-rank tensor has been perturbed by independent Gaussian noise of standard deviation equal to $0.05T_{ijk}.$ We can see that Orth-ALS and Hybrid-ALS perform significantly better than the other algorithms and are able to recover all factors in the noiseless case even when the weights are highly skewed. Note that the reason the Hybrid-ALS and Orth-ALS fail to recover all factors in the noisy case when the weights are highly skewed is that the magnitude of the noise essentially swamps the contribution from the smallest weight factors. \\%s greater than the smaller factors in this case, as we add noise proportional to the tensor entries. 

\begin{figure}[]
	\centering
	\begin{subfigure}[t]{0.4\textwidth}
		\centering
		\includegraphics[width=1.75 in]{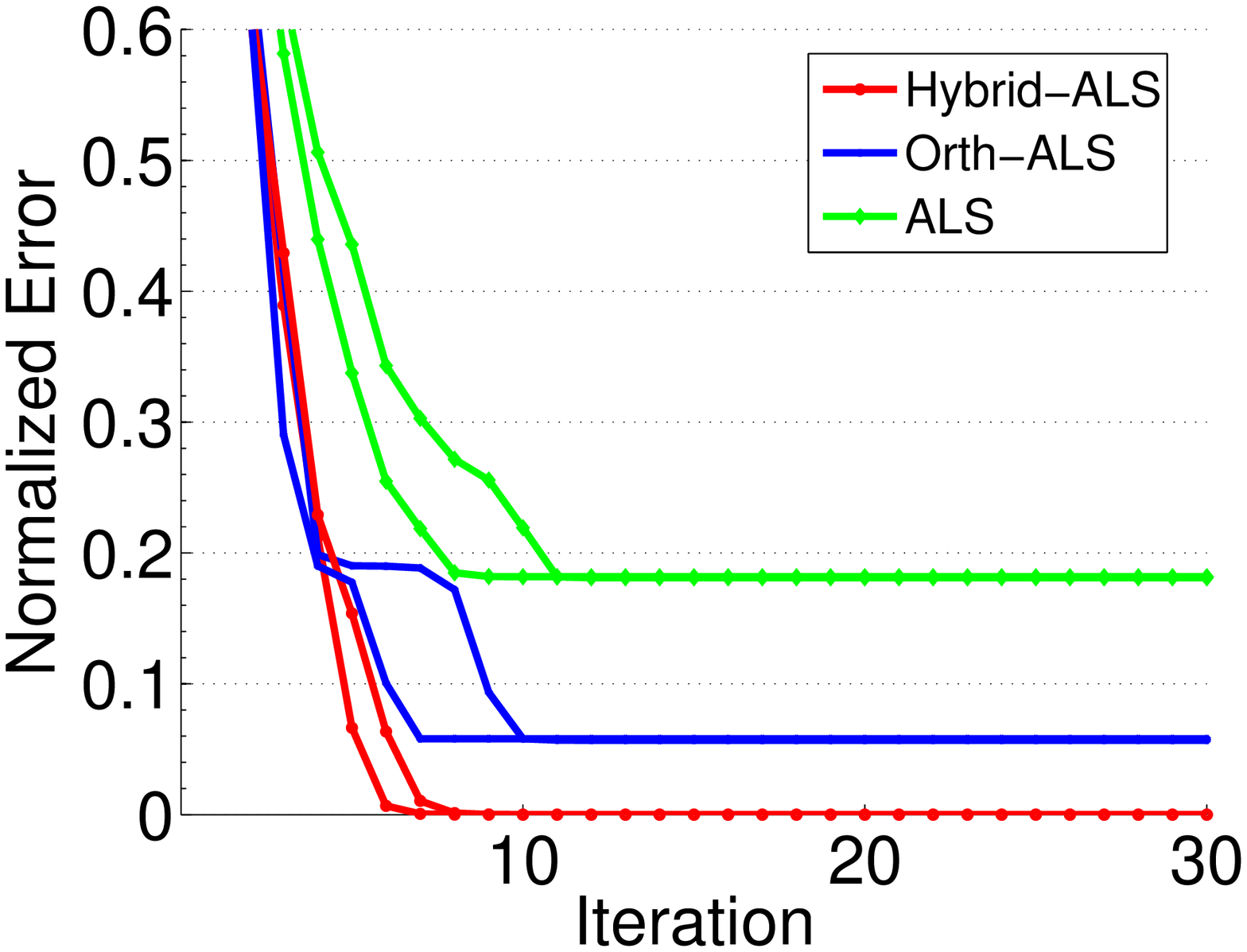}
		\subcaption{$k=30, d=100$, \\
			\hspace*{3mm} uniform weights}
		\label{mse_small_r1}
	\end{subfigure}
	~
	\begin{subfigure}[t]{0.4\textwidth}
		\centering
		\includegraphics[width=1.75 in]{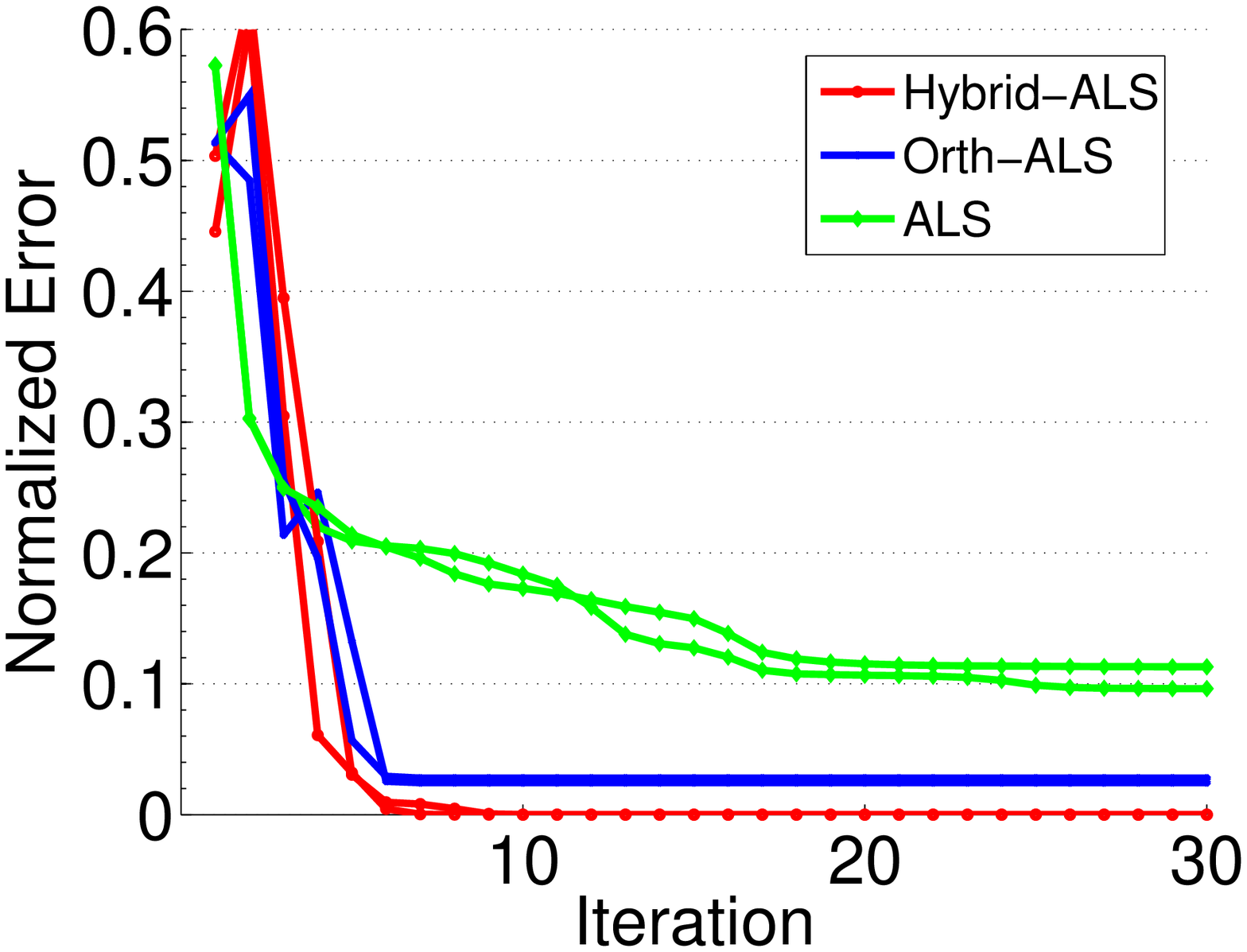}
		\subcaption{$k=30, d=100$,\\ \hspace*{3mm}$\frac{w_{\max}}{w_{\min}}=100$}
		\label{mse_small_r3}
	\end{subfigure}
	~\\
	\begin{subfigure}[t]{0.4\textwidth}
		\centering
		\includegraphics[width=1.75 in]{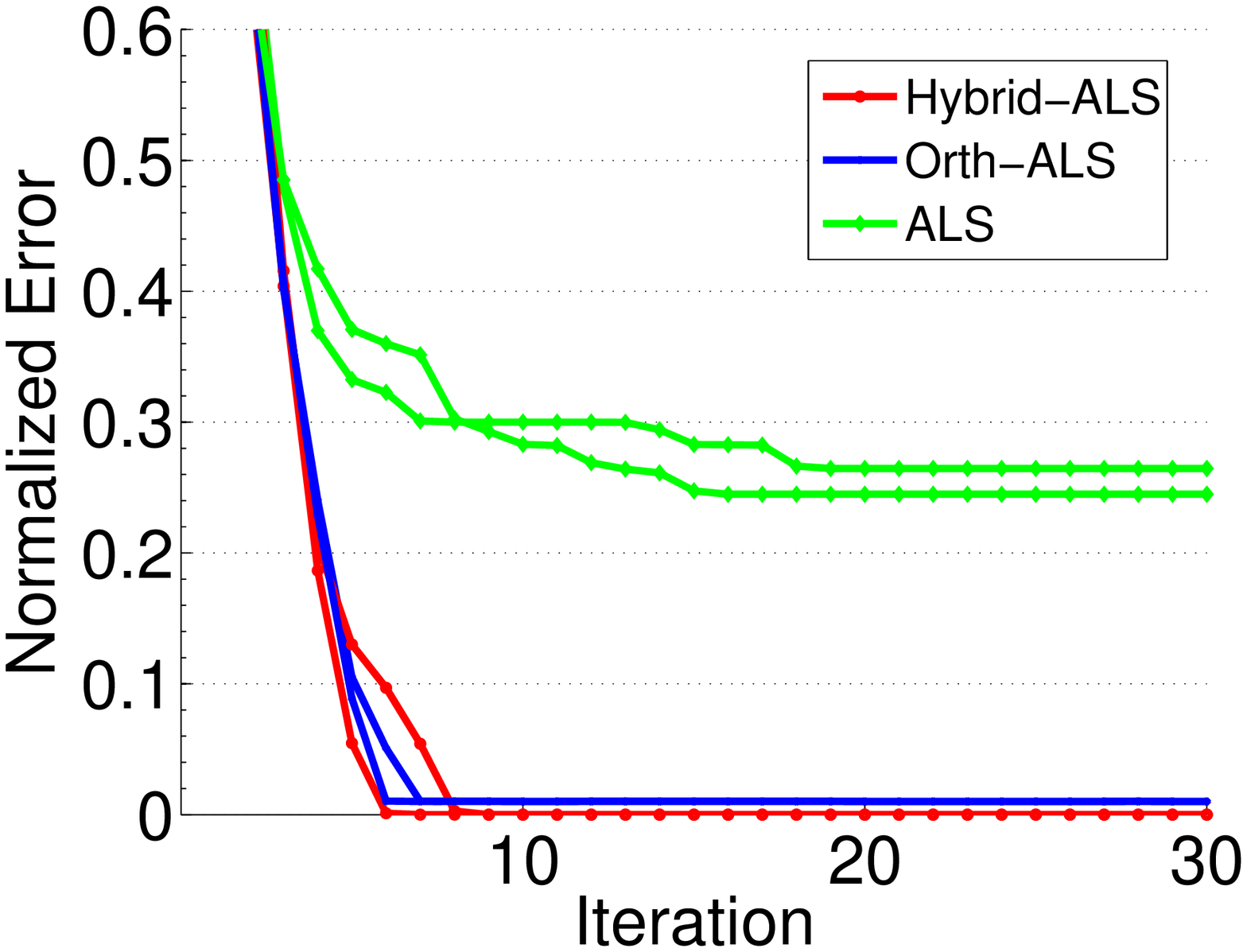}
		\subcaption{ $k=100, d=1000$,\\ 
			\hspace*{3mm}uniform weights}
		\label{mse_large_r1}
	\end{subfigure}
	~
	\begin{subfigure}[t]{0.4\textwidth}
		\centering
		\includegraphics[width=1.75 in]{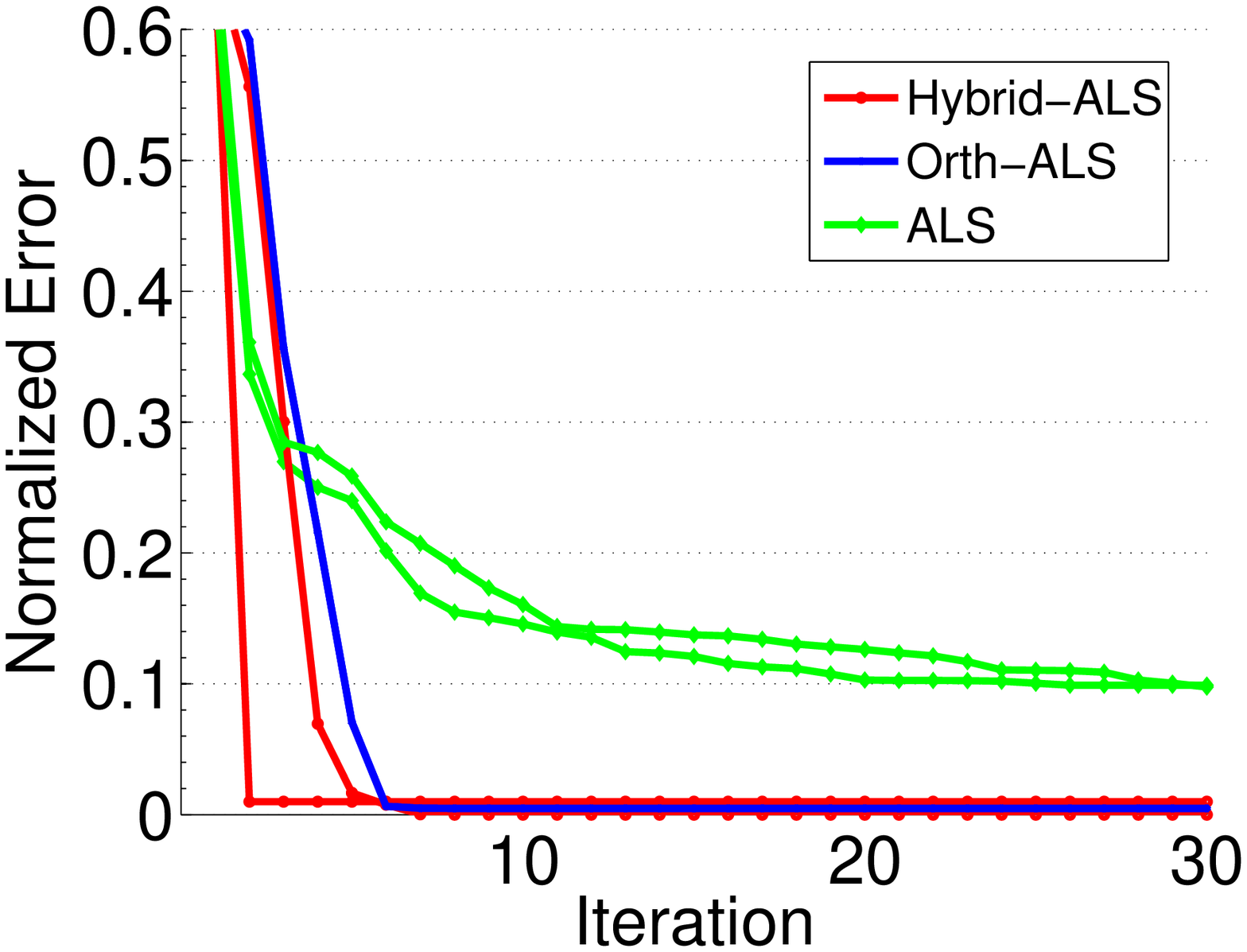}
		\subcaption{$k=100, d=1000$,\\ \hspace*{3mm}$\frac{w_{\max}}{w_{\min}}=100$}
		\label{mse_large_r3}
	\end{subfigure}
	\caption{Plot of the normalized discrepancy between the recovered rank $k$ tensor $\hat{T}$ and the true tensor $T$:  $\frac{\lpnorm{T-\hat{T}}{\text{F}}}{\lpnorm{T}{\text{F}}},$ as a function of the number of iterations.  In all four settings, the Orth-ALS and the hybrid algorithms drive this discrepancy nearly to zero, with the performance of Orth-ALS improving for the higher dimensional cases, whereas the standard ALS algorithm has slower convergence and gets stuck in bad local optima.}
	\label{fig:mse_iter}
	%			\begin{subfigure}[t]{0.4\textwidth}
	%				\centering
	%				\includegraphics[width=2.5 in]{iterations}
	%				\subcaption{Distribution of the number of iterations across all the experiments}
	%				\label{iterations}
	%			\end{subfigure}
\end{figure}
\begin{figure*}[]
	\centering
	\begin{subfigure}[t]{0.4\textwidth}
		\centering
		\includegraphics[width=2.5 in]{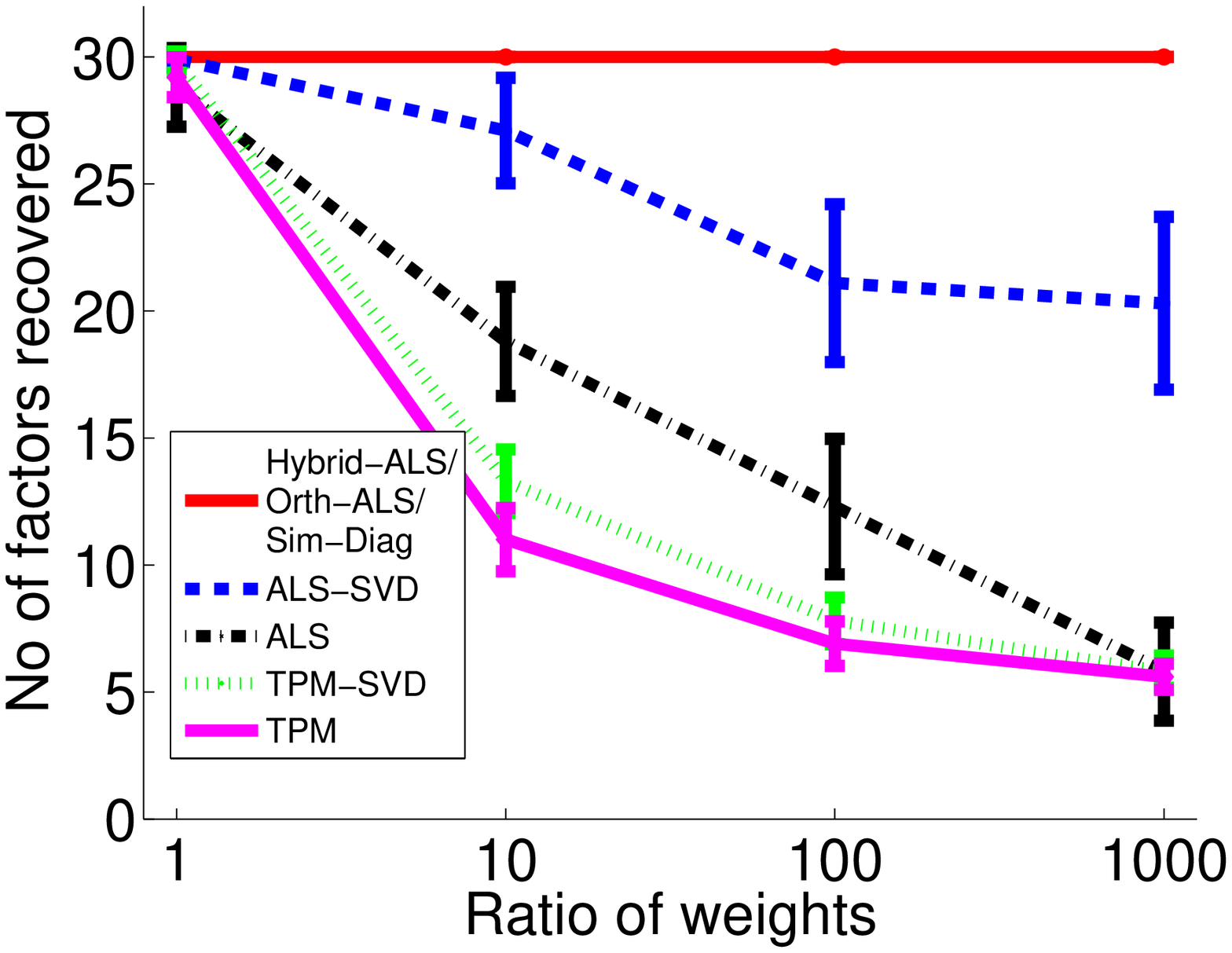}
		\subcaption{Noiseless case, ratio of weights equals $\frac{w_{\max}}{w_{\min}}$}
		\label{no_recover_noiseless}
	\end{subfigure}
	~
	\begin{subfigure}[t]{0.4\textwidth}
		\centering
		\includegraphics[width=2.5 in]{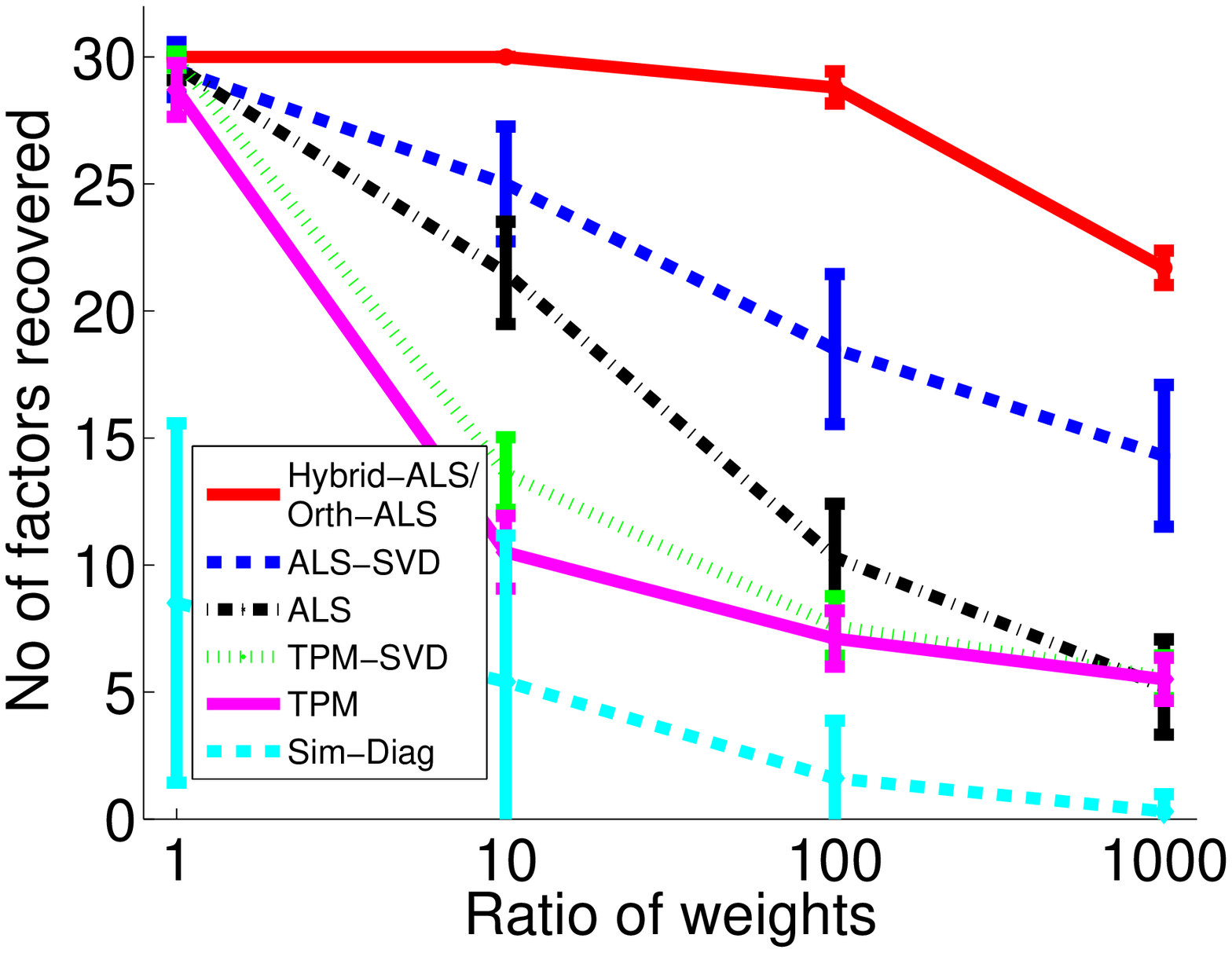}
		\subcaption{Noisy case, ratio of weights equals $\frac{w_{\max}}{w_{\min}}$}
		\label{no_recover_noisy}
	\end{subfigure}
	\caption{Average number of factors recovered by different algorithms for different values of $\frac{w_{\max}}{w_{\min}}$, the ratio of the maximum factor weight to minimum factor weight (with the weights spaced geometrically), along with error bars for the standard deviation in the number of factors recovered, across independent trials. The true rank $k=30$, and the dimension $d=100$.  We say a factor $\{A_i,B_i,C_i\}$ of the tensor $T$ is successfully recovered if there exists at least one recovered factor $\{\hat{A}_j, \hat{B}_j, \hat{C}_j\}$ with correlation at least 0.9 in all modes. Orth-ALS and Hybrid-ALS recover all factors in almost all settings, whereas ALS and the tensor power method struggle when the weights are skewed, even with the more expensive SVD based initialization.}
	\label{fig:random_rec}
	%			\begin{subfigure}[t]{0.4\textwidth}
	%				\centering
	%				\includegraphics[width=2.5 in]{iterations}
	%				\subcaption{Distribution of the number of iterations across all the experiments}
	%				\label{iterations}
	%			\end{subfigure}
\end{figure*}

\begin{figure*}[]
	\centering
	\begin{subfigure}[t]{0.4\textwidth}
		\centering
		\includegraphics[width=2.5 in]{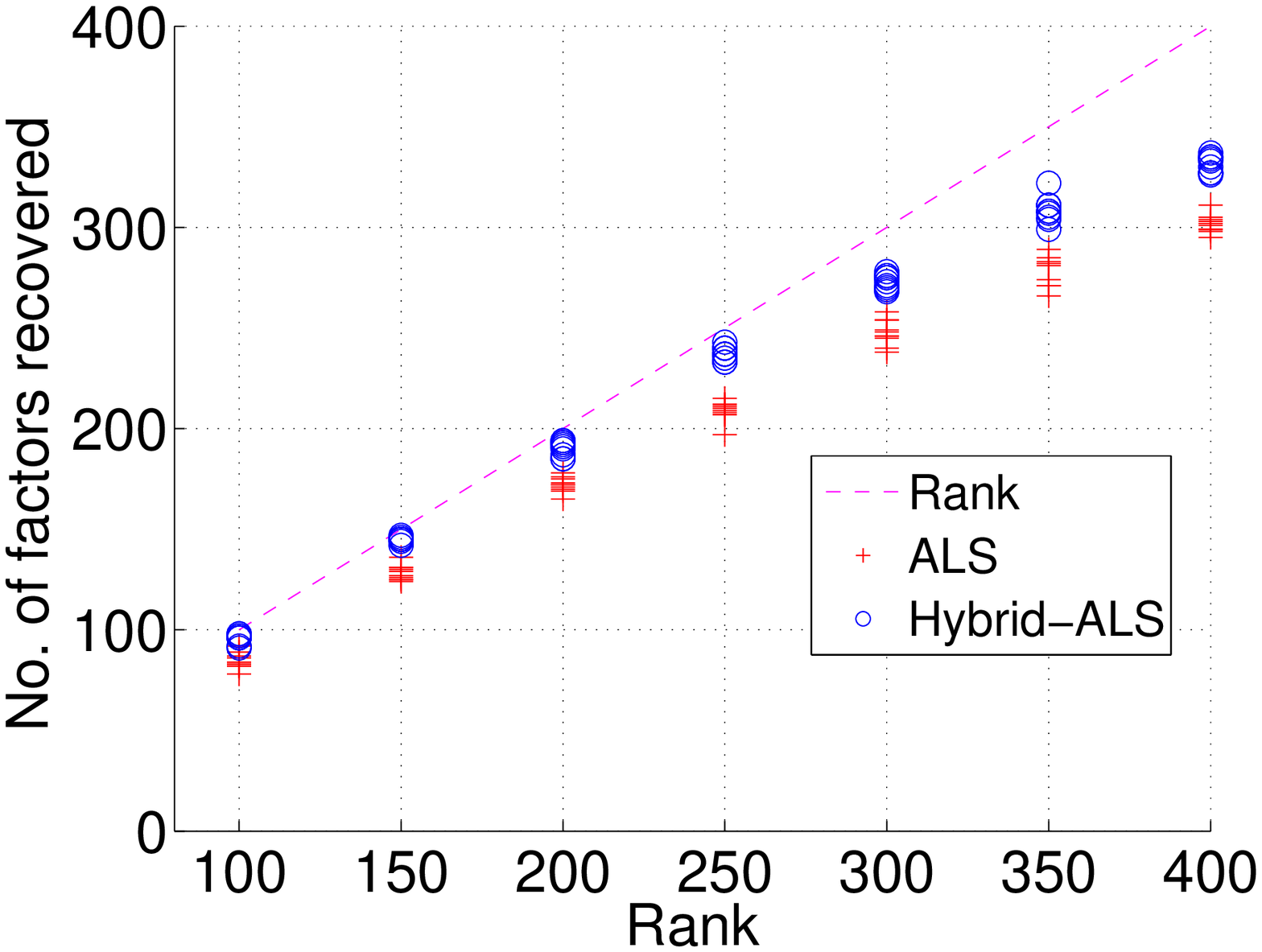}
		\subcaption{No. of factors recovered for a dimension $d=50$ tensor with varying rank. The weights of the factors are geometrically spaced and the ratio of the weights between every pair of consecutive factors is 1.05}
		\label{fig:overcomplete}
	\end{subfigure}
	~
	\begin{subfigure}[t]{0.4\textwidth}
		\centering
		\includegraphics[width=2.5 in]{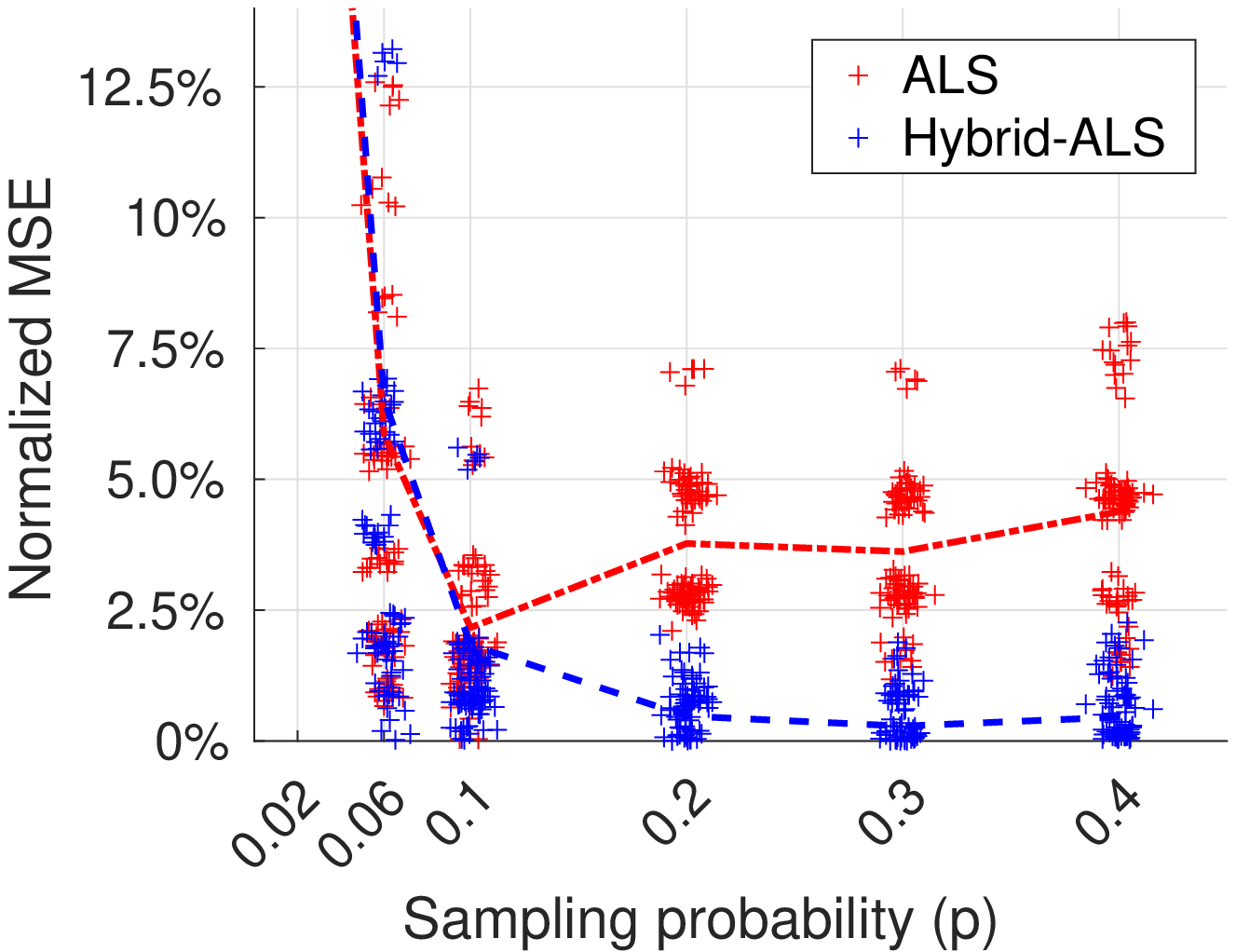}
		\subcaption{Average error on the missing entries for tensor completion where each entry is sampled with probability $p$, on a 100 different runs with each setting of $p$. The tensor has dimension $d=50$ and rank $k=10$.}
		\label{fig:completion}
	\end{subfigure}
	\caption{Experiments on overcomplete tensors and tensor completion. Even though our theoretical guarantees do not apply to these settings, we see that orthogonalization leads to significantly better performance over standard ALS.}
	\label{fig:over_completion}
	%			\begin{subfigure}[t]{0.4\textwidth}
	%				\centering
	%				\includegraphics[width=2.5 in]{iterations}
	%				\subcaption{Distribution of the number of iterations across all the experiments}
	%				\label{iterations}
	%			\end{subfigure}
\end{figure*}

\noindent \textbf{Recovering over-complete tensors:} Overcomplete tensors are tensors with rank higher than the dimension, and have found numerous theoretical applications in learning latent variable models \cite{anandkumar2015learning}. Even though orthogonalization cannot be directly applied to the setting where the rank is more than the dimension (as the factors can no longer be orthogonalized), we explore a deflation based approach in this setting. Given a tensor $T$ with dimension $d=50$ and rank $r>d$, we find a rank $d$ decomposition $T_1$ of $T$, subtract $T_1$ from $T$, and then compute a rank $d$ decomposition of $T_1$ to recover the next set of $d$ factors. We repeat this process to recover subsequent factors. After every set of $d$ factors has been estimated, we also refine the factor estimates of all factors estimated so far by running an additional ALS step using the current estimates of the extracted factors as the initialization. Fig. \ref{fig:overcomplete} plots the number of factors recovered when this deflation based approach is applied to a dimension $d=50$ tensor with a mild power low distribution on weights. We can see that Hybrid-ALS is successful at recovering tensors even in the overcomplete setup, and gives an improvement over ALS. \\

\noindent \textbf{Tensor completion:} We also test the utility of orthogonalization on a tensor completion task, where the goal is to recover a large missing fraction of the entries. Fig. \ref{fig:completion} suggests Hybrid-ALS gives considerable improvements over standard ALS. Further examining the utility of orthogonalization in this important setting, in theory and practice, would be an interesting direction. 

%\begin{figure}
%	\centering
%	\includegraphics[width=2.5 in]{mse_r4}
%	\caption{Plot of fit vs iteration for $r=4$. Fit is defined as $1-\frac{\lpnorm{T-T_{\text{rec}}}{\text{fr}}}{\lpnorm{T}{\text{fr}}}$}
%	\label{mse_r4}
%\end{figure}

\subsection{Learning Word Embeddings via Tensor Factorization}

\begin{table*}[t]
	\centering	
	\begin{tabular}{c c c c c} 
		\toprule
		Algorithm & WordSim & MEN & Mixed analogies & Syntactic analogies \\
		\toprule
		Vanilla ALS, $w=3$  & 0.47    & 0.51    & 30.22\% & 32.01\%     \\
		%Vanilla ALS, $w=4$   & 0.45    & 0.53    & 36.83\% & 37.41\% \\
		Vanilla ALS, $w=5$   & 0.44    & 0.51    & 37.46\% & 37.27\% \\
		\midrule
		Orth-ALS, $w=3$   & 0.57     & 0.59    & 38.44\% & 42.40\%    \\
		%Orth-ALS, $w=4$  & 0.57     & 0.60    & 43.17\% & 45.92\%    \\
		Orth-ALS, $w=5$  & 0.56     & 0.60    & 45.87\% & 47.13\%    \\
		\midrule
		Matrix SVD, $w=3$  & 0.61    & 0.69    & 47.77\% & 61.79\%     \\
		Matrix SVD, $w=5$   & 0.59    & 0.68    & 54.29\% & 62.20\% \\
		\bottomrule
	\end{tabular}
	\caption{Results for word analogy (mixed analogies and syntactic analogies) and word similarity tasks (WordSim, MEN) for different window lengths $w$ over which the co-occurrences are counted.  The embeddings recovered by Orth-ALS are significantly better than those recovered by standard ALS.  Despite this, embeddings derived from word co-occurrences using matrix SVD still outperform the tensor embeddings, and we are unsure whether this is due to the relative sparsity of the tensor, sub-optimal element-wise scaling (i.e. the $f(x)=\log(1+x)$ function applied to the co/tri-occurrence counts), or something more fundamental.}	\label{tasks}
\end{table*}

A word embedding is a vector representation of words which preserves some of the syntactic and semantic relationships in the language. %The underlying principle is the distributional hypothesis%\cite{harris1954distributional}
%, which states that words in similar contexts have similar meanings. 
Current methods for learning word embeddings implicitly \cite{mikolov2013distributed,levy2014neural} or explicitly \cite{pennington2014glove} factorize some matrix derived from the matrix of word co-occurrences $M$, where $M_{ij}$ denotes how often word $i$ appears with word $j$. We explore tensor methods for learning word embeddings, and contrast the performance of standard ALS and Orthogonalized ALS on standard tasks.
\vspace{-8pt}
\paragraph{Methodology.} We used the English Wikipedia as our corpus, with 1.5 billion words. We constructed a word co-occurrence tensor $T$ of the 10,000 most frequent words, where the entry $T_{ijk}$ denotes the number of times the  words $i$, $j$ and $k$ appear in a sliding window of length $w$ across the corpus. We consider two different window lengths, $w=3$ and $w=5$.   Before factoring the tensor, we apply the non-linear element-wise scaling $f(x)=\log(1+x)$ to the tensor.  This scaling is known to perform well in practice for co-occurrence matrices \cite{pennington2014glove}, and makes some intuitive sense in light of the Zipfian distribution of word frequencies.   Following the application of this element-wise nonlinearity, we recover a rank 100 approximation of the tensor using Orth-ALS or ALS. 

We concatenate the (three) recovered factor matrices into one matrix and normalize the rows. The $i$th row of this matrix is then the embedding for the $i$th word. We test the quality of these embeddings on two tasks aimed at measuring the syntactic and semantic structure captured by these word embeddings. %We also repeated the experiments using a dictionary of the most common 2,000 words and 4,000 words with window length $w=5$. 

We also evaluated the performance of matrix SVD based methods on the task. For this, we built the co-occurrence matrix $M$ with a sliding  window of length $w$ over the corpus. We applied the same non-linear element-wise scaling and performed a rank 100 SVD, and set the word embeddings to be the singular vectors after row normalization.

It is worth highlighting some implementation details for our experiments, as they indicate the practical efficiency and scalability inherited by Orth-ALS from standard ALS. Our experiments were run on a cluster with 8 cores and 48 GB of RAM memory per core. Most of the runtime was spent in reading the tensor, the runtime for Orth-ALS was around 80 minutes, with 60 minutes spent in reading the tensor (the runtime for standard ALS was around 100 minutes because it took longer to converge). Since storing a dense representation of the 10,000$\times$10,000$\times$10,000 tensor is too expensive, we use an optimized ALS solver for sparse tensors~\cite{splattsoftware,smith2015dms} which also has an efficient parallel implementation.  %The tensors are very sparse which speeds up computation---the tensor for $w=5$ with 10000 words has a density of $0.5 \%$.

\paragraph{Evaluation: Similarity and Analogy Tasks.} We evaluated the quality of the recovered word embeddings produced by the various methods via their performance on two different NLP tasks for which standard, human-labeled data exists: estimating the similarity between a pair of words, and completing word analogies.  

The word similarity tasks~\cite{bruni2012distributional,finkelstein2001placing} contain word pairs along with human assigned similarity scores, and the objective is to maximize the correlation between the similarity in the embeddings of the two words (according to a similarity metric such as the dot product) and human judged similarity. %We use the dot product or cosine similarity as the similarity metric for our word embeddings. 

The word analogy tasks \cite{mikolov2013efficient,mikolov2013linguistic} present questions of the form ``$a$ is to $a^*$ as $b$ is to $\rule{0.4cm}{0.15mm}$?'' (e.g. ``{Paris} is to {France} as {Rome} is to $\rule{0.4cm}{0.15mm}$?''). We find the answer to ``$a$ is to $a^*$ as $b$ is to $b^*$'' by finding the word whose embedding is the closest to $w_{a^*}-w_a+w_{b}$ in cosine similarity, where $w_a$ denotes the embedding of the word $a$.%After excluding pairs not in our dictionary of 10000 words, there are 260 pairs in the WordSim task, 1693 pairs in the MEN task, 5585 questions in the mixed analogies task and 3630 questions in the syntactic analogies task. The performance is summarized in the Table \ref{tasks}.

\paragraph{Results.}  The performances are summarized in the Table \ref{tasks}. The use of Orth-ALS rather than standard ALS leads to significant improvement in the quality of the embeddings as judged by the similarity and analogy tasks. However, the matrix SVD method still outperforms the tensor based methods. We believe that it is possible that better tensor based approaches (e.g. using better renormalization, additional data, or some other tensor rather than the symmetric tri-occurrence tensor) or a combination of tensor and matrix based methods can actually improve the quality of word embeddings, and is an interesting research direction.    Alternatively, it is possible that natural language does not contain sufficiently rich higher-order dependencies among words that appear close together, beyond the co-occurrence structure, to truly leverage the power of tensor methods.  Or, perhaps, the two tasks we evaluated on---similarity and analogy tasks---do not require this higher order.  In any case, investigating these possibilities seems worthwhile.

\section{Proof Overview: the Orthogonal Tensor Case}\label{sec:orthogonal_tensors}

% WE need to make it clear that the orthognalized updates for the top factor look like the TPM updates, and explain why we then analyze that.  THM->Prop as above.  clean up language.  is it $\log k + \log \log d$ or $\log \log (k+d)$ as it says somewhere?  shouldnt say "denote $\delta_{i,t} = blah,$ instead use ``let''.  whole proofs sounds like stream-of-conscience. we dont want to just prove this, we want to explain the proof.

In this section, we will consider Orthogonalized ALS for the special case when the factors matrix of the tensor is an orthogonal matrix. Although this special case is easy and numerous algorithms provably work in this setting, it will serve to highlight the high level analysis approach that we apply to the more general settings.

The analysis of Orth-ALS hinges on an analysis of the tensor power method. For completeness we describe the tensor power method in Algorithm \ref{tpm_all}. We will go through some preliminaries for our analysis of the tensor power method. Let the iterate of the tensor power method at time $t$ be $Z_t$. The tensor power method update equations can be written as (refer to \cite{anandkumar2014guaranteed})
\begin{align}
Z_t = \frac{\sum_{i=1}^{k}w_i \langle Z_{t-1}, A_i \rangle^2 A_i}{\norm{\sum_{i=1}^{k}w_i \langle Z_{t-1}, A_i \rangle^2 A_i}} \label{eq:tpm}
\end{align}
Eq. \ref{eq:tpm} is just the tensor analog of the matrix power method updates. For tensors, the updates are quadratic in the previous inner products, in contrast to matrices where the updates are linear in the inner products in the previous step.

\begin{algorithm}\caption{Tensor power method to recover all factors \cite{anandkumar2014guaranteed}} \label{tpm_all}
	\textbf{Input:} Tensor $T\in \Real^{d\times d\times d}$, number of initializations $L$, number of iterations $N$.
	\begin{algorithmic}[1]
		\FOR {$\tau=1:L$}
		\STATE Initialize $x_{0}^{(\tau)}, y_{0}^{(\tau)}, z_{0}^{(\tau)} \in \Real^d$ uniformly from the unit sphere or using the SVD based method %\cite{anandkumar2014guaranteed}
		\FOR {$t=1:N$}
		\STATE Rank-1 ALS/Power method updates-$x_{t+1}^{(\tau)}  = T_{(1)}(\hat{z}_t^{(\tau)}  \odot \hat{y}_t^{(\tau)} )$
		\STATE Rank-1 ALS/Power method updates-$y_{t+1}^{(\tau)}  = T_{(2)}(\hat{z}_t^{(\tau)}  \odot \hat{x}_t^{(\tau)} )$
		\STATE Rank-1 ALS/Power method updates-$z_{t+1}^{(\tau)}  = T_{(3)}(\hat{y}_t^{(\tau)}  \odot \hat{x}_t^{(\tau)} )$
		%\begin{align}
		%\hat{x}_{t+1} = \frac{T(I, \hat{x}_t, \hat{x}_t)}{\norm{T(I, \hat{x}_t, \hat{x}_t}}
		%x_{t+1}^{(\tau)}  \leftarrow T_{(1)}(\hat{x}_t^{(\tau)}  \odot \hat{x}_t^{(\tau)} )
		%\end{align}
		\STATE Normalize $x_{t+1}^{\tau},y_{t+1}^{\tau},z_{t+1}^{\tau}$ and store results in $\hat{x}_{t+1}^{\tau},\hat{y}_{t+1}^{\tau},\hat{z}_{t+1}^{\tau}$.
		%\STATE Set $\hat{x}_{t+1}^{(\tau)}  = x_{t+1}^{(\tau)} /\norm{x_{t+1}}^{(\tau)}$, $\hat{y}_{t+1}^{(\tau)}  = y_{t+1}^{(\tau)} /\norm{y_{t+1}}^{(\tau)} $, $\hat{z}_{t+1}^{(\tau)}  = z_{t+1}^{(\tau)} /\norm{z_{t+1}}^{(\tau)} $
		\ENDFOR
		\STATE Estimate weights-$\hat{w}^{(\tau)}  = T(\hat{x}_N^{(\tau)} , \hat{y}_N^{(\tau)} , \hat{z}_N^{(\tau)} )$
		%\begin{align}
		%\hat{w}^{(\tau)}  = T(\hat{x}_N^{(\tau)} , \hat{x}_N^{(\tau)} , \hat{x}_N^{(\tau)} )
		%\end{align}
		\ENDFOR
		\STATE Cluster set $\{(\hat{w}^{(\tau)}, \hat{x}_N^{(\tau)}, \hat{y}_N^{(\tau)}, \hat{z}_N^{(\tau)}), \tau \in [L] \}$ into $k$ clusters.
		\STATE \textbf{return} the centers $\{(\hat{w}_i, \hat{a}_i, \hat{b}_i, \hat{c}_i), i \in [k]\}$ of the $k$ clusters as the estimates 
	\end{algorithmic}
\end{algorithm}

Observe from Algorithm \ref{orth} (Orth-ALS) that the ALS steps in step 4-6 have the same form as tensor power method updates, but on the orthogonalized factors. This is the key idea we use in our analysis of Orth-ALS. Note that the first factor estimate is never affected by the orthogonalization, hence the updates for the first estimated factor have exactly the same form as the tensor power method updates, as this factor is unaffected by orthogonalization. The subsequent factors have an orthogonalization step before every tensor power method step. This ensures that they never have high correlation with the factors which have already been recovered, as they are projected orthogonal to the recovered factors before each ALS step. We then use the incoherence of the factors to argue that orthogonalization does not significantly affect the updates of the factors which have not been recovered so far, while ensuring that the factors which have already been recovered always have a small correlation.

Note that Eq. \ref{eq:tpm} is invariant with respect to multiplying the weights of all the factors by some constant. Hence for ease of exposition, we assume that all the weights lie in the interval $[1,\gamma]$, where $\gamma=\frac{w_{\max}}{w_{\min}}$. We also define $\eta = \max\{c_{\max},1/d\}$.

\begin{restatable}{proposition}{orthogonaltensor}
	\label{orthogonaltensor}
	Consider a $d$-dimensional rank $k$ tensor $T=\sum_{i=1}^{k}w_i A_i \otimes A_i \otimes A_i$ where the factor matrix $A$ is orthogonal. Define $\gamma=\frac{w_{\max}}{w_{\min}}$ to be the ratio of the largest and smallest weight. If the initial estimates for all the factors are initialized randomly from the unit sphere and the factors $\{A_j, j\ge i+1\}$ are re-randomized after $i(\log k+\log 
	\log d)$ steps where $i$ is an integer, then with high probability the orthogonalized ALS updates converge to the true factors in $O( k(\log k+\log\log d ))$ steps, and the error at convergence satisfies $\normsq{A_i - \hat{A}_i} \le O(\gamma k/d^2\})$ and $|1- \frac{\hat{w}_i}{w_i}| \le  O(1/d)$ for all $i$.
\end{restatable}

\begin{proof}
	Without loss of generality, we assume that the $i$th recovered factor converges to the $i$th true factor. As mentioned earlier, the iterations for the first factor are the usual tensor power method updates and are unaffected by the remaining factors. Therefore to show that orthogonalized ALS recovers the first factor, we only need to analyze the tensor method updates. We  show that the tensor power method with random initialization converges in $O(\log k + \log\log d)$ steps with probability at least ${(1-1/k^{1+\epsilon})}$, for any $\epsilon>0$. Hence this implies that Orth-ALS correctly recovers the first factor in $O(\log k + \log\log d)$ steps with probability at least $(1-1/k^{1+\epsilon})$, for any $\epsilon>0$.
	
	The main idea of our proof of convergence of the tensor power method is the following -- with decent probability, there is some separation between the correlations of the factors with the random initialization. By the tensor power method updates (Eq. \ref{eq:tpm}), this gap is amplified at every stage. We analyze the updates for all the factors together by a simple recursion. We then show that this recursion converges in in $O( \log k +\log\log d)$ steps.
	
%	Recall the tensor power method update equations-
%	\begin{align}
%	\hat{x}_{t+1} = \frac{T(I, \hat{x}_t, \hat{x}_t)}{\norm{T(I, \hat{x}_t, \hat{x}_t}}
%	\end{align}
	
	Let $Z_t$ be the iterate of the tensor power method updates at time $t$. Without loss of generality, we will be proving convergence to the first factor $A_1$. Let $a_{i,t}$ be the correlation of the $i$th factor $A_i$ with $Z_t$, i.e. $a_{i,t}=\langle A_i,Z_t\rangle$ (note that this should technically be called the inner product, but we refer to it as the correlation). We will refer to $w_i a_{i,t}$ as the weighted correlation of the $i$th factor.
	
	The first step of the proof is that with decent probability, there is some separation between the weighted correlation of the factors with the initial random estimate. This is Lemma \ref{good_start_whp}.
	\begin{restatable}{lemma}{goodstartwhp}
		\label{good_start_whp}
		If $\gamma kc_{\max} \le 1/k^{1+\epsilon}$ for some $\epsilon>0$, then with probability at least $\Big(1-\frac{\log^5 k}{k^{1+\epsilon}}\Big)$, $\frac{|{w_i} a_{i,0}|}{\max_i|{w_i} a_{i,0}|}  \le {1-5/k^{1+\epsilon}}$ $ \; \forall \; i\ne \argmax_i |w_ia_{i,0}|$.
	\end{restatable}
	The proof of Lemma \ref{good_start_whp} is a bit technical, but relies on basic concentration inequalities for Gaussians. Then using Eq. \ref{eq:tpm} the correlation at the end of the $(t+1)$th time step is given by
	\begin{align}
	a_{i,t+1} &= { w_i a_{i,t}^2}/\kappa_{t}\nonumber%\label{tpm_iter_orth}
	\end{align}
	where $\kappa_{t}=\norm{\sum_{i=1}^{k}w_i \langle Z_{t-1}, A_i \rangle^2}$ is the normalizing factor at the $t$th time step. 

Because the estimate is normalized at the end of the updates, we only care about the ratio of the correlations of the factors with the estimate rather than the magnitude of the correlations themselves. Hence, it is convenient to normalize all the correlations by the correlation of the largest factor and normalize all the weights by the weight of the largest factor. Therefore, let $\hat{a}_{i,t} = \frac{a_{i,t}}{a_{1,t}}$ and $\hat{w}_{i} = \frac{w_{i}}{w_{1}}$. The new update equation for the ratio of correlations $\hat{a}_{i,t}$ is-
\begin{align}
\hat{a}_{i,t+1} &= { \hat{w}_i \hat{a}_{i,t}^2} \label{eq:orth_update}
\end{align}
Our goal is to show that $a_{i,t}$ becomes small for all $i\ne 1$ in $O(\log k +\log \log d)$ steps. Instead of separately analyzing the different $a_{i,t}$ for different factors $A_i$, we upper bound $a_{i,t}$ for all $i$ via a simple recursion. Consider the recursion,
\begin{align}
\beta_{0} &= \max_{i\ne 1}{\Big|\hat{w}_i \hat{a}_{i,0}\Big|} \label{eq:orth_rec1} \\
\beta_{t+1} &= \beta_{t}^2 \label{eq:orth_rec2}
\end{align}
We claim that $|\hat{w}_i \hat{a}_{i,t}| \le \beta_{t}$ for all $t$ and $i\ne 1$. By Eq. \ref{eq:orth_rec1}, this is true for $t=0$ by definition. We prove our claim via induction. Assume that $|\hat{w}_i \hat{a}_{i,t}| \le \beta_{t}$ for $t=p$. Note that by Eq. \ref{eq:orth_update}, $\hat{w}_i\hat{a}_{i,p+1} = { \hat{w}_i^2 \hat{a}_{i,p}^2}$. Therefore $w_i\hat{a}_{i,p+1} \le \beta_{p+1}^2$ for all $i\ne 1$. Hence $|\hat{w}_i \hat{a}_{i,t}| \le \beta_{t}$ for all $t$ and $i\ne 1$. Note that as the weights lie in the interval $[1, \gamma]$, $\hat{a}_{i,t} \le \beta_t$. 

To show convergence, we will now analyze the recursion in Eq. \ref{eq:orth_rec1}. We will show that $\beta_t$ becomes sufficiently small in $O(\log k +\log\log d)$ steps. Note that $\beta_{t} = (\beta_{0})^{2^t}$ and $\beta_0 \le 1-5/k^{1+\epsilon}$. Therefore $\beta_t \le 0.1$ for $t=2\log k $. In another $\log \log d$ steps, $\beta_t\le1/d$. Hence $\beta_t \le 1/d$ in $O(\log k +\log \log d)$ steps. As $\beta_t$ is an upper bound for the correlation of all but the first factor, hence $|\hat{a}_{i,t}|\le 1/d$ for all $i\ne 1$ in $O(\log k +\log \log d)$ steps. 

To finish the proof of convergence for the tensor power method, we need to show that the estimate $Z_t$ is close to $A_1$ if it has small correlation with every factor other than $A_1$. Lemma \ref{error_converge} shows that if the ratio of the correlation of every other factor with $A_1$ is small, then the residual $\ell_2$ error in estimating $A_1$ is also small.

\begin{restatable}{lemma}{errorconverge}\label{error_converge}
	Let $\gamma kc_{\max} \le 1/k^{1+\epsilon}$. Without loss of generality assume convergence to the first factor $A_1$. Define $\hat{a}_{i,t} =|\frac{{a}_{i,t}}{{a}_{1,t}}|$- the ratio of the correlation of the $i$th and 1st factor with the iterate at time $t$. If $\hat{a}_{i,t} \le  2\eta\;\forall\; i\ne 1$, then ${\normsq{A_1 - \hat{A}_{1}}} \le 10\gamma k\eta^2$ in the subsequent iteration. Also, if $\normsq{A_1 - \hat{A}_{1}} \le O(\gamma k\eta^2)$ the relative error in the estimation of the weight $w_1$ is at most $O(\eta)$.
\end{restatable}

Using Lemma \ref{error_converge}, it follows that the estimate $\hat{A}_1$ and $\hat{w}_1$ for the factor $A_1$ satisfies $\normsq{A_1 - \hat{A}_{1}} \le 10\gamma k/d^2$ and $|1-\frac{\hat{w}_1}{w_1}| \le  O(1/d)$. Hence we have shown that Orth-ALS correctly recovers the first factor.

We now prove that Orth-ALS also recovers the remaining factors. The proof proceeds by induction. We have already shown that the base case is correct and the algorithm recovers the first factor. We next show that if the first $(m-1)$ factors have converged, then the $m$th factor converges in $O(\log k + \log \log d)$ steps with failure probability at most $\tilde{O}(1/{k^{1+\epsilon}})$. The main idea is that as the factors have small correlation with each other, hence orthogonalization does not affect the factors which have not been recovered but ensures that the $m$th estimate never has high correlation with the factors which have already been recovered. Recall that we assume without loss of generality that the $i$th recovered factor $X_i$ converges to the $i$th true factor, hence $X_{i} = A_i + \hat{\Delta}_i$ for $i<m$, where $\norm{ \hat{\Delta}_i} \le 10\gamma k /d^2$. This is our induction hypothesis, which is true for the base case as we just showed that the tensor power method updates converge with residual error at most $10\gamma k/d^2$. 

Let $X_{m,t}$ denote the $m$th factor estimate at time $t$ and let $Y_m$ denote it's value at convergence. We will first calculate the effect of the orthogonalization step on the correlation between the factors and the estimate $X_{m,t}$. Let $\{\bar{X}_i, i<m\}$ denote an orthogonal basis for $\{X_i, i<m\}$. The basis  $\{\bar{X}_i, i<m\}$ is calculated via QR decomposition, and can be recursively written down as follows,
\begin{align}
\bar{X}_i = \frac{{X}_{i} - \sum_{j<i}^{}\bar{X}_j^T X_{i} \bar{X}_j}{\norm{{X}_{i} - \sum_{j<i}^{}\bar{X}_j^T X_{i} \bar{X}_j}}\label{orth_basis_calc_1}
\end{align}
Note that the estimate $X_{m,t}$ is projected orthogonal to this basis. Define $\bar{X}_{m,t}$ as this orthogonal projection, which can be written down as follows --
\begin{align}
\bar{X}_{m,t} = \bar{X}_{m,t} - \sum_{j<m}^{}\bar{X}_j^T X_{m,t} \bar{X}_j \nonumber
\end{align}
In the QR decomposition algorithm $\bar{X}_{m,t}$ is also normalized to have unit norm but we will ignore the normalization of $X_{m,t}$ in our analysis because as before we only consider ratios of correlations of the true factors with $\bar{X}_{m,t}$, which is unaffected by normalization. 

We will now analyze the orthogonal basis $\{\bar{X}_i, i<m\}$. The key idea is that the orthogonal basis $\{\bar{X}_i, i<m\}$ is close to the original factors $\{A_i, i<m\}$ as the factors are incoherent. Lemma \ref{orth_basis} proves this claim.

\begin{restatable}{lemma}{orthbasis}\label{orth_basis}
	Consider a stage of the Orthogonalized ALS iterations when the first $(m-1)$ factors have converged. Without loss of generality let $X_{i} = A_i + \hat{\Delta}_i, i<m$, where$\norm{ \hat{\Delta}_i} \le 10\gamma k \eta^2$. Let $\{\bar{X}_i, i<m\}$ denote an orthogonal basis for $\{X_i, i<m\}$ calculated using Eq. \ref{orth_basis_calc_1}. Then,
	\begin{enumerate}
		\item $\bar{X}_i = A_i + \Delta_i, \; \forall \; i<m$ and $\norm{\Delta_j} \le 10k\eta$. 
		\item $|A_j ^T \Delta_i| \le 3\eta, \; \forall \; i<m, \; j<i$.
		\item $|A_j ^T \Delta_i| \le 20\gamma k\eta^2, \; \forall \; i<m,\; j>i$.
	\end{enumerate}
\end{restatable}

Using Lemma \ref{orth_basis}, we will find the effect of orthogonalization on the correlations of the factors with the iterate $X_{m,t}$. At a high level, we need to show that the iterations for the factors $\{A_i, i \ge m\}$ are not much affected by the orthogonalization, while the correlations of the factors $\{A_i, i < m\}$ with the estimate $X_{m,t}$ are ensured to be small. Lemma \ref{orth_basis} is the key tool to prove this, as it shows that the orthogonalized basis is close to the true factors. 

We will now analyze the inner product between $\bar{X}_{m,t}$ and factor $A_i$. This is given by-
\begin{align}
A_i^T \bar{X}_{m,t} = A_i^T X_{m,t} -\sum_{j<m}^{}X_{m,t}^T \bar{X}_jA_i^T \bar{X}_j\nonumber
\end{align}
As earlier, we normalize all the correlations by the correlation of the largest factor, let $\bar{a}_{i,t}$ be the ratio of the correlations of $A_i$ and $A_m$ with the orthogonalized estimate $\bar{X}_{m,t}$ at time $t$. We can write $ \bar{a}_{i,t}$ as-
\begin{align}
\bar{a}_{i,t} = \frac{ A_i^T X_{m,t} -\sum_{j<m}^{} X_{m,t}^T \bar{X}_jA_i^T\bar{X}_j}{ A_m^T X_{m,t} -\sum_{j<m}^{} X_{m,t}^T \bar{X}_jA_m^T\bar{X}_j}\nonumber
\end{align}
We can multiply both sides by $\hat{w}_i$ and substitute $\bar{X}_j$ from Lemma \ref{orth_basis} and then rewrite as follows-
\begin{align}
\hat{w}_i \bar{a}_{i,t} = \frac{ \hat{w}_iA_i^T X_{m,t} -\sum_{j<m}^{} \hat{w}_iX_{m,t}^T (A_j + \Delta_j)A_i^T \Delta_j}{ A_m^T X_{m,t} -\sum_{j<m}^{} X_{m,t}^T (A_j + \Delta_j)A_m^T\Delta_j}\nonumber
\end{align}
We divide the numerator and denominator by $ A_m^TX_{m,t} $ to derive an expression in terms of the ratios of correlations. Let $\delta_{i,t} =  \frac{X_{m,t}^T\Delta_j}{X_{m,t}^T A_m}$. 
\begin{align}
\hat{w}_i\bar{a}_{i,t} = \frac{ \hat{w}_i\hat{a}_{i,t} -\sum_{j<m}^{} ( \hat{w}_i\hat{a}_{j,t}+ \hat{w}_i\delta_{i,t} )A_i^T\Delta_j}{ 1 -\sum_{j<m}^{}  (\hat{a}_{j,t} +\delta_{i,t} )A_m^T \Delta_j}\nonumber
\end{align}

We now need to show $\hat{w}_i\bar{a}_{i,t}$ is small for all $i<m$ and is close to $\hat{w}_i{a}_{i,t}$, the weighted correlation before orthogonalization, for all $i>m$. Lemma \ref{orth_effect} proves this, and shows that the weighted correlation of factors which have not yet been recovered, $\{A_i, i\ge m\}$, is not much affected by orthogonalization, but the factors which have already been recovered. $\{A_i, i< m\}$, are ensured to be small after the orthogonalization step.

	\begin{restatable}{lemma}{ortheffect}
		\label{orth_effect}
		Let $|\hat{w}_i \hat{a}_{i,t}| \le \beta_t \; \forall \;i \ne m$ at the end of the $t$th iteration. Let $\bar{a}_{i,t}$ be the ratio of the correlation of the $i$th and the $m$th factor with $X_{m,t}$, the iterate at time $t$ after the orthogonalization step. Then,
		\begin{enumerate}
			\item $|\hat{w}_i\bar{a}_{i,t}| \le \beta_t(1+1/k^{1+\epsilon})), \; \forall \; i > m$.
			\item $|\hat{w}_i\bar{a}_{i,t}| \le 50\gamma  k\eta\beta_t, \; \forall \; i < m$.
		\end{enumerate} 
	\end{restatable}
	
We are now ready to analyze the Orth-ALS updates for the $m$th factor. First, we argue about the initialization step. Lemma \ref{orth_effect} shows that an orthogonalization step performed after a random initialization ensures that the factors which have already been recovered have small correlation with the orthogonalized initialization. This is where we need a periodic re-randomization of the factors which have not converged so far.
	
	\begin{restatable}{lemma}{orthini}\label{orth_ini}
			Let $X_{m,0}$ be initialized randomly and the result be projected orthogonal to the $(m-1)$ previously estimated factors, let these be $\{X_i, i<m\}$ without loss of generality. Then $\argmax_i |w_i a_{i,0}| \ge m$ with high probability. Also, with failure probability at most $\Big(1-\frac{\log^5 k}{k^{1+\epsilon}}\Big)$, $\Big|\frac{w_i\bar{a}_{i,0}}{\max_i\{w_i \bar{a}_{i,0}\}} \Big| \le {1-4/k^{1+\epsilon}} \; \forall \; i\ne \argmax_i |w_i a_{i,0}|$ after the orthogonalization step.
	\end{restatable}
	
	Lemma \ref{orth_ini} shows that with high probability, the initialization for the $m$th recovered factor has the largest  weighted correlation with a factor which has not been recovered so far after the orthogonalization step. It also shows that the separation condition in Lemma \ref{good_start_whp} is satisfied for all remaining factors with probability $(1-\log^5 k/k^{1+\epsilon})$. 
	
	Now, we combine the effects of the tensor power method step and the orthogonalization step for subsequent iterations to show that that $X_{m,t}$ converges to $A_m$. Consider a tensor power method step followed by an orthogonalization step. By our previous argument about the convergence of the tensor power method,  if $|\hat{w}_i \hat{a}_{i,t-1}| \le \beta_{t-1}$ $i\ne m$ at some time $(t-1)$, then $|\hat{w}_i \hat{a}_{i,t}| \le \beta_{t-1}^2$ for $i\ne m$ after a tensor power method step. Lemma \ref{orth_effect} shows that the correlation of all factors other than the $m$th factor is still small after the orthogonalization step, if it was small before. Combining the effect of the orthogonalization step via Lemma \ref{orth_effect}, if $|\hat{w}_i \hat{a}_{i,t-1}| \le \beta_{t-1}$ $i\ne m$ for some time $(t-1)$, then $|\hat{w}_i \hat{a}_{i,t}| \le \beta_{t-1}^2(1+1/k^{1+\epsilon})$ for $i\ne m$ after both the tensor power method and the orthogonalization steps. By also using Lemma \ref{orth_ini} for the initialization, can now write the updated combined recursion analogous to Eq. \ref{eq:orth_rec1} and Eq. \ref{eq:orth_rec2}, but which combines the effect of the tensor power method step and the orthogonalization step.	
	\begin{align}
	\beta_{0} &= \max_{i\ne 1}{\Big|\hat{w}_i \hat{a}_{i,0}\Big|} \\
	\beta_{t+1} &= \beta_{t}^2(1+1/k^{1+\epsilon}) \label{eq:gen_rec2}
	\end{align}
	
	By the previous argument, $|w_i \bar{a}_{i,t}| \le \beta_t$. Note that $\beta_0 \le 1-4/k^{1+\epsilon}$ by Lemma \ref{orth_ini}. By expanding the recursion \ref{eq:gen_rec2}, $\beta_t=(\beta_0(1+1/k^{1+\epsilon}))^{2^t}$. Hence $\beta_t \le 1/d$ in $2\log k + \log \log d$ steps as was the case for the analysis for the tensor power method. This shows that the correlation of the estimate $X_{m,t}$ with all factors other than $A_m$ becomes small in $2\log k + \log \log d$. We now again use Lemma \ref{error_converge} to argue that this implies that the recovery error is small, i.e. $\normsq{A_m - \hat{A}_{m}} \le 10\gamma k/d^2$ and $|1- \frac{\hat{w}_m}{w_m}| \le  O(1/d)$. 
	
	Hence we have shown that if the first $(m-1)$ factors have converged to $X_{i} = A_i + \hat{\Delta}_i$ where $\norm{ \hat{\Delta}_i} \le 10\gamma k/d^2, \; \forall\; i <m$ then the $m$th factor converges to $X_{m} = A_m + \hat{\Delta}_m$ where $\norm{ \hat{\Delta}_m} \le 10\gamma k/d^2$ in $O(\log k + \log \log d)$ steps with probability at least $\Big(1-\frac{\log^5k}{k^{1+ \epsilon}}\Big)$. This proves the induction hypothesis.
	
	We can now do a union bound to argue that each factor converges with $\ell_2$ error at most $ O(\gamma k/d^2)$ in $O(k(\log k + \log \log d))$ with overall failure probability at most $\tilde{O}(1/k^{-\epsilon}), \epsilon>0$. This finishes the proof of convergence of Orth-ALS for the special case of orthogonal tensors.

\end{proof}
\vspace{-10pt}\section{Conclusion}

Our results suggest the theoretical and practical benefits of Orthogonalized ALS, versus standard ALS.   An interesting direction for future work would be to more thoroughly examine the practical and theoretical utility of orthogonalization for other tensor-related tasks, such as tensor completion.  Additionally, its seems worthwhile to  investigate Orthogonalized ALS or Hybrid ALS in more application-specific domains, such as natural language processing.

\bibliographystyle{plainnat}
\bibliography{references.bib}
% \newpage
\appendix
\section{Global convergence of the tensor power method for incoherent tensors}

In this section, we will analyze the tensor power method updates for worst-case incoherent tensors. This is a necessary step before analyzing Orth-ALS, because as was pointed out in the proof of convergence of Orth-ALS in the orthogonal tensor case, analyzing Orth-ALS updates reduces to analyzing a perturbed version of the tensor power method updates. Our convergence results for the tensor power method are interesting independent of Orth-ALS though, as they prove global convergence under random initialization. The proof idea is similar to the proof of convergence of the tensor power method in the orthogonal case, but we now need to analyze the cross-terms which come in because the factors are no longer orthogonal.

%\begin{theorem}
%	\label{global_convergence_whp}
%	Let $\gamma kc_{\max} < 1/k^{1+\epsilon}, \epsilon > 0$ and $\gamma$ be a constant such that $\frac{w_{\max}}{w_{\min}}\le\gamma, w_{\min} \le \gamma$ and $w_{\max} \le \gamma$. If $x_0\in \Real^d$ is chosen uniformly from the unit sphere, then the alternating least squares updates converge to one of the true factors (say $A_1$)  with constant probability over the random initialization in $O(\log k + \log\log d)$ steps, and the error at convergence satisfies $\norm{A_1 - \hat{x}_{N}} \le O(\max(kc_{\max}^2),k/d^2)$. Also, the estimate of the weight $\bar{w}_1$ satisfies $|1- \frac{\bar{w}_1}{w_1}| \le O(\max(c_{\max}),1/d)$.
%\end{theorem}

\globalconvergence*

\begin{proof}
	
	Without loss of generality, we will prove convergence to the first factor $A_1$. The proof is similar in spirit to the proof of convergence of the tensor power method in the orthogonal case in Section \ref{sec:orthogonal_tensors}. 
	
	As in the orthogonal case, Lemma \ref{good_start_whp} states that with high probability there is some separation between the weighted correlation of the largest and second largest factors.
	\goodstartwhp*
	
	We normalize all the correlations by the correlation of the largest factor, let $\hat{a}_{i,t+1} = \frac{a_{i,t}}{a_{1,t}}$ and normalize all the weights by the weight of the largest factor, $\hat{w}_{i} = \frac{w_{i}}{w_{1}}$. The new update equations in terms of the ratio of correlations $\hat{a}_{i,t}$ become-
	\begin{align}
	\hat{a}_{i,t+1} &=\frac { \hat{w}_i \hat{a}_{i,t}^2+ c
		_{i,1}+ \sum_{j: j \ne \{i,1\}}^{}c_{i,j}\hat{w}_j \hat{a}_{j,t}^2}{1+\sum_{j: j \ne 1}^{}c_{1,j}\hat{w}_j \hat{a}_{j,t}^2}\label{eq:non_orth_update}
	\end{align}
	Notice that we have cross terms in Eq. \ref{eq:non_orth_update} as compared to Eq. \ref{eq:orth_update} in the orthogonal case, due to the correlation $c_{i,j}$ between the factors being non-zero. The goal of the analysis for the non-orthogonal case is to bound these cross-terms using the incoherence between the factors. 
	
	As in the orthogonal case, we will analyze all the correlations $\hat{a}_{i,t}$ via a single recursion. We define $\beta_t$ in the non-orthogonal case keeping in mind the cross-terms because of the correlations between the factors being non-zero.
	\begin{align}
		\beta_{0} &= \max_{i\ne 1}{\Big|{w_i} \hat{a}_{i,0}\Big|} \label{beta_1}\\
		\beta_{t+1} &= \gamma c_{\max} + \beta_{t}^2 + 3\gamma kc_{\max}\beta_{t}^2\label{beta_2}
	\end{align}
	We now show that $|{w_i}\hat{a}_{i,t}|\le \beta_t, \forall \; i \ne 1$ and all $t$.
	\begin{lemma}\label{induction_global}
		If $|\hat{w}_i\hat{a}_{i,m}| \le \beta_m$ for some time $m$ and for all $i\ne 1$, then at time $(m+1)$ for all $i\ne 1$,
		\begin{enumerate}
			\item  $|\hat{w}_i\hat{a}_{i,m+1}| \le \beta_{m+1}$. 
			\item $|\hat{a}_{i,m+1}-c_{i,1}|\le 2 \beta_{m}^2$
		\end{enumerate}
	\end{lemma}
	\begin{proof}
		%We will prove the results by induction. By definition $|\hat{w}_i\hat{a}_{i,0}| \le \beta_0$, hence the base case is correct. Let the statement be true after $m$ steps. 
		Note that by Lemma \ref{iteration_global}, $\beta_t < 1 \;\forall \;t \implies \hat{w}_i\hat{a}_{i,m}^2 \le 1$. Therefore $\sum_{j}^{}|c_{i,j}\hat{w}_j \hat{a}_{j,m}^2| \le  kc_{\max} \le 1/k^{1+\epsilon} \;\forall \; i$. Hence we can write,
	\begin{align}
	\frac{1}{1+\sum_{j: j \ne 1}^{}c_{1,j}\hat{w}_j \hat{a}_{j,m}^2} = 1 - \sum_{j: j \ne 1}^{}c_{1,j}\hat{w}_j \hat{a}_{j,m}^2+\epsilon_1 \nonumber
	\end{align}
	where $\epsilon_1$ is the residual term, and $|\epsilon_1| \le \Big|\sum_{j: j \ne 1}^{}c_{1,j}\hat{w}_j \hat{a}_{j,m}^2\Big|^2 \le k^2c_{\max}^2\le 1/k^2$.
	We can now rewrite the updates for $\hat{a}_{i,m+1}$ as-
	\begin{align}
	\hat{a}_{i,m+1} &= \Big(c_{i,1}+\hat{w}_i \hat{a}_{i,m}^2 +\sum_{j: j \ne \{i,1\}}^{}c_{i,j}\hat{w}_j \hat{a}_{j,m} ^2\Big)\Big(1-\sum_{j: j \ne 1}^{}c_{1,j}\hat{w}_j \hat{a}_{j,m}^2+\epsilon_1\Big)\nonumber
	\end{align}
	Let $\rho_{i,m}=c_{i,1}+\hat{w}_i \hat{a}_{i,m}^2 +\sum_{j: j \ne \{i,1\}}^{}c_{i,j}\hat{w}_j \hat{a}_{j,m} ^2 $. We can write,
	\begin{align}
	\hat{a}_{i,m+1} &= c_{i,1} + \hat{w}_i \hat{a}_{i,m}^2 +\sum_{j: j \ne \{i,1\}}^{}c_{i,j}\hat{w}_j \hat{a}_{j,m}^2 -\rho_{i,m}\sum_{j: j \ne 1}^{}c_{1,j}\hat{w}_j \hat{a}_{j,m}^2+\rho_{i,m}\epsilon_1\label{recusion_global}\\
	\implies \Big|\hat{a}_{i,m+1}\Big| &\le \Big| c_{i,1} \Big| + \Big|\hat{w}_i\hat{a}_{i,m}^2\Big| +\Big|\sum_{j: j \ne \{i,1\}}^{}c_{i,j}\hat{w}_j \hat{a}_{j,m}^2 \Big|+\Big|\rho_{i,m}\sum_{j: j \ne 1}^{}c_{1,j}\hat{w}_j \hat{a}_{j,m}^2\Big| +\Big|\rho_{i,m}\epsilon_1\Big|\label{eq:recursion_global_2}
	\end{align}
	 We claim that $\rho_{i,m}\le 1$. We verify this as follows,
		\begin{align}
		\rho_{i,m}&= c_{i,1}+\hat{w}_i \hat{a}_{i,m}^2 +\sum_{j: j \ne \{i,1\}}^{}c_{i,j}\hat{w}_j \hat{a}_{j,m} ^2\nonumber\\
		\implies \Big|\rho_{i,m}\Big| &\le \Big|c_{i,1}\Big|+\Big|\hat{w}_i \hat{a}_{i,m}^2\Big| + \Big|\sum_{j: j \ne \{i,1\}}^{}c_{i,j}\hat{w}_j \hat{a}_{j,m} ^2\Big|\nonumber\\
		\implies   \Big|\hat{w}_{i}\rho_{i,m}\Big| &\le \gamma \Big|c_{i,1}\Big|+\Big|\hat{w}_i^2 \hat{a}_{i,m}^2\Big| +  \Big|\sum_{j: j \ne \{i,1\}}^{}c_{i,j}\hat{w}_i\hat{w}_j \hat{a}_{j,m} ^2\Big|\nonumber\\
		&\le  \gamma c_{\max} + \beta_m^2 + \gamma k c_{\max} \beta_{m}^2\nonumber\\
		&\le \beta_{m+1} \le 1\nonumber\\
		\implies   \Big|\rho_{i,m}\Big| &\le 1 \nonumber
		\end{align}
		where we used the fact that the weights lie in the interval $[1,\gamma]$. Hence $|\rho_{i,m}|\le 1$. Therefore, by Eq. \ref{eq:recursion_global_2},
		\begin{align}
		%\Big|\hat{a}_{i,m+1}\Big| &\le \Big|c_{i,1} \Big| + \Big|\hat{w}_i\hat{a}_{i,m}^2\Big| +\Big|\sum_{j: j \ne \{i,1\}}^{}c_{i,j}\hat{w}_j \hat{a}_{j,m}^2 \Big|+\gamma \Big|\sum_{j: j \ne 1}^{}c_{1,j}\hat{w}_j \hat{a}_{j,m}^2\Big| +\gamma \Big|\epsilon_1\Big|\label{iter}\\
		\Big|\hat{a}_{i,m+1}\Big| &\le \Big|c_{i,1} \Big| + \Big|\hat{w}_i\hat{a}_{i,m}^2\Big| + \Big|\sum_{j: j \ne \{i,1\}}^{}c_{i,j}\hat{w}_j \hat{a}_{j,m}^2 \Big|+  \Big|\sum_{j: j \ne 1}^{}c_{1,j}\hat{w}_j \hat{a}_{j,m}^2\Big| +\Big|\epsilon_1\Big|\nonumber\\
		\implies \Big|\hat{w}_i \hat{a}_{i,m+1}\Big| &\le \gamma \Big|c_{i,1} \Big| + \Big|\hat{w}_i^2\hat{a}_{i,m}^2\Big| + \Big|\sum_{j: j \ne \{i,1\}}^{}c_{i,j}\hat{w}_i\hat{w}_j \hat{a}_{j,m}^2 \Big|+ \Big|\sum_{j: j \ne 1}^{}c_{1,j}\hat{w}_i\hat{w}_j \hat{a}_{j,m}^2\Big| +\gamma \Big|\epsilon_1\Big|\nonumber\\
		&\le \gamma c_{\max} + \beta_m^2 + 3\gamma k c_{\max} \beta_{m}^2= \beta_{m+1}
		\end{align}
		To show that $|\hat{a}_{i,m+1}-c_{i,1}|\le 2 \beta_{m}^2$ we use Eq. \ref{recusion_global} and repeat the steps used to show that $|\hat{w}_i\hat{a}_{i,m+1}| \le \beta_{m+1}\;\forall \;t$.
	\end{proof}
	By using induction and Lemma \ref{induction_global}, the iterates at all time $t$ satisfy the following properties, for all $i\ne 1$,
			\begin{enumerate}
				\item  $|\hat{w}_i\hat{a}_{i,t}| \le \beta_{t}\;\forall \;t$. 
				\item $|\hat{a}_{i,t}-c_{i,1}|\le 2 \beta_{t-1}^2$
			\end{enumerate}
	This allows us to analyze the iterations of $\beta_t$ instead of keeping track of the different $a_{i,t}$. We will now analyze the recursion for $\beta_t$. The following Lemma shows that $\beta_t$ becomes sufficiently small in $O(\log k +\log \log d)$ steps.
	\begin{lemma}\label{iteration_global}
		$\beta_t \le 3\gamma \eta \; \forall \; t\ge O(\log k +\log \log d)$, also $\beta_t < 1 \;\forall \;t$
	\end{lemma}
	\begin{proof}
		We divide the updates into three stages.
		\begin{enumerate}
			\item $0.1 \le \beta_t \le  {1-5/k^{1+\epsilon}}$: 
			
			As $\beta_t \ge 0.1$, therefore $k \beta_t^2 \ge 1$ in this regime and hence $\gamma c_{\max}\le \gamma k c_{\max}\beta_t^2$, and we can write-
			\begin{align}
			\beta_{t+1} &= \gamma c_{\max} + \beta_{t}^2 +  3\gamma k c_{\max}\beta_{t}^2\nonumber\\ 
			\beta_{t+1} &\le  \beta_{t}^2 + 4\gamma k c_{\max}\beta_{t}^2\nonumber
			\end{align}
			We claim that $\beta_t < 0.1$ for $t= O(\log d)$. To verify, note that-
			\begin{align}
			\beta_{t} &\le (\beta_0 ( 1+4\gamma k c_{\max}))^{2^t}\nonumber\\
			&\le \Big((1-5/k^{1+\epsilon})(1+1/k^{1+\epsilon})\Big)^{2^t}\nonumber\\
			&\le \Big(1-1/k^{1+\epsilon}\Big)^{2^t}
			\end{align}	
			where we used the fact that $\gamma kc_{\max}\le 1/k^{1+\epsilon}$. Note that $(1-1/k^{1+\epsilon})^{2^t}\le 0.1$ for $t=2\log k$ and hence we stay in this regime for at most $2\log k$ steps. 
			\item $\sqrt{\gamma \eta} \le \beta_t \le 0.1:$
			
			For notational convenience, we restart $t$ from 0 in this stage. Because $\gamma c_{\max}\le \gamma \eta\le \beta_t^2$ in this regime and $3\gamma kc_{\max} \beta_{t}^2\le 0.1 \beta_t^2$ as $\gamma kc_{\max}\le 1/k^{1+\epsilon}$, we can write-
			\begin{align}
			\beta_{t+1} &= \gamma c_{\max} + \beta_{t}^2 +  3\gamma k c_{\max}\beta_{t}^2\nonumber\\
			&\le \beta_t^2 + \beta_t^2 + 0.3\beta_t^2 \le 2.5\beta_{t}^2\nonumber
			\end{align}
			We claim that $\beta_t < \sqrt{\gamma \eta}$ for $t= O(\log\log \gamma \eta^{-1})$. To verify, note that-
			\begin{align}
			\beta_{t} &\le (2.5\beta_{t})^{2^t}\le (0.25)^{2^t}
			\end{align}	
			Note that $(0.25)^{2^t}\le \sqrt{\gamma \eta}$ for $t=O(\log\log (\gamma \eta)^{-1})$ and hence we stay in this stage for at most $O(\log\log (\gamma \eta)^{-1})$ steps. As $\eta^{-1} = O(d)$, this stage continues for at most $O(\log \log d)$ steps.
			\item Note that in the next step, $\beta_{t} \le\gamma c_{\max} + 1.1\gamma \eta\le 3\gamma \eta$. This is again because $3\gamma k c_{\max} \beta_{t}^2\le 0.1 \beta_t^2$and $\beta_t \le \sqrt{\gamma \eta}$ at the end of the previous stage.
		\end{enumerate}
	\end{proof}
	Hence $\beta_t \le 3\gamma \eta$ for $t=O(\log \log d +\log k)$. By Lemma \ref{induction_global}, $|\hat{a}_{i,t}-c_{i,1}|\le 18\gamma^2 \eta^2, i \ne 1$. Hence $|\hat{a}_{i,t} |\le 2\eta$. By Lemma \ref{error_converge}, the error at convergence satisfies $\norm{A_1 - \hat{A}_{1}} \le 10\gamma k\eta^2$ and the estimate of the weight $\hat{w}_1$ satisfies  $|1- \frac{\hat{w}_1}{w_1}| \le  O(\eta)$.
\end{proof}
\section{Global convergence of the tensor power method for random tensors}

The previous section gives global convergence guarantees for the tensor power method for incoherent tensors. Applying Theorem \ref{global_convergence} to a tensor whose factors are chosen uniformly at random, we can say that the tensor power method converges with random initialization whenever the rank $k=o(d^{0.25})$. Theorem \ref{global_convergence} also proves a linear convergence rate. However, this is quite suboptimal for random tensors. In this section, we use the randomness in the tensor to get much stronger convergence results.  

The techniques used in this section are very different from the rest of the paper. Instead of recursively analyzing the tensor power method updates by showing that the algorithm makes progress at every step by boosting its correlation with some fixed factor, we directly express the correlation of the factors with the estimate $Z_{\tau}$ after a fixed number of $\tau=O(\log \log d)$ time steps in terms of the initial correlations of the factors with the random initialization. This allows us to then skillfully use the randomness in the factors to get strong results. The difficulty with the recursive approach is that all the randomness in the tensor is ``lost'' after just one tensor power method update, i.e. the correlations of different factors with the estimate are no longer independent of each other, which makes the analysis much more difficult. 

\randomtensor*
\begin{proof}
	
	Without loss of generality, we will prove convergence to the first factor $A_1$. Let $\tau=5\log \log d^2$. As before, define $a_{i,t}=\langle A_i, Z_t \rangle$ where $Z_t$ is the iterate at time $t$. For the analysis of the tensor power method updates for random tensors we ignore the normalization step of the updates, till the last iteration. It is easy to see that this makes no difference in the analysis, though in practice it is important to normalize after every step to prevent the vectors from becoming too small and causing numerical errors. Recall that the update equations for ${a}_{i,t}$ for any $t$ are-
	\begin{align}
	{a}_{i,t} &={w}_i {a}_{i,t-1}^2+ c_{i,1}w_1a_{1,t-1}^2+ \sum_{j: j \ne \{i,1\}}^{}c_{i,j}{w}_j {a}_{j,t-1}^2 \label{recursion}
	\end{align}
	and the iterate $X_{\tau+1}$ at time $\tau+1$ can be written as
	\begin{align}
	Z_{\tau+1} = {w}_1 {a}_{1,\tau}^2A_1+ \sum_{i\ne 1}^{}{w}_i {a}_{i,\tau}^2A_i \nonumber
	\end{align}
	On expanding $w_1{a}_{1,\tau}^2$ recursively using Eq. \ref{recursion}, one of the terms that appears in the expansion is ${(w_1 a_{1,0})^{2^{\tau}}}/{w_1}$. Define $\alpha_{\tau} = {|w_1 a_{1,0}|^{2^{\tau}}}/{w_1}$. Let $\Delta_{\tau} = (1/\alpha_{\tau})\sum_{i\ne 1}^{}{w}_i {a}_{i,\tau}^2A_i$. We show that $\norm{\Delta_{\tau}}\le \tilde{O}(1/\sqrt{d})$ with failure probability at most $\log^{-1} d$.
	We can write $(1/\alpha_{\tau}){{w}_1 {a}_{1,\tau}^2}$, the coefficient for first factor $A_1$ normalized by $\alpha_{\tau}$, as follows
	\begin{align}
	\frac{{w}_1 {a}_{1,\tau}^2}{\alpha_{\tau}}&= \frac{(w_1 a_{1,0})^{2^{\tau}}}{w_1\alpha_{\tau}} + \frac{1}{\alpha_{\tau}}\Big( w_1 a_{1,\tau}^2 - \frac{(w_1 a_{1,0})^{2^{\tau}}}{w_1} \Big)\nonumber\\
	&= 1+\frac{1}{\alpha_{\tau}}\Big( w_1 a_{1,\tau}^2 - \frac{(w_1 a_{1,0})^{2^{\tau}}}{w_1} \Big)\nonumber
	\end{align}
	Let $\lambda_{\tau} = \frac{1}{\alpha_{\tau}}\Big( w_1 a_{1,\tau}^2 - \frac{(w_1 a_{1,0})^{2^{\tau}}}{w_1} \Big)$. Let $Z_{\tau+1}'=Z_{\tau+1}/\alpha_{\tau}$. We can write $Z_{\tau+1}'$ as
	\begin{align}
	Z_{\tau+1}'=(1+\lambda_{\tau})A_1+\Delta_{\tau} \nonumber
	\end{align}
	Note that $\frac{Z_{\tau+1}'}{\norm{Z_{\tau+1}'}}=\frac{Z_{\tau+1}}{\norm{Z_{\tau+1}}}$. Let $\frac{Z_{\tau+1}'}{\norm{Z_{\tau+1}'}}=\tilde{Z}_{\tau+1}$. We desire to show that the residual $\norm{\tilde{Z}_{\tau+1}-A_1}\le \tilde{O}(1/\sqrt{d})$. We can bound $\norm{\tilde{Z}_{\tau+1}-A_1}$ as follows using the triangle inequality,
	\begin{align}
	\norm{\tilde{Z}_{\tau+1}-A_1} \le  \Big|\frac{1+\lambda_{\tau}}{\norm{(1+\lambda_{\tau})A_1+\Delta_{\tau}}}-1\Big| + \frac{\norm{\Delta_{\tau}}}{\norm{(1+\lambda_{\tau})A_1+\Delta_{\tau}}}\nonumber
	\end{align}
	If $\norm{\Delta_{\tau}}\le \tilde{O}(1/\sqrt{d})$ and $|\lambda_{\tau}|\le d^{-\epsilon}$ then,
	\begin{align}
	\norm{\tilde{Z}_{\tau+1}-A_1} &\le \Big|\frac{1}{\norm{A_1+\Delta_{\tau}/(1+\lambda_{\tau})}}-1\Big|+ \frac{\norm{\Delta_{\tau}}}{1-|\lambda_{\tau}|-\norm{\Delta_{\tau}}}\nonumber\\ &\le  \frac{2\norm{\Delta_{\tau}}}{1-|\lambda_{\tau}|}+\frac{\norm{\Delta_{\tau}}}{1-|\lambda_{\tau}|-\norm{\Delta_{\tau}}}\le \tilde{O}(1/\sqrt{d}) \nonumber
	\end{align}
	If $\norm{\tilde{Z}_{\tau+1}-A_1} \le \tilde{O}(1/\sqrt{d})$ then, by Lemma \ref{error_converge}, the estimate of the weight $\hat{w}_1$ satisfies  $|1- \frac{\hat{w}_1}{w_1}| \le  \tilde{O}(1/\sqrt{d})$.	
	
	Hence we will show that $\norm{\Delta_{\tau}}\le \tilde{O}(1/\sqrt{d})$ and $|\lambda_{\tau}|\le d^{-\epsilon}$ with failure probability at most $\log^{-1} d$. Let $\epsilon_{\tau}=\norm{\Delta_{\tau}}^2$. We can write $\epsilon_{\tau}$ as
	\begin{align}
	\epsilon_{\tau}=\norm{\Delta_{\tau}}^2 &= \sum_{i\ne 1, j\ne 1}^{}(1/\alpha_{\tau}^2){w}_i w_j {a}_{i,\tau}^2 {a}_{j,\tau}^2 c_{i,j}\nonumber
	\end{align}	
	
	We can also write $\lambda_{\tau}^2$ as follows-
	\begin{align}
	\lambda_{\tau}^2 = (1/\alpha_{\tau}^2)w_1^2\Big(a_{1,\tau}^2 - \frac{(w_1 a_{1,0})^{2^{\tau}}}{w_1^2}\Big)^2\nonumber
	\end{align}
	
	Note that $\lambda_{\tau}$ has the same form as $\epsilon_{\tau}$ with the restriction that $i=j=1$ and the $\frac{(w_1 a_{1,0})^{2^{\tau}}}{w_1^2}$ in the expansion of $a_{1,\tau}^2$ is removed. 
	
	Our approach will be to recursively expand the ${a}_{i,\tau}^2$ terms to express $\epsilon_{\tau}$ and $\lambda_{\tau}$ only in terms of $a_{i,0}$ (the initial correlations at time 0),  the correlation between factors $c_{i,j}$ and the weights $w_i$. We use the recursion Eq. \ref{recursion} to do this. 
	
	We first consider the expansion of $a_{i,\tau}^2$ for any $i$ using recursion Eq. \ref{recursion}. $a_{i,t}^2$ can be written as a weighted sum of correlations of the factors with the iterate at the $(t-1)$st time step as follows using recursion Eq. \ref{recursion}-
	\begin{align}
	a_{i,t}^2 &= \Big(w_ia_{i,t-1}^2 + \sum_{j \ne i}^{}c_{i,j}w_j a_{j,t-1}^2\Big)^2\\
	%&=\Big(\sum_{j,k}^{}c_{i,j}c_{i,k}w_ja_{j,t-1}^2w_ka_{k,t-1}^2\Big)\\
	&=  \sum_{j, k}^{} w_j w_k c_{i,j}c_{i,k}a_{j,t-1}^2a_{k,t-1}^2
	\end{align}
	Each term in the summation corresponds to two choices for the terms at time 
	$(t-1)$, the $j$ and $k$ variables. Continuing this recursive expansion for $\tau$ time steps, we can represent each monomial in the expansion by a complete binary tree with depth $\tau$. We label a node of the binary tree as $j$ if it corresponds to factor $A_j$. For ease of exposition, we will consider the initialization $Z_0$ as the 0th factor for the graph representation, hence $c_{i,0}=a_{i,0}$. The root of the tree is labeled as $i$ as it corresponds to the factor $A_i$. The descendants of the root $i$ are labelled as $j$ and $k$ if $a_{i,\tau}^2$ is expanded into $a_{j,\tau-1}^2$ and $a_{k,\tau-1}^2$ using recursion Eq. \ref{recursion}. The process is repeated at any step of the recursion, by expanding $a_{j,t}^2$ in terms of $a_{k,t-1}^2$ and $a_{l,t-1}^2$ for some $k$ and $l$. Refer to Fig. \ref{fig:example_tree} for an example of a monomial and it's binary tree representation.
	
	\begin{figure}
		\centering
		\includegraphics[width = 3 in]{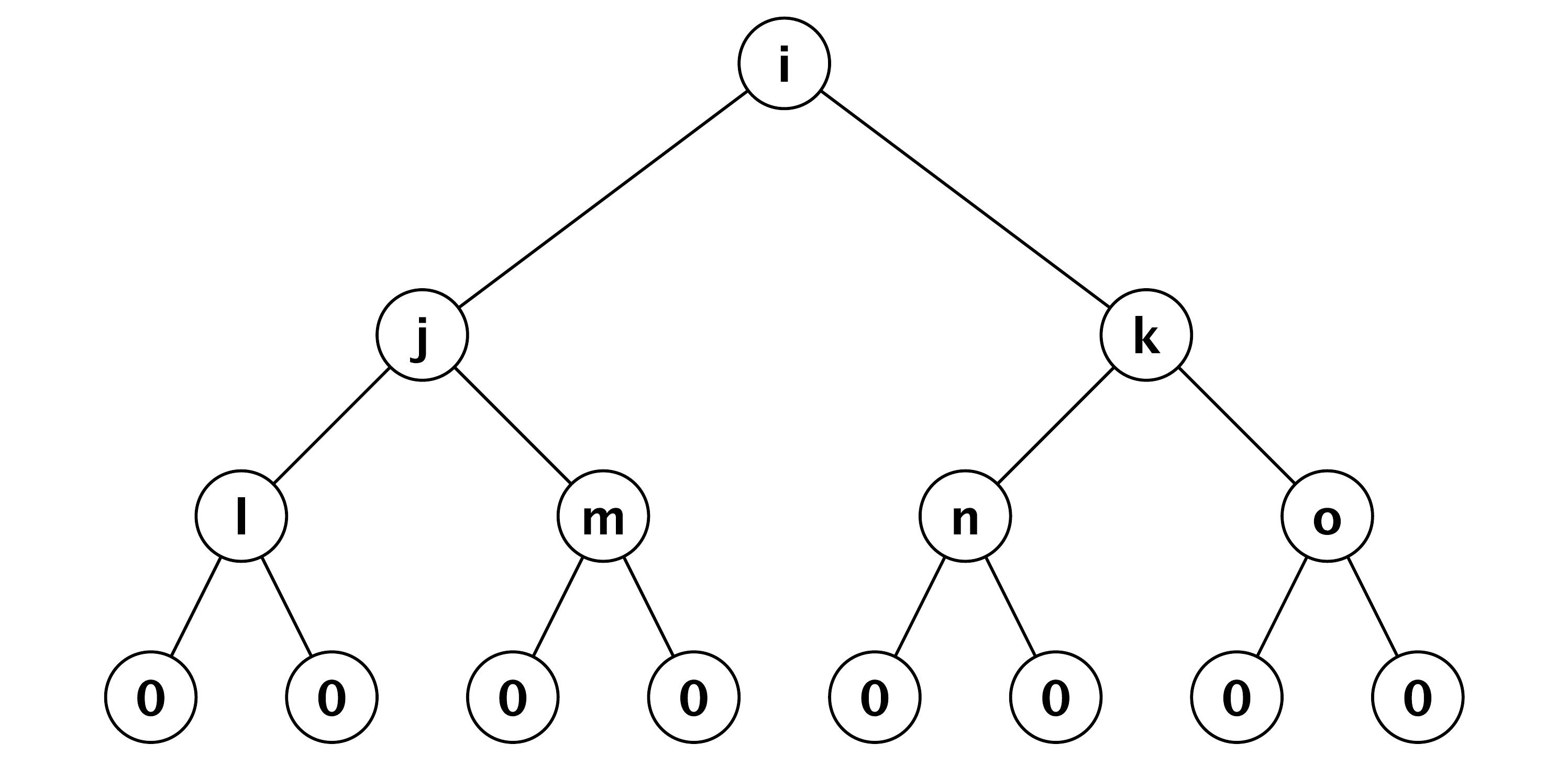}
		\caption{Example of a monomial in the expansion of $w_i a_{i,2}^2$ represented as a binary tree. The monomial is $c_{i,j}c_{i,j}c_{j,l}c_{j,m}c_{i,k}c_{k,n}c_{k,o}w_i w_j w_l w_m w_k w_n w_o a_{l,0}^2 a_{m,0}^2 a_{n,0}^2 a_{o,0}^2$}
		\label{fig:example_tree}
	\end{figure}
	
	Given any complete binary tree $B$, the monomial associated with the tree can be written down recursively. We write down the procedure for finding the monomial corresponding to a binary tree $B$ explicitly in Algorithm \ref{alg:recursion} for clarity.
	
	\begin{algorithm}\caption{Finding monomial $f$ from binary tree $B$}\label{alg:recursion}
		\textbf{Input:} Binary tree $B$, root $u$\\
		\verb|monomial|($B,u$)
		\begin{algorithmic}[1]
			%\State Set $f=\phi$
			\WHILE {$u$ is not a leaf}
			\STATE Set $i$ to be the factor corresponding to $u$
			\STATE Set $v$ to be the left child of $u$, set $j$ to be the factor corresponding to $v$
			\STATE Set $w$ to be the right child of $u$, set $j$ to be the factor corresponding to $w$
			\STATE $f=fw_i$
			\STATE  $f=fc_{i,j}\verb|monomial|(B,v)$
			\STATE $f=fc_{i,j} \verb|monomial|(B,w)$
			\ENDWHILE 
			\STATE \textbf{return} $f$
		\end{algorithmic}
	\end{algorithm}

	Therefore, by successively using Eq. \ref{recursion}, we expand $w_i a_{i,1}^2$ in terms of the correlations of the factors with the random initialization $Z_0$ (the $a_{j,0}^2$ factors) and define a complete binary tree $B_f$ for every monomial $f$ in the expansion. We also define a graph $G_f$ for the monomial $f$ by coalescing nodes of the binary tree having the same label and removing self-loops. We allow more than one edge between two nodes. %Note that \ref{recursion}, ${G_f}$ is a connected graph.
	
	For any monomial $f$ in the expansion of $(1/\alpha_{\tau}^2){w}_i w_j {a}_{i,\tau}^2 {a}_{j,\tau}^2 c_{i,j}$ in $\epsilon_{t}$, we construct two binary trees corresponding to the expansion of ${w}_i {a}_{i,\tau}^2$ and $w_j {a}_{j,\tau}^2$. We construct the graph $G_f$ by adding an edge between the roots of the two binary trees (this corresponds to the $c_{i,j}$ term) and then coalescing nodes of the new graph having the same label and removing self-loops, while allowing multiple edges between two nodes. The same procedure is followed for the expansion of $\lambda_{\tau}^2$, with the difference that now $i=j=1$, and the $\frac{(w_1 a_{1,0})^{2^{\tau}}}{w_1^2}$ term in the expansion of $a_{1,\tau}^2$ is removed. 
	
	\subsection{Choosing a suitable basis for the factors}
	
	Without loss of generality, assume that the first $(n-1)$ factors are present in $G_f$, for some $n$. The $(n-1)$ vectors corresponding to the $(n-1)$ factors and the initialization $Z_0$ span a $n$ dimensional subspace. We will choose a particular basis $\{v_i\}, i\in [n]$ for the $n$ dimensional subspace and express the factors with respect to that basis. $v_1=Z_0$, and $v_i$ is unit vector along the projection of $A_{i-1}$ orthogonal to $\{A_j, j < i-1\}$. In terms of this basis, $Z_0=(1,0,\dotsb, 0)$. Let the 1st factor $A_1$ have component $x_{1,1}$ along the first coordinate axis and $u_{1,2}$ along the second coordinate axis. Note that $x_{1,1}$ is distributed as $\frac{\tilde{x}_{1,1}}{r_1}$ and $u_{1,2}$ is distributed as $\frac{\tilde{u}_{1,2}}{r_1}$ where $\tilde{x}_{1,1} \sim N(0,1/{d})$, $\tilde{u}_{1,2} \sim v_1\sqrt{\sum_{i=2}^{d}\tilde{y}_{1,i}^2}$ and $r_1=\sqrt{\tilde{x}_{1,1}^2+\tilde{u}_{1,2}^2}$. Here $\tilde{y}_{1,i} \sim N(0,1/{d})$ and $v_1$ is uniform on $\{-1,+1\}$. Similarly, the 2nd factor $A_2$ has components $(x_{2,1},x_{2,2},u_{2,3})\sim \Big(\frac{\tilde{x}_{2,1}}{r_2},\frac{\tilde{x}_{2,2}}{r_2},\frac{\tilde{u}_{2,3}}{r_2}\Big)$ along the first three coordinate axes. Here $\tilde{x}_{2,1}, \tilde{x}_{2,2} \sim N(0,1/{d})$ and $\tilde{u}_{2,3} \sim v_2\sqrt{\sum_{i=3}^{d}\tilde{y}_{2,i}^2}$ and $r_2=\sqrt{\tilde{x}_{2,1}^2+\tilde{x}_{2,2}^2+\tilde{u}_{2,3}^2}$, where $\tilde{y}_{2,i} \sim N(0,1/{d})$ and $v_2$ is uniform on $\{-1,+1\}$. We continue this projection for all subsequent factors. 
	
	We first prove a Lemma that bounds the magnitude of the projection of any factor along the basis vectors.
	
	\begin{lemma}The projection of $n$ factors along the basis defined above has the following properties- \label{lem:scaling}
		\begin{enumerate}
			\item $ 1-\frac{1}{d^{0.25}}\le r_i^2 \le 1+\frac{1}{d^{0.25}} \; \forall \; i \in [n]$ with failure probability at most $2ne^{-\sqrt{d}/8}$.
			\item $|\tilde{x}_{i,j}|\le \log^5 d/\sqrt{d}$ for all valid $i,j$ (i.e. for all $j<i, i \in [n]$) with failure probability at most $n(\frac{1}{d})^{\log^8 d}$.
		\end{enumerate}
	\end{lemma}
	\begin{proof}
		The proof relies on basic concentration inequalities. 
		
		\begin{enumerate}
			\item 
		
		Consider the vector $(x_{i,1},\dotsb, u_{i,i+1}, 0\dotsb, 0)$ corresponding to factor $i$. The squared scaling factor $r_i^2$ is distributed as $r_i^2 \sim (\tilde{x}_{i,1}^2+\dotsb+\tilde{x}_{i,i}^2+\tilde{u}_{i,i+1}^2)$, where $\tilde{u}_{i,i+1}^2 \sim \tilde{y}_{i,i+1}^2+ \dotsb+\tilde{y}_{i,d}^2$, the $\tilde{y}_{i,j}$ are independent $N(0,1/d)$ random variables. $r_i^2$ is the sum of squares of independent zero mean Gaussian random variables each having variance $1/d$, and hence $x_i^2=dr_i^2$ is a $\chi^2$ random variable with $d$ degrees of freedom. We use the following tail bound on a $\chi^2$ random variable $x$ with $d$ degrees of freedom (the bound follows from the sub-exponential property of the $\chi^2$ random variable)
		\begin{align}
		   \Pr[|x^2-d|\ge d t] \le 2e^{-dt^2/8} \nonumber
		\end{align}
		Choosing $t=d^{-0.25}$, $\Pr[|x_i^2-d|\ge d^{0.75}] \le 2e^{-\sqrt{d}/8}$. Therefore $\Pr[|r_i-1|\ge d^{-0.25}] \le 2e^{-\sqrt{d}/8}$. By a union bound, $|r_i^2-1| \le  \frac{1}{d^{0.25}}\; \forall \; i \in [n]$ with failure probability at most $2ne^{-\sqrt{d}/8}$.
		\item The bound follows directly from basic Gaussian tail bounds (refer to Eq. \ref{concentration}) and a union bound.
	\end{enumerate}
	\end{proof}
	Note that as $\tau=5\log\log d^2$, the total number of nodes of the binary tree corresponding to a monomial is at most $2^{\tau+1}=2\log^5 d^2$. As each monomial corresponds to two binary trees, the number of number in the graph $G_f$ can be at most $4\log^5 d^2$. Let $N=4\log^5 d^2$. We can now use a union bound to argue that the properties of the factors in Lemma \ref{lem:scaling} hold with high probability for any set of $N$ factors. We define $\beta_0=\max\Big\{\Big|\frac{w_ix_{i,1}}{w_1 x_{1,1}}\Big|,i\ne 1 \Big\}$ and $\beta_t = \beta_0^{2^t}$ for any $t$.
	
	\begin{lemma}\label{lem:bound_coord}
		Consider the projection of any set of $N=4\log^5 d^2$ factors. With failure probability at most $1/d^{\log d}$, $|x_{i,j}|\le 2\log^3 d/\sqrt{d}$ for all valid $i,j$ (i.e. for all $j<i, i \in [N]$). Also, with failure probability at most $1/\log d$, $\beta_0 \le 1-1/\log^4 k$.
	\end{lemma}
	\begin{proof}
		Using Lemma \ref{lem:scaling} and a union bound, $|\tilde{x}_{i,j}|\le (\log d)^5/\sqrt{d}$ for all valid $i,j$ and $|r_i^2-1| \le  \frac{1}{d^{0.25}}\; \forall \; i \in [N]$ with failure probability at most $N(\frac{1}{d})^{\log^8 d}+2Ne^{-\sqrt{d}/8} \le 2N(\frac{1}{d})^{\log^8 d}$. As $x_{i,j}=\tilde{x}_{i,j}/r_i$, therefore $x_{i,j}\le  2\log^5 d/\sqrt{d}$ whenever $\tilde{x}_{i,j}\le \log^5 d/\sqrt{d}$ and $r_i\ge 1-\frac{1}{d^{0.25}}$. Therefore, as the total number of sets of $N$ factors is at most $k^N\le d^N$, by doing a union bound over all possible sets of $N$ factors, $|x_{i,j}|\le 2\log^5 d/\sqrt{d}$ for all valid $i,j$ with failure probability at most $2d^N N(\frac{1}{d})^{\log^8 d}\le 1/d^{\log d}$.
		
		Using Lemma \ref{lem:random_ratio}, with failure probability at most $1/\log d$, $\Big|\frac{w_i\tilde{x}_{i,1}}{w_1 \tilde{x}_{1,1}}\Big| \le 1-1/\log^5 k$ for all $i\ne 1$. As $|r_i^2-1| \le  \frac{1}{d^{0.25}}\; \forall \; i \in [k]$ with failure probability at most $2ke^{-\sqrt{d}/8}$, therefore with failure probability at most $2/\log d$, $\Big|\frac{w_ix_{i,1}}{w_1 x_{1,1}}\Big| \le 1-0.5/\log^5 k$ for all $i\ne 1$. 
	\end{proof}
	
	Let $E$ be the event that for any projection of up to $N$ factors $|x_{i,j}|\le 2\log^3 d/\sqrt{d}$ for all valid $i,j$ (i.e. for all $j<i, i \in [n]$) and $\beta_0 \le 1-1/\log^4 k$. By Lemma \ref{lem:bound_coord}, probability of the event $E$ is at least $(1-3/\log d)$. We condition on the event $E$ for the rest of the proof.
	
	\subsection{Characterizing when the monomial has non-zero expectation}
	
	Let $f_2$ refer to the product of all $a_{i,0}^2$ terms, all the weights $w_i$ for any $i$ appearing in $f_{}$ and $1/\alpha_{\tau}^2$. Let $f_1$ refer to all the terms in $f_{}$ not present in $f_2$, hence $f_{}=f_1f_2$.  Let $G'_f$ be the graph obtained by removing the node corresponding to the initialization $X_0$ and all it's edges from $G_f$. Note that $G'_f$ is a connected graph, as the 0th factors only appears in the leaves of the binary tree.
	
	As the $c_{i,j}$ terms are inner products between the factors, we can write $c_{i,j}$ in terms of the co-ordinates of the vectors $A_i$ and $A_j$, in terms of the basis we described previously. Note that $a_{i,0} = {x}_{i,1}$ hence there is only one term in the inner product $a_{i,0}$. $f_1$ is a product of the cross-correlation terms $c_{i,j}$, hence it can be written as the summation of a product of a choices of coordinate for every $c_{i,j}$ term. Let the terms obtained on rewriting $f_{1}$ in terms of the coordinates of the vectors be $g_i$, hence $f_{1} =\sum_{i=1}^{K}g_i$.
	
	\begin{lemma}\label{lem:euler}
		$f_{}$ has non-zero expectation only if $G_f$ is Eulerian. Also, every term $g_i$ having non-zero expectation corresponds to choosing a split of $G'_f$ into a disjoint union of cycles and then choosing a single coordinate for all inner products $c_{i,j}$ which are part of a particular cycle.
	\end{lemma}
	
	\begin{proof}
		
	We claim that every node in $G_f$ must have even degree for $f_{}$ to have non-zero expectation. To verify, consider any node $j$ which has odd degree. Note that the 0th node corresponding to the initialization $X_0$ always has even degree, hence $j \ne 0$. $\E[f_{}]$ is the expectation of the product of all correlation terms $c_{i,j}$ and $a_{i,0}$ appearing in the monomial. Each inner product $c_{i,j}$ involves a $x_{i,t}$ term or $u_{i,t}$ term for some coordinate $t$. Hence if node $i$ has odd degree, then there is at least some $t$ such that $x_{i,t}$ or $u_{i,t}$ is raised to an odd power. Note that the sign of $x_{i,t}$ or $u_{i,t}$ is an independent zero mean random variable, hence the expectation evaluates to 0 in this case. Hence every node in $G_f$ must have even degree for $f_{}$ to have non-zero expectation. By Euler's theorem every node in a graph has even degree if and only if the graph is Eulerian (there exists a trail in the graph which uses every edge exactly once and returns to its starting point). Also, an Eulerian graph can be written as a disjoint union of cycles (Veblen's theorem).
	
	$G'_f$ is also Eulerian and can be written as a disjoint union of cycles as every node has an even number of edges to node 0 and hence removal of these edges preserves the Eulerian property.
	
	We now prove the second part of the Lemma, that every term $g_i$ having non-zero expectation corresponds to choosing a split of $G'_f$ into a disjoint union of cycles and then choosing a single coordinate for all inner products $c_{i,j}$ which are part of a particular cycle. To verify this, let's start at any node $i$ and consider it's inner product with a neighbor $j$. Say we choose coordinate $t$ for the inner product $c_{i,j}$ which leads to a $x_{i,t}x_{j,t}$ term in $g_i$. To ensure that $g_i$ has non-zero expectation, $x_{j,t}$ must appear in the term an even number of times (as the sign of $x_{j,t}$ is an independent zero mean random variable). Hence the coordinate $t$ must be chosen in the inner product of node $j$ with some neighbor of $j$. By repeating this argument, there must exist a cycle $C$ with node $i$ such that the coordinate $t$ is chosen for all correlation terms in that cycle $C$. We then repeat the process on the graph obtained by removing the edges corresponding to cycle $C$ from $G'_f$. Hence every $g_i$ term having non-zero expectation corresponds to choosing a split of $G'_f$ into a disjoint union of cycles and then choosing a single coordinate for all inner products $c_{i,j}$ which are part of a particular cycle.
\end{proof}
We let $f_1'=\sum_{i:\E[g_i]\ne 0}^{}g_i$ and $f'=f_1'f_2$. We claim that $\E[f]=\E[f']$. Consider any term $g_i$, such that $\E[g_i]= 0$. We claim that $\E[g_if_2]$ also equals 0, hence $\E[f]=\E[f']$. This is because if $g_i$ has zero expectation, then there is some $x_{i,t}$ term raised to an odd power, as otherwise the expectation is non-zero. But, as all terms are raised to an even power in $f_2$, the $x_{i,t}$ term is also raised to an odd power in $g_if_2$, which implies that $\E[g_if_2]=0$. This verifies the claim that $\E[g_if_2]=0$ if $\E[g_i]=0$. 
	
	\subsection{Bounding expected value of monomial}\label{subsec:bounding_exp}
	
	We are now ready to bound the expected value of $f_{}$. Note that $\E[f_{}]=0$ if $G_f$ is not Eulerian. If $G_f$ and hence $G'_f$ are Eulerian, split $G'_f$ into some disjoint union of cycles. Say we split $G'_f$ into $p$ cycles $\{C_1, C_2, \dotsb, C_p\}$ with $m_1, m_2, \dotsb, m_p$ edges. Let $D(C_j)$ refer to the choice of coordinate $D(C_j)$ for cycle $C_j$. Let $g(\cup_j C_j(D(C_j)))$ be the term in the expansion of $f_{}$ corresponding to a split of $G_f$ into cycles $\{C_1, C_2, \dotsb, C_p\}$ and the choice of coordinate $D(C_j)$ for cycle $C_j$. We also define $h(C_j(D(C_j)))$ as the product of terms corresponding to cycle $C_j$ and the choice of coordinate $D(C_j)$ for the cycle $C_j$. Note that $g(\cup_j C_j(D(C_j))) = \Pi_{j=1}^{p}h( C_j(D(C_j)))$. We can write
	\begin{align}
		g(\cup_j C_j(D(C_j)))] &= \Pi_j h_{}( C_j(D(C_j)))\nonumber
	\end{align}
	$h_{}( C_j(D(C_j)))$ is the product of the square of the $D(C_j)$-th co-ordinate of all the factors appearing in the cycle $C_j$. Conditioned on the event $E$, there is only one factor having a component greater than $\log^{5} d/\sqrt{d}$ in absolute value along the $D(C_j)$-th co-ordinate axis, hence 
	\begin{align}
	h_{}( C_j(D(C_j)))\le \frac{(\log^{10} d)^{m_1-1}}{d^{m_1-1}}\nonumber
	\end{align}
	Hence, conditioned on event $E$, we can bound $g(\cup_j C_j(D(C_j)))$ as-
	\begin{align}
	g(\cup_j C_j(D(C_j))) &\le  \frac{(\log^{10} d)^{m/2}}{d^{m-p}}\nonumber
	\end{align}
	Let $c(G_f)$ be the largest $p$ such that $G'_f$ can be decomposed into a union of $p$ disjoint cycles. There can be at most $m/2$ disjoint cycles in $G'_f$ as there are $m$ edges, therefore $c(G_f)\le m/2$. Each edge can be placed in one of the total number of possible cycles, hence the total number of ways of splitting $G'_f$ into a disjoint union of cycles is at most $(m/2)^{m}$. Also, there are $n$ possible choices for a coordinate for each cycle, hence there are at most $n^{(m/2)}$ terms corresponding to the same split of $G'_f$ into a disjoint union of cycles. Hence for any particular monomial $f$, the number of possible $g_i$ terms having non-zero expectation is at most $(m/2)^{m}n^{(m/2)}$. Note that $m\le 2\log^5d^2$ as the graph $G'_f$ is constructed by collapsing the two binary trees corresponding to monomial $f$. Each binary tree has depth $\tau=5\log \log d^2$, hence the number of edges is at most $2\log^5d^2$. Hence the total number of edges in graph $G'_f$ is at most $4\log^5d^2\le \log^6d$. Hence we can bound $\E_{|E}[f_{}]$ as-
	\begin{align}
	 \E_{|E}[f_{}]\le f_{}' \le (m/2)^{m}n^{(m/2)}\frac{(\log^{10} d)^{m/2}}{d^{m-c(G_f)}}f_{2}\le \frac{(\log^{10} d)^{5m/2}}{d^{m-c(G_f)}}f_{2}\label{eq:mono_bound}
	\end{align}
	We will now bound the $f_{2}$ term. Let $\theta=\max\Big\{\Big|\frac{w_ix_{i,1}}{w_1 x_{1,1}}\Big|\Big\}$ over all nodes $i\in G'_f$. Clearly $\theta\le \beta_0$ if node 1 is not in $G$ and is at most 1 otherwise. We will consider the representation of the monomial $f$ as two complete binary trees. Recall that the leaves of the binary tree correspond to the 0th factor. Each pair of leaves having the factor $i$ as their parent corresponds to a $a_{i,0}^2$ term. We will pair every leaf node with it's successor, regarding the binary tree as a binary search tree. Note that the left child of any node has the same node as it's successor. Let the right child of the node with factor $i$ have a node with factor $j$ as it's successor. We group the $w_i$ term due to the successor of the left child and $w_j$ term due to the successor of the right child together with the $a_{i,0}^2$ term. We bound the $w_j w_i a_{i,0}^2$ term by $\gamma (w_i a_{i,0})^2$ whenever $j \ne i $ and by $(w_i a_{i,0})^2$ when $j = i$. If all the edges from the successor to the leaf are self-loops of the form $c_{j,j}$, then $j=i$. Note that the paths of all leaves of a binary tree to their successor are disjoint, hence each cross-correlation term $c_{i,j}, i\ne j$ can lead to at most one leaf with $j\ne i$. The number of cross-correlation terms equals $m$, the number of edges in the graph $G'_f$. Recall that $\alpha_{\tau}=|w_1 a_{1,0}|^{2^{\tau}}/w_1$. Therefore the product of all the $w_i$ and $a_{i,0}^2$ terms normalized by $\alpha_{\tau}$ is at most $ \gamma^m\theta^{2^{\tau}}$.
	
	\begin{framed}
	As an example, consider the monomial $f=w_2^2w_1^4(c_{1,2})^4(a_{1,0})^8$. The binary tree $B_f$ corresponding to $f$ is given in Fig. \ref{random_ex2}. Both binary trees are the same in this case. The graph $G_f$ obtained by coalescing the two binary trees is given in Fig. \ref{random_ex1}.
	\begin{enumerate}
		\item \emph{Projecting factors onto suitable basis:} We can write the initialization $Z_0$ as the vector $(1,0\dotsb, 0)$. We write the factor $A_1$ as $(x_{1,1},u_{1,2},0,\dotsb, 0)$. Similarly, the 2nd factor $A_2$ has components $(x_{2,1},x_{2,2},u_{2,3})$. Using Lemma \ref{lem:bound_coord}, $\max\{|x_{1,1}|,|x_{2,1}|,|x_{2,2}|\}\le \log^{5} d/\sqrt{d}$. 
		\item \emph{Writing expectation of $f$ as product of expectation of cycles:} Let $f_1=(c_{1,2})^4$. Let $f_2=(a_{1,0})^8=(x_{1,0})^8$. $f_{}$ can be expanded by choosing a coordinate corresponding to each $c_{1,2}$ term, and then summing across all choices. Let the terms obtained on rewriting $f_1$ in terms of the co-ordinates of the factors $A_1$ and $A_2$ be $g_i$, hence $f_{} =\sum_{j=1}^{K}g_i$. By Lemma \ref{lem:euler}, every term $g_i$ having non-zero expectation corresponds to choosing a split of $G'_f$ into a disjoint union of cycles and then choosing a single coordinate for all inner products $c_{i,j}$ which are part of a particular cycle. Say we split $G'_f$ into the union of cycles $C_1$ and $C_2$ where $C_1$ and $C_2$ are 2 edge cycles between node 1 and node 2. Say we choose the 2nd coordinate for both the cycles $C_1$ and $C_2$. Following the notation of subsection \ref{subsec:bounding_exp}, $D(C_1)=D(C_2)=2$ and $g(C_1(2)\cup C_2(2))$ is the term in the expansion of $f$ corresponding to split of $G'_f$ into cycles $C_1$ and $C_2$ and then choosing the second coordinate for both cycles. $g(C_1(2)\cup C_2(2))=h(C_1(2))h(c_2(2))=x_{1,2}^4u_{1,2}^4\le \log^{10} d/d^2$, again following the notation of subsection \ref{subsec:bounding_exp}. Recalling the definition of  $c(G_f)$ be the largest $p$ such that $G'_f$ can be decomposed into a union of $p$ disjoint cycles, for our example, $c(G_f)=2$. As each edge can be placed in one of the two possible cycles and there are 4 edges, the total number of ways of splitting $G'_f$ into a disjoint union of cycles is at most $2^4$. There are $2$ possible choices for coordinates for each cycle as we have two factors. Hence we can bound $f'$ and $\E[f]$ as -
		\begin{align}
		\E[f_{}]\le f_{}' \le 2^4 4^2\frac{(\log^{10} d)^{2}}{d^{2}}f_{2}\le \frac{(\log^{10} d)^{10}}{d^{2}}f_{2}
		\end{align}
	\end{enumerate} 
		\begin{center}
			\includegraphics[width=3 in]{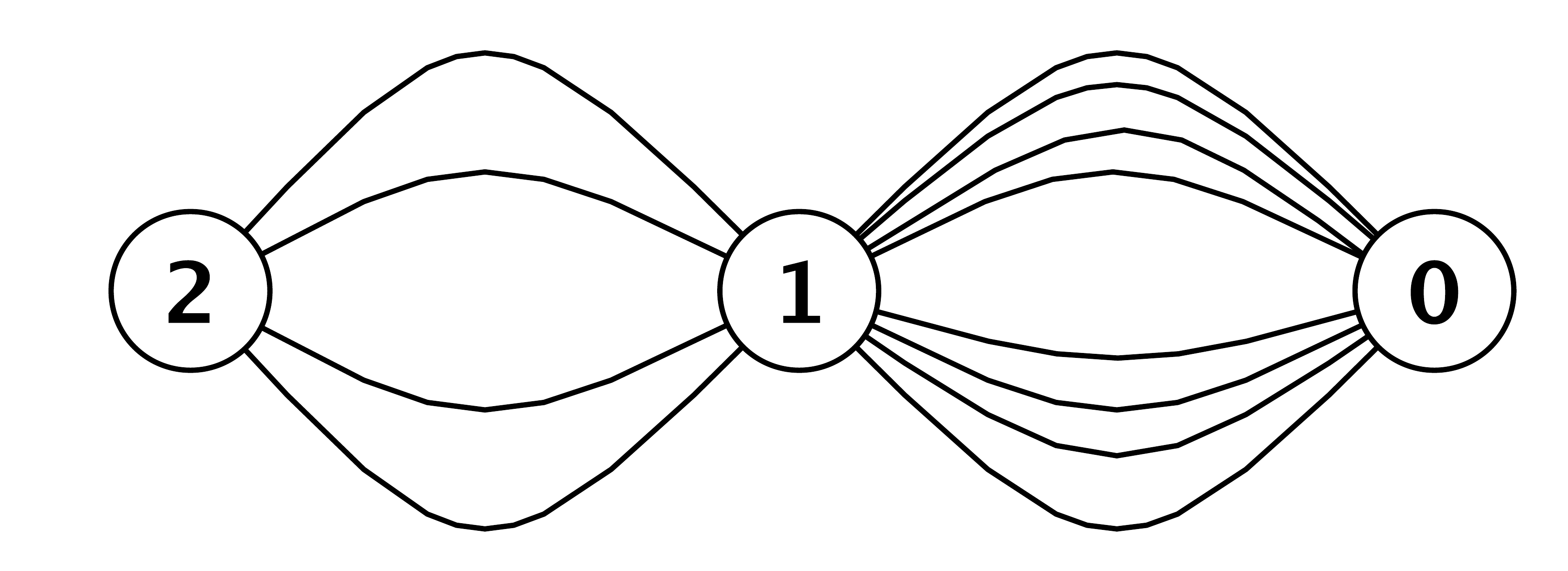}
			\captionof{figure}{Graph $G_f$ for $f=w_2^2w_1^4(c_{1,2})^4(a_{1,0})^8$}
			\label{random_ex1}
					\begin{center}
						\includegraphics[width=3 in]{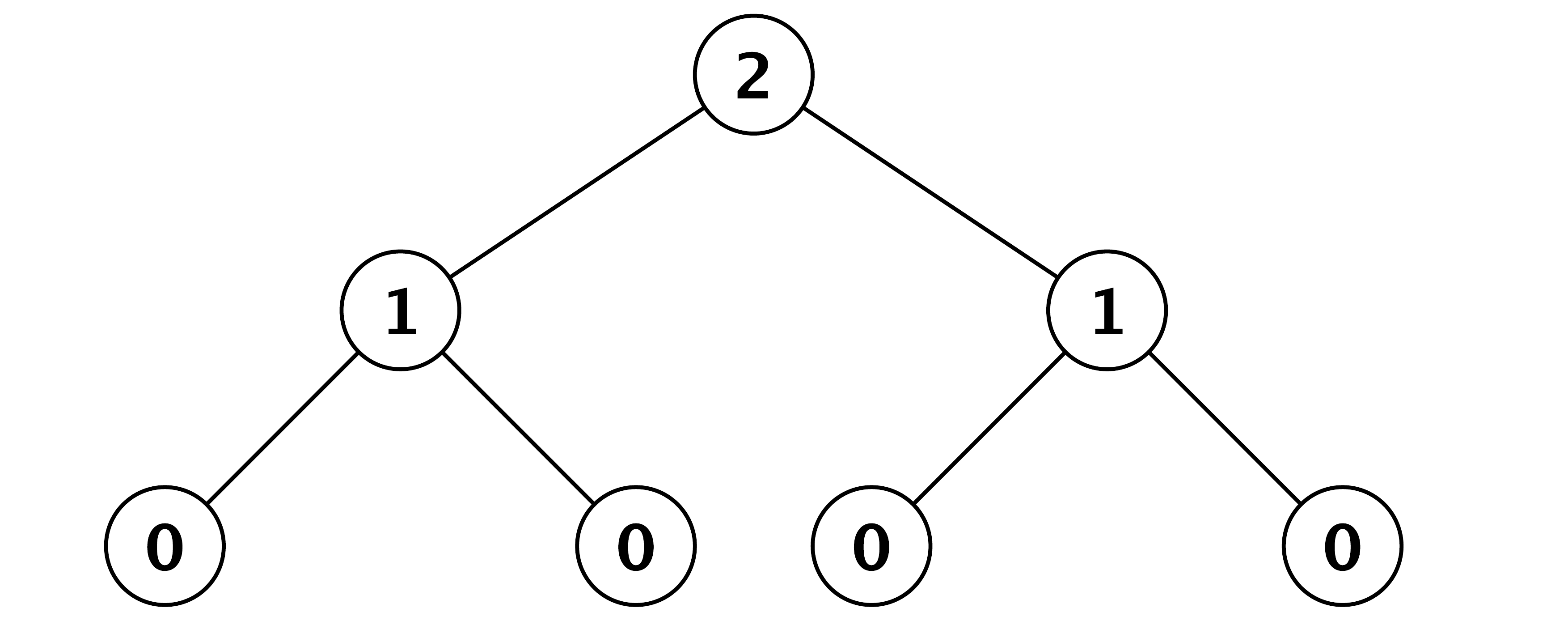}
						\captionof{figure}{Binary tree $B_f$ for $f=w_2^2w_1^4(c_{1,2})^4(a_{1,0})^8$ (both binary trees for $f$ are the same)}
						\label{random_ex2}
					\end{center}
		\end{center}
	\end{framed}
	
	We are now ready to bound $\epsilon_{\tau}$. We will divide $\epsilon_{\tau}$ into 2 sets of monomials and bound each one of them separately-
	
	\begin{enumerate}
		\item All monomials with root nodes $i$ and $j$ and with either no path from node $i$ to node 1 or no path from node $j$ to node 1. We call this set $S_1$.
		\item All monomials with root nodes $i$ and $j$ and at least two paths from node $i$ to node 1 and at least two paths from node $j$ to node 1. We call this set $S_2$.
	\end{enumerate}
	Note that the number of paths between two nodes in the graph $G'_f$ is always even if $f$ has non-zero expectation, as $G'_f$ is Eulerian in that case. We need to relate the number of nodes and edges of an Eulerian graph for the rest of the proof, Lemma \ref{lem:euler_graph} provides a simple bound.
	
		\begin{restatable}{lemma}{eulergraph}\label{lem:euler_graph}
			For any connected Eulerian graph $G$, let $N$ be the number of nodes and $M$ be the number of edges. Consider any decomposition of $G$ into a edge-disjoint set of $p$ cycles. Then, $N\le M-p+1$. Moreover, if $G$ has four edge-disjoint paths between a pair of nodes then $N\le M-p$.
		\end{restatable}  
	
	We first consider the set $S_1$. As there are no paths from node $i$ to node 1 or from node $j$ to node 1, therefore $\theta\le \beta_0$ for at least one of the binary trees. Therefore $f_2\le \gamma^m\beta_0^{2^\tau}=\gamma^m \beta_{\tau}$. For any graph $G_f$ with $n$ nodes, there can be at most $k^n\le d^n$ monomials having a graph isomorphic to $G_f$ as their representation. By Lemma \ref{lem:euler_graph}, $n\le m-c(G_f)+1$. The total number of graphs with $n$ nodes and $m$ edges is be at most $(n^2)^m$. As the graph $G'_f$ is connected, $n \le m$. Note that the number of edges can be at most $4\log^5 d^2 \le \log^6 d$. Hence we can bound the contribution of all monomials in the set $S_1$ as follows-
	\begin{align}
		\sum_{f:f\in S_1}^{}\E[f] &\le \sum_{m=0}^{\log^6 d} k^{m-c(G_f)+1}\frac{(m^2)^{m}(\log^{10}d)^{5m/2}\gamma^m}{d^{m-c(G_f)}}\beta_t\nonumber\\
		&\le \sum_{m=0}^{\log^{6} d}k^{m-c(G_f)+1}\frac{(\gamma \log^{55}d)^{m}}{d^{m-c(G_f)}}\frac{1}{d^2}\nonumber\\
		&\le \frac{1}{d}\sum_{m=0}^{\log^6 d}(\gamma\log^{55}d)^{m}\Big(\frac{k}{d}\Big)^{m-c(G_f)}\nonumber\\
		&\le \frac{1}{d}\sum_{m=0}^{\log^6 d}(\gamma \log^{55}d)^{m}\Big(\frac{1}{d^\epsilon}\Big)^{m/2}\nonumber
%		&\le \frac{1}{d}\sum_{m=0}^{\log^6 d}\Big(\frac{\log^{55}d}{d^{0.5\epsilon}}\Big)^{m}\nonumber\\
		\le \frac{1}{d}\sum_{m=0}^{\infty}\Big(\frac{\gamma\log^{55}d}{d^{0.5\epsilon}}\Big)^{m}\le \frac{2}{d}\nonumber
	\end{align}
	
	We next consider the set $S_2$. For any graph $G_f$ with $n$ nodes with at least one of the nodes corresponding to factor $A_1$, there can be at most $nk^{n-1}\le nd^{n-1}$ monomials having a graph isomorphic to $G_f$ as their representation as there are $n$ possible positions to place the factor $A_1$ and at most $d^{n-1}$ ways to label the remaining nodes. We claim that by Lemma \ref{lem:euler_graph}, $n\le m-c(G_f)$. This is because there are two paths from node $i$ to node 1 and two paths from node $j$ to node 1. Note that there is always an edge between nodes $i$ and $j$, as we connect the roots of the binary trees by an edge. Hence there are at least three edge-disjoint paths between nodes $i$ and $j$. But there cannot be an odd number of edge-disjoint paths between 2 nodes in an Eulerian graph, hence there must be at least four edge-disjoint paths between nodes $i$ and $j$. Hence by Lemma \ref{lem:euler_graph}, $n\le m-c(G_f)$. Also, note that the number of edges $m\ge4$ for monomials in $S_2$ as there are two paths from node $i$ to node 1 and two paths from node $j$ to node 1. Hence we can bound the contribution of all monomials in the set $S_2$ as follows-
	\begin{align}
	\sum_{f:f\in S_2}^{}\E[f] &\le \sum_{m=4}^{\log^6 d} k^{m-c(G_f)-1}\frac{m(m^2)^{m}(\log^{10}d)^{5m/2}\gamma^m}{d^{m-c(G_f)}}\nonumber\\
	&\le \sum_{m=4}^{\log^{6} d}k^{m-c(G_f)-1}\frac{(\gamma\log^{55}d)^{m}}{d^{m-c(G_f)}}\nonumber\\
	&\le \frac{1}{d}\sum_{m=4}^{\log^6 d}(\gamma\log^{55}d)^{m}\Big(\frac{k}{d}\Big)^{m-c(G_f)-1}\nonumber\\
	&\le \frac{1}{d}\sum_{m=4}^{\log^6 d}(\gamma\log^{55}d)^{m}\Big(\frac{1}{d^\epsilon}\Big)^{m/2-1}\nonumber \le \frac{1}{d}\sum_{m=4}^{\log^6 d}(\gamma\log^{55}d)^{m}\Big(\frac{1}{d^\epsilon}\Big)^{m/10}\nonumber\\
	&\le \frac{1}{d}\sum_{m=4}^{\log^6 d}\Big(\frac{\gamma\log^{55}d}{d^{0.1\epsilon}}\Big)^{m}\nonumber\le \frac{1}{d}\sum_{m=4}^{\infty}\Big(\frac{\gamma\log^{55}d}{d^{0.1\epsilon}}\Big)^{m}\le \frac{2}{d}\nonumber
	\end{align}
	
	$\lambda_{\tau}$ is composed of monomials with at least one correlation ($c_{1,i}$) term for $i\ne 1$. Also, all graphs for monomials corresponding to the expansion of $\lambda_{\tau}$ must include a node with label $A_1$. As before, for any graph $G_f$ with $n$ nodes with at least one of the nodes corresponding to factor $A_1$, there can be at most $nk^{n-1}\le nd^{n-1}$ monomials having a graph isomorphic to $G_f$ as their representation. By Lemma \ref{lem:euler_graph}, $n\le m-c(G_f)+1$. Hence we can bound $\lambda_{\tau}$ as follows,
	\begin{align}
	\E[\lambda_{\tau}] &\le \sum_{m=1}^{\log^6 d} k^{m-c(G_f)}\frac{m(m^2)^{m}(\log^{10}d)^{5m/2}\gamma^m}{d^{m-c(G_f)}}\nonumber\\
	&\le \sum_{m=1}^{\log^{6} d}k^{m-c(G_f)}\frac{(\gamma\log^{55}d)^{m}}{d^{m-c(G_f)}}\nonumber\\
	&\le \sum_{m=1}^{\log^6 d}(\gamma\log^{55}d)^{m}\Big(\frac{k}{d}\Big)^{m-c(G_f)}\nonumber\\
	&\le \sum_{m=1}^{\log^6 d}(\gamma\log^{55}d)^{m}\Big(\frac{1}{d^\epsilon}\Big)^{m/2}\nonumber 
	\le \sum_{m=1}^{\log^6 d}\Big(\frac{\gamma\log^{55}d}{d^{0.5\epsilon}}\Big)^{m}\nonumber\\
	&\le \sum_{m=1}^{\infty}\Big(\frac{\gamma\log^{55}d}{d^{0.5\epsilon}}\Big)^{m}\le \frac{1}{d^{\epsilon'}}\nonumber
	\end{align}
	for some $\epsilon'>0$.
	We now use Markov's inequality to get high probability guarantees
	\begin{align}
	\Pr\Big[\epsilon_{\tau}\ge \log d/d\Big] \le {4}/{\log^2d}\nonumber\\
	\Pr\Big[\lambda_{\tau}\ge \log d/d^{\epsilon'}\Big] \le {1}/{\log^2d}\nonumber
	\end{align}
	Hence we have shown that $\norm{\Delta_{\tau}}\le \tilde{O}(1/\sqrt{d})$ and $|\lambda_{\tau}|\le d^{-\epsilon}$ with failure probability at most $\log^{-1} d$.
\end{proof}

\section{Proof of convergence for Orth-ALS}

The proof of convergence of Orth-ALS for incoherent tensors mirrors the proof for orthogonal tensors in Section \ref{sec:orthogonal_tensors}. For clarity, we will try to stick to the proof for the orthogonal case as far possible, while also providing proofs for intermediate Lemmas which were stated without proof in Section \ref{sec:orthogonal_tensors}.

\orthalsconvergence*
\begin{proof}
	Without loss of generality, we assume that the $i$th recovered factor converges to the $i$th true factor. Note that the iterations for the first factor are the usual tensor power method updates and are unaffected by the remaining factors. Hence by Theorem \ref{global_convergence}, Orth-ALS correctly recovers the first factor $O(\log k + \log\log d)$ steps with probability at least $(1-\log^5 k/k^{1+\epsilon})$, for any $\epsilon>0$. 
	
	We now prove that Orth-ALS also recovers the remaining factors. The proof proceeds by induction. We have already shown that the base case is correct and the algorithm recovers the first factor. We next show that if the first $(m-1)$ factors have converged, then the $m$th factor converges in $O(\log k + \log \log d)$ steps with failure probability at most $\tilde{O}(1/{k^{1+\epsilon}})$. The main idea is that as the factors have small correlation with each other, hence orthogonalization does not affect the factors which have not been recovered but ensures that the $m$th estimate never has high correlation with the factors which have already been recovered. Recall that we assume without loss of generality that the $i$th recovered factor $X_i$ converges to the $i$th true factor, hence $X_{i} = A_i + \hat{\Delta}_i$ for $i<m$, where $\norm{ \hat{\Delta}_i} \le 10\gamma k \eta^2$. This is our induction hypothesis, which is true for the base case as by Theorem \ref{global_convergence} the tensor power method updates converge with residual error at most $10\gamma k\eta^2$. 
	
	Let $X_{m,t}$ denote the $m$th factor estimate at time $t$ and let $Y_m$ denote it's value at convergence. We will first calculate the effect of the orthogonalization step on the correlation between the factors and the estimate $X_{m,t}$. Let $\{\bar{X}_i, i<m\}$ denote an orthogonal basis for $\{X_i, i<m\}$. The basis  $\{\bar{X}_i, i<m\}$ is calculated via QR decomposition, and can be written down as follows,
	\begin{align}
	\bar{X}_i = \frac{{X}_{i} - \sum_{j<i}^{}\bar{X}_j^T X_{i} \bar{X}_j}{\norm{{X}_{i} - \sum_{j<i}^{}\bar{X}_j^T X_{i} \bar{X}_j}}\nonumber
	\end{align}
	Note that the estimate $X_{m,t}$ is projected orthogonal to this basis. Define $\bar{X}_{m,t}$ as this orthogonal projection, which can be written down as follows --
	\begin{align}
	\bar{X}_{m,t} = \bar{X}_{m,t} - \sum_{j<m}^{}\bar{X}_j^T X_{m,t} \bar{X}_j\nonumber
	\end{align}
	In the QR decomposition algorithm $\bar{X}_{m,t}$ is also normalized to have unit norm but we will ignore the normalization of $X_{m,t}$ in our analysis because as before we only consider ratios of correlations of the true factors with $\bar{X}_{m,t}$, which is unaffected by normalization. 
	
	We will now analyze the orthogonal basis $\{\bar{X}_i, i<m\}$. The key idea is that the orthogonal basis $\{\bar{X}_i, i<m\}$ is close to the original factors $\{A_i, i<m\}$ as the factors are incoherent. Lemma \ref{orth_basis} proves this claim.

	\orthbasis*
	\begin{proof}
		We argue the result by induction. As the first estimate converges to $A_1+\hat{\Delta}_1$ where $\norm{\hat{\Delta}_1} \le 10\gamma k\eta^2$, the base case is correct. Assume that the result is true for the first $p-1$ vectors in the basis. After orthogonalization, the $p$th basis vector has the following form-
		\begin{align}
		\bar{X}_{p} &= \frac{1}{\kappa}\Big((A_p + \hat{\Delta}_p) -\sum_{j<p}^{}((A_p + \hat{\Delta}_p)^T\bar{X}_j)\bar{X_j}\Big)\nonumber
		\end{align}
		where $\kappa$ is the normalizing factor which ensures $\norm{\bar{X}_{p}}=1$. Define $\mu_{p,j} = (A_p^T(A_j + \Delta_j)$. As $|A_p^T\Delta_j|\le 20\gamma k\eta^2 $ by the induction hypothesis and $|A_p^T A_j| \le \eta$ by definition of $\eta$, $|\mu_{p,j} | \le 2\eta$. Using the induction hypothesis, we can write 
		\begin{align}
		\kappa\bar{X}_{p} &= A_p -\sum_{j<p}^{}\Big(A_p^T(A_j + \Delta_j)\Big)(A_j +\Delta_j) + \hat{\Delta}_p-\sum_{k<p}^{}\Big(\hat{\Delta}_p^T(A_j + \Delta_j)\Big)(A_j +\Delta_j)\nonumber\\
		&= A_p -\sum_{j<p}^{}\mu_{p,j}(A_j +\Delta_j) +\hat{\Delta}_{\epsilon} \nonumber%\label{orth_1}
		\end{align}
		 where $\hat{\Delta}_{\epsilon}=\hat{\Delta}_p-\sum_{k<p}^{}\Big(\hat{\Delta}_p^T(A_j + \Delta_j)\Big)(A_j +\Delta_j)$. As $\hat{\Delta}_{\epsilon}$ is a projection of $\hat{\Delta}_p$ orthogonal to the basis $\{\bar{X}_i, i <p\}$, $\norm{\hat{\Delta}_{\epsilon}}\le \norm{\hat{\Delta}_p}\le 10\gamma k\eta^2$. We can write-
		\begin{align}
		\kappa\bar{X}_{p} &= A_p -\sum_{j<p}^{}\mu_{p,j}A_j -\sum_{j<p}^{}\mu_{p,j}\Delta_j+\hat{\Delta}_{\epsilon}\nonumber\\
		&= A_p + \Delta_{p}'\nonumber
		\end{align}
		where $\Delta_{p}'= -\sum_{j<p}^{}\mu_{p,j}A_j -\sum_{j<p}^{}\mu_{p,j}\Delta_j+\hat{\Delta}_{\epsilon}$. We bound $\norm{\Delta_{p}'}$ as follows-
		\begin{align}
		\norm{\Delta_{p}'} &\le \sum_{j<p}^{}\norm{\mu_{p,j}A_j} +\sum_{j<p}^{}\norm{\mu_{p,j}\Delta_j}+\norm{\hat{\Delta}_{\epsilon}}\nonumber\\
		&\le 2k\eta + 20k^2\eta^2+10\gamma k\eta^2\le 3k\eta \nonumber
		\end{align}
		Note that $\kappa = \norm{A_p + \Delta_{p}'} \implies 1-3k\eta\le \kappa \le 1+3k \eta$ by the triangle inequality. Hence $ 1-3k\eta\le 1/ \kappa\le 1+6k\eta$. Therefore we can rewrite $\bar{X}_{p}$ as-
		\begin{align}
		\bar{X}_{p} &= \frac{1}{\kappa}(A_p + \Delta_{p}')\nonumber\\
		&= A_p + (1-\frac{1}{\kappa})A_p + \frac{1}{\kappa}\Delta_{p}'\nonumber\\
		&= A_p + c_1 A_p + c_2 \Delta_{p}'\nonumber\\
		&= A_p + \Delta_p\nonumber
		\end{align}
		where $c_1 = (1-\frac{1}{\kappa}), c_2=\frac{1}{\kappa}$ and $\Delta_p = c_1 A_p + c_2 \Delta_{p}'$. Note that $|c_1| \le 6k\eta$ and $1-3k\eta \le c_2 \le 1+6k\eta$. Hence $\norm{\Delta_p}\le 10k\eta$. \\
		
		We now show that $|A_i^T \Delta_p| \le 3\eta, i <p$,
		\begin{align}
		\Delta_p &= c_1 A_p +c_2\Big( -\sum_{j<p}^{}\mu_{p,j}A_j -\sum_{j<p}^{}\mu_{p,j}\Delta_j+\hat{\Delta}_{\epsilon}\Big)\nonumber\\
		\implies \Big|A_i^T\Delta_p| &= \Big|c_1 A_i^TA_p\Big|+c_2\Big|\sum_{j<p}^{}\mu_{p,j}A_i^T A_j -\sum_{j<p, j\ne i}^{}\mu_{p,j}A_i^T\Delta_j -\mu_{p,i}A_i^T\Delta_i +A_i^T\hat{\Delta}_{\epsilon}\Big|\nonumber\\
		&\le 6k\eta^2+(1+6k\eta)(2\eta(1+k\eta) +6k\eta^2+20k\eta^2+10\gamma k\eta^2) \nonumber\\
		&\le 3\eta\nonumber
		\end{align}
		Finally, we show that $|A_i^T \Delta_p| \le 20\gamma k\eta^2, i > p$,
		\begin{align}
		\Big|A_i^T\Delta_p| &= c_1\Big| A_i^TA_p\Big|+c_2\Big|\sum_{j<p}^{}\mu_{p,j}A_i^T A_j -\sum_{j<p}^{}\mu_{p,j}A_i^T\Delta_j+A_i^T\hat{\Delta}_{\epsilon}\Big|\nonumber\\
		&\le 6k\eta^2+(1+6k\eta)(2k\eta^2 +40\gamma k^2 \eta^3+10\gamma k \eta^2) \nonumber\\
		&\le 20\gamma k \eta^2\nonumber
		\end{align}
	\end{proof}

Using Lemma \ref{orth_basis}, we will find the effect of orthogonalization on the correlations of the factors with the iterate $X_{m,t}$. At a high level, we need to show that the iterations for the factors $\{A_i, i \ge m\}$ are not much affected by the orthogonalization, while the correlations of the factors $\{A_i, i < m\}$ with the estimate $X_{m,t}$ are ensured to be small. Lemma \ref{orth_basis} is the key tool to prove this, as it shows that the orthogonalized basis is close to the true factors. 

We will now analyze the inner product between $\bar{X}_{m,t}$ and factor $A_i$. This is given by-
\begin{align}
A_i^T \bar{X}_{m,t} = A_i^T X_{m,t} -\sum_{j<m}^{}X_{m,t}^T \bar{X}_jA_i^T \bar{X}_j\nonumber
\end{align}
As earlier, we normalize all the correlations by the correlation of the largest factor, let $\bar{a}_{i,t}$ be the ratio of the correlations of $A_i$ and $A_m$ with the orthogonalized estimate $\bar{X}_{m,t}$ at time $t$. We can write $ \bar{a}_{i,t}$ as-
\begin{align}
\bar{a}_{i,t} = \frac{ A_i^T X_{m,t} -\sum_{j<m}^{} X_{m,t}^T \bar{X}_jA_i^T\bar{X}_j}{ A_m^T X_{m,t} -\sum_{j<m}^{} X_{m,t}^T \bar{X}_jA_m^T\bar{X}_j}\nonumber
\end{align}
We can multiply both sides by $\hat{w}_i$ and substitute $\bar{X}_j$ from Lemma \ref{orth_basis} and then rewrite as follows-
\begin{align}
\hat{w}_i \bar{a}_{i,t} = \frac{ \hat{w}_iA_i^T X_{m,t} -\sum_{j<m}^{} \hat{w}_iX_{m,t}^T (A_j + \Delta_j)A_i^T \Delta_j}{ A_m^T X_{m,t} -\sum_{j<m}^{} X_{m,t}^T (A_j + \Delta_j)A_m^T\Delta_j}\nonumber
\end{align}
We divide the numerator and denominator by $ A_m^TX_{m,t} $ to derive an expression in terms of the ratios of correlations. Let $\delta_{i,t} =  \frac{X_{m,t}^T\Delta_j}{X_{m,t}^T A_m}$. 
\begin{align}
\hat{w}_i\bar{a}_{i,t} = \frac{ \hat{w}_i\hat{a}_{i,t} -\sum_{j<m}^{} ( \hat{w}_i\hat{a}_{j,t}+ \hat{w}_i\delta_{i,t} )A_i^T\Delta_j}{ 1 -\sum_{j<m}^{}  (\hat{a}_{j,t} +\delta_{i,t} )A_m^T \Delta_j}\nonumber
\end{align}
Lemma \ref{bound_delta} upper bounds $\delta_{i,t}$.
		\begin{lemma}\label{bound_delta}
			Let $|\hat{w}_i \hat{a}_{i,t-1}| \le \beta_{t-1} \; \forall \;i \ne m$ and some time $(t-1)$. Also, let $\beta_t \le \gamma \eta + \beta_{t-1}^2$. Then for all $i<m$, $\delta_{i,t} \le 40\gamma k\eta \beta_t$.
		\end{lemma}
		\begin{proof}
			By the power method updates $X_{m,t} = \frac{\sum_i w_i \lambda_i A_i}{\norm{\sum_i w_i \lambda_i A_i}}$ where $\lambda_i = a_{i,t-1}^2$. Note that $\delta_{i,j}$ is normalized by the correlation of the largest factor $A_m$, hence the normalizing factor $\norm{\sum_i w_i \lambda_i A_i}$ does not matter and we will ignore it. We use Lemma \ref{orth_basis} to bound $|A_i^T \Delta_j|$. Hence, 
			\begin{align}
			\Big|\frac{X_{m,t}^T \Delta_j }{X_{m,t}^T A_m}\Big| &\le \frac{\sum_{i}^{}\hat{w}_i\hat{a}_{i,t-1}^2 |A_i^T \Delta_j|}{\sum_{i}^{}\hat{w}_i \hat{a}_{i,t-1}^2 A_i^T A_m}\nonumber\\
			&= \frac{|A_m^T \Delta_j|+ \sum_{i\ne{j,m}}^{}\hat{w}_i\hat{a}_{i,t-1}^2 |A_i^T \Delta_i| +\hat{w}_i \hat{a}_{j,t-1}^2 |A_j^T \Delta_j|}{1 + \sum_{i\ne m}^{}c_{i,m}\hat{w}_i \hat{a}_{i,t-1}^2}\nonumber\\
			&\le \frac{20\gamma k\eta^2 + 3 k\eta\beta_{t-1}^2 + 10  k\eta\beta_{t-1}^2}{1-\gamma k\eta\beta_{t-1}^2}\nonumber\\
			&\le \frac{k\eta(20\gamma\eta+13 \beta_{t-1}^2)}{1-0.5}\nonumber\\
			&\le 40\gamma k\eta\beta_t \nonumber
			\end{align}
			
		\end{proof}
		
	We now need to show $\hat{w}_i\bar{a}_{i,t}$ is small for all $i<m$ and is close to $\hat{w}_i{a}_{i,t}$, the weighted correlation before orthogonalization, for all $i>m$. Lemma \ref{orth_effect} proves this, and shows that the weighted correlation of factors which have not yet been recovered, $\{A_i, i\ge m\}$, is not much affected by orthogonalization, but the factors which have already been recovered. $\{A_i, i< m\}$, are ensured to be small after the orthogonalization step.
		
		\ortheffect*
		\begin{proof}
			
			We can bound $\bar{a}_{i,t}$ for all $i\ge m$ as-
			\begin{align}
			\Big|\hat{w}_i\bar{a}_{i,t}\Big| &\le \frac{ \Big|\hat{w}_i\hat{a}_{i,t}\Big| +\Big| \sum_{j<m}^{} (\hat{w}_i \hat{a}_{j,t}+ \hat{w}_i \delta_{j,t} )A_i^T(A_j + \Delta_j)\Big|}{ 1 -\Big|\sum_{j<m}^{}  (\hat{a}_{j,t} +\delta_{j,t} )A_m^T(A_j + \Delta_j)\Big|}\nonumber\\
			&\le \frac{ \Big|\hat{w}_i\hat{a}_{i,t}\Big| + \sum_{j<m}^{} \Big|(\hat{w}_i \hat{a}_{j,t}+\hat{w}_i \delta_{j,t} )\Big|\Big|A_i^T(A_j + \Delta_j)\Big|}{ 1 -\sum_{j<m}^{}  \Big|(\hat{a}_{j,t} +\delta_{j,t} )\Big| \Big|A_m^T(A_j + \Delta_j)\Big|}\nonumber\\
			&\le \frac{ \Big|\hat{w}_i\hat{a}_{i,t}\Big| + \sum_{j<m}^{} \Big(\Big|\hat{w}_i \hat{a}_{j,t}\Big|+ \Big|\hat{w}_i \delta_{j,t} \Big|\Big)\Big|\Big|A_i^T(A_j + \Delta_j)\Big|}{ 1 -\sum_{j<m}^{}  \Big(\Big|\hat{a}_{j,t}\Big| +\Big|\delta_{j,t} \Big|\Big) \Big|A_m^T(A_j + \Delta_j)\Big|}\nonumber
			\end{align}
			Note that $|\hat{w}_i \hat{a}_{i,t}| \le  \beta_t, |\hat{w}_i \hat{a}_{j,t}| \le \gamma \beta_t$ and $|\hat{w}_i\delta_{j,t}| \le 40\gamma k\eta\beta_t$. Also, $|A_i^T(A_j + \Delta_j)| \le 4\eta$ using Lemma \ref{orth_basis}. Hence we can write,
			\begin{align}
			\Big|\hat{w}_i\bar{a}_{i,t}\Big| &\le  \beta_t\frac{1+8\gamma k\eta}{1-4k\eta\beta_t}\nonumber\\
			&\le  \beta_t(1+8\gamma k\eta)(1+8k\eta\beta_t)\nonumber\\
			&\le  \beta_t(1+20\gamma k\eta) \nonumber\\
			&\le  \beta_t(1+1/k^{1+\epsilon}) \nonumber
			\end{align}
			
			%By the same analysis, we can also show that $|\hat{w}_i\bar{a}_{i,t}| \ge |\hat{w}_i\hat{a}_{i,t}|(1-1/k^{1+\epsilon})$. 
			Similarly, we can bound $\bar{a}_{i,t}$ for all $i<m$ as-
			\begin{align}
			\hat{w}_i\bar{a}_{i,t} &= \frac{ \hat{w}_i\hat{a}_{i,t} -\sum_{j<m}^{} ( \hat{w}_i\hat{a}_{j,t}+ \hat{w}_i \delta_{j,t} )A_i^T(A_j + \Delta_j)}{ 1 -\sum_{j<m}^{}  (\hat{a}_{j,t} + \delta_{j,t} )A_m^T(A_j + \Delta_j)}\nonumber\\
			&= \frac{ \hat{w}_i\hat{a}_{i,t} - ( \hat{w}_i\hat{a}_{i,t}+ \hat{w}_i\delta_{j,t})A_i^T(A_i + \Delta_i) -\sum_{j<m,j\ne i}^{} ( \hat{w}_i\hat{a}_{j,t}+ \hat{w}_i\delta_{j,t})A_i^T(A_j + \Delta_j)}{ 1 -\sum_{j<m }^{}  (\hat{a}_{j,t} +\delta_{j,t})A_m^T(A_j + \Delta_j)}\nonumber\\
			&= \frac{\hat{w}_i\delta_{j,t}A_i^T(A_i + \Delta_i) -\sum_{j<m,j\ne i}^{} ( \hat{w}_i\hat{a}_{j,t}+ \hat{w}_i\delta_{j,t})A_i^T(A_j + \Delta_j)}{ 1 -\sum_{j<m }^{}  (\hat{a}_{j,t} +\delta_{j,t})A_m^T(A_j + \Delta_j)}\nonumber\\
			&\le \frac{ \Big|\hat{w}_i\delta_{j,t}\Big|\Big|A_i^T(A_i + \Delta_i)\Big|  + \sum_{j\ne i}^{} \Big(\Big|\hat{w}_i \hat{a}_{j,t}\Big|+ \Big|\hat{w}_i \delta_{j,t}\Big|\Big)\Big|A_i^T(A_j + \Delta_j)\Big| }{ 1 -\sum_{j,m}^{}  \Big(\Big|\hat{a}_{j,t}\Big| +\Big|\delta_{j,t}\Big|\Big) \Big|A_m^T(A_j + \Delta_j)\Big|}\nonumber\\
			%&\le \frac{O(k\eta\beta_t )}{1-O(k\eta)}\\
			&\le  \frac{40\gamma k\eta\beta_t+8\gamma k\eta\beta_t}{1-4k\eta\beta_t}\nonumber\\
			&\le  50\gamma  k\eta\beta_t\nonumber
			%&\le O(k\eta\beta_t)\\
			%&\le \beta_tO(1/\log^2 k) \label{orth_corr}
			\end{align}
			where we have again used the relations $|\hat{w}_i \hat{a}_{i,t}| \le \beta_t, |\hat{w}_i \hat{a}_{j,t}| \le \gamma \beta_t$, $|\hat{w}_i\delta_{j,t}| \le 40\gamma k\eta\beta_t $ and $|A_i^T(A_j + \Delta_j)| \le 4\eta$.
		\end{proof}
		
		We are now ready to analyze the Orth-ALS updates for the $m$th factor. First, we argue about the initialization step. Lemma \ref{orth_effect} shows that an orthogonalization step performed after the initialization ensures that the factors which have already been recovered have small correlation with the orthogonalized initialization --
		
		\orthini*
			\begin{proof}
				%From Eq. \ref{non_orth_corr} the correlation of all the remaining factors differs by at most a constant multiplicative factor. As the initial vector $X_{m,0}$ is a random vector and $\norm{\Delta_j} \le O(kc_{\max})$, the inner product $X_{m,0}^T \Delta_j$ is less than $kc_{\max}^2$ with high probability. Hence, the correlation of factors $j<m$ with $X_{m,0}$ is shrunk by a factor of $kc_{\max}=O(\log^{-2} k)$. Note that the initial correlations are projections of a random vector onto a fixed direction, and are distributed as Gaussian random variables by the Central Limit Theorem as $d$ is large. The expected value of the maximum of $k$ samples from a $N(0,1)$ distribution is at most $\sqrt{2\log k}$ times hence with high probability no factor $i<m$ will have maximum correlation after the orthogonalization step. Lemma \ref{good_start} can now recursively be applied on all remaining factors, to get the initialization condition.
				
				We first show that $\argmax_i |w_i a_{i,0}| \ge m$. From Lemma \ref{orth_effect}, the ratio of the weighted correlation of all factors $\{A_i,i<m\}$ with the random initialization and the weighted correlation of all factors $\{A_i,i\ge m\}$ with the random initialization is shrunk by a factor of $O(k^{-(1+\epsilon)})$ after the orthogonalization step. Hence no factor $\{A_i,i<m\}$ will have maximum weighted correlation after the orthogonalization step.
				
				Lemma \ref{good_start_whp} can now be applied on all remaining factors, to get the initialization condition. Without loss of generality, assume that $\argmax |w_i a_{i,0}| = m$. Consider the set of factors $\{A_i, m\le i\le n\}$. From Lemma \ref{good_start_whp}, with probability at least ${\Big(1-\frac{\log^5 k}{k^{1+\epsilon}}\Big)}$, $\Big|\frac{{w_i} a_{i,0}}{{w_m} a_{m,0}} \Big| \le {1-5/k^{1+\epsilon}}, \epsilon>0 \; \forall \; i\ne 1$. Applying Lemma \ref{orth_effect} once more, $|\hat{w}_i\bar{a}_{i,t}| \le {\beta_t(1+1/k^{1+\epsilon})}, \; \forall \; i > m$. Therefore combining Lemma \ref{good_start_whp} and Lemma \ref{orth_effect}, with failure probability at most $\Big(1-\frac{\log^5 k}{k^{1+\epsilon}}\Big)$, $\Big|\frac{w_i\bar{a}_{i,0}}{w_m \bar{a}_{m,0}} \Big| \le {1-4/k^{1+\epsilon}} \; \forall \; i\ne m$ after the orthogonalization step.
				
			\end{proof}
		
	%	Now, we combine the effects of the tensor power method step and the orthogonalization step to show that that $X_{m,t}$ converges to $A_m$. Consider a tensor power method step followed by an orthogonalization step. By our previous argument about the convergence of the tensor power method, $|\hat{w}_i \hat{a}_{i,t}| \le \beta_{t-1}^2$ for $i\ne m$, if $|\hat{w}_i \hat{a}_{i,t-1}| \le \beta_{t-1}$ $i\ne m$. Combining the effect of the orthogonalization step via Lemma \ref{orth_effect}, $|\hat{w}_i \hat{a}_{i,t}| \le \beta_{t-1}^2(1+1/k^{1+\epsilon})$ for $i\ne m$ after both the tensor power method and the orthogonalization steps, if $|\hat{w}_i \hat{a}_{i,t-1}| \le \beta_{t-1}$ $i\ne m$. By using Lemma \ref{orth_ini} for the initialization, can now write the updated combined recursion for both the steps-	 
	
		Lemma \ref{orth_ini} shows that with high probability, the initialization for the $m$th recovered factor has the largest  weighted correlation with a factor which has not been recovered so far after the orthogonalization step. It also shows that the separation condition in Lemma \ref{good_start_whp} is satisfied for all remaining factors with probability $(1-\log^5 k/k^{1+\epsilon})$. 
		
		Now, we combine the effects of the tensor power method step and the orthogonalization step for subsequent iterations to show that that $X_{m,t}$ converges to $A_m$. Consider a tensor power method step followed by an orthogonalization step. By Lemma \ref{induction_global}, if $|\hat{w}_i \hat{a}_{i,t-1}| \le \beta_{t-1}$ $i\ne m$ at some time $(t-1)$, then $|\hat{w}_i \hat{a}_{i,t}| \le (\gamma c_{\max} + \beta_{t-1}^2 + 3\gamma kc_{\max}\beta_{t-1}^2)$ for $i\ne m$ after a tensor power method step. Lemma \ref{orth_effect} shows that the correlation of all factors other than the $m$th factors is still small after the orthogonalization step if it was small before. Combining the effect of the orthogonalization step via Lemma \ref{orth_effect}, if $|\hat{w}_i \hat{a}_{i,t-1}| \le \beta_{t-1}$ $i\ne m$ for some time $(t-1)$, then $|\hat{w}_i \hat{a}_{i,t}| \le (\gamma c_{\max} + \beta_{t-1}^2 + 3\gamma kc_{\max}\beta_{t-1}^2)(1+1/k^{1+\epsilon})$ for $i\ne m$ after both the tensor power method and the orthogonalization steps. By also using Lemma \ref{orth_ini} for the initialization, can now write the updated combined recursion analogous to Eq. \ref{beta_1} and Eq. \ref{beta_1}, but which combines the effect of the tensor power method step and the orthogonalization step.	
	\begin{align}
	\beta_{0} &= \max_{i\ne 1}{\Big|{w_i} \hat{a}_{i,0}\Big|} \label{beta_1_orth}\\
	\beta_{t+1} &= (\gamma c_{\max} + \beta_{t}^2 + 3\gamma k c_{\max}\beta_{t}^2)(1+1/k^{1+\epsilon})\label{beta_2_orth}
	\end{align}
	By the previous argument, $|w_i \bar{a}_{i,t}| \le \beta_t$. Note that $\beta_0 \le 1-4/k^{1+\epsilon}$ by Lemma \ref{orth_ini}
	\begin{lemma}\label{converge_orth}
		$\beta_t \le 3\gamma c_{\max} \; \forall \; t\ge O(\log k+\log \log d)$, also $\beta_t < 1-1/k^{1+\epsilon} \;\forall \;t$
	\end{lemma}
	\begin{proof}
		The proof is very similar to the proof for Lemma \ref{iteration_global}. 
		We divide the updates into three stages.
		\begin{enumerate}
			\item $0.1 \le \beta_t \le  {1-4/k^{1+\epsilon}}$: 
			
			As $\beta_t\ge 0.1$ therefore $k \beta_t^2 \ge 1$ in this regime and hence $\gamma c_{\max}\le \gamma k \beta_t^2c_{\max}$, and we can write-
			\begin{align}
			\beta_{t+1} &= (\gamma c_{\max} + \beta_{t}^2 + 3\gamma kc_{\max}\beta_t^2)(1+1/k^{1+\epsilon})\nonumber\\ 
			\beta_{t+1} &\le  (\beta_{t}^2 + 4\gamma k c_{\max}\beta_{t}^2)(1+1/k^{1+\epsilon})\nonumber
			\end{align}
			We claim that $\beta_t < 0.1$ for $t= 2\log k$. To verify, note that-
			\begin{align}
			\beta_{t} &\le (\beta_0 ( 1+4\gamma^2 k c_{\max})(1+1/k^{1+\epsilon}))^{2^t}\nonumber\\
			&\le \Big((1-4/k^{1+\epsilon})(1+1/k^{1+\epsilon})(1+1/k^{1+\epsilon})\Big)^{2^t}\nonumber\\
			&\le \Big(1-1/k^{1+\epsilon}\Big)^{2^t}\nonumber
			\end{align}	
			This follows because $\gamma kc_{\max}\le 1/k^{1+\epsilon}$. Note that $(1-1/k^{1+\epsilon})^{2^t}\le 0.1$ for $t=2\log k$ and hence we stay in this regime for at most $2\log k$ steps. 
			\item $\sqrt{\gamma \eta} \le \beta_t \le 0.1:$
			
			For notational convenience, we restart $t$ from 0 in this stage. Because $\gamma c_{\max}\le \gamma \eta\le \beta_t^2$ in this regime and $3\gamma k \beta_{t}^2c_{\max} \le 0.1 \beta_t^2$ as  $\gamma kc_{\max}\le 1/k^{1+\epsilon}$, we can write-
			\begin{align}
			\beta_{t+1} &=( \gamma c_{\max} + \beta_{t}^2 + 4\gamma kc_{\max \beta_t^2})(1+1/k^{1+\epsilon})\nonumber\\
			&\le( \beta_{t}^2  + \beta_{t}^2 + 0.1\beta_{t}^2 )(1+1/k^{1+\epsilon})\nonumber\\
			&\le 2.5\beta_{t}^2\nonumber
			\end{align}
			We claim that $\beta_t < \sqrt{\gamma \eta}$ for $t= O(\log\log (\gamma \eta)^{-1})$. To verify, note that-
			\begin{align}
			\beta_{t} &\le (2.5(1+O(\log^{-2} k))\beta_{t_1})^{2^t}\nonumber\\
			&\le (0.25)^{2^t}\nonumber
			\end{align}	
			Note that $(0.25)^{2^t}\le \sqrt{\gamma \eta}$ for $t=O(\log\log (\gamma \eta)^{-1})$ and hence we stay in this stage for at most $O(\log\log (\gamma \eta)^{-1})$ steps. As $\eta^{-1} = O(d)$, this stage continues for at most $O(\log\log  d)$ steps.
			\item Note that in the next step, $\beta_{t} \le(\gamma c_{\max} + 1.1\gamma \eta)(1+1/k^{1+\epsilon}) \le 3\gamma \eta$. This is again because $3\gamma^2 k \beta_{t}^2\eta \le 0.1 \beta_t^2$and $\beta_t \le \sqrt{\gamma \eta}$ at the end of the previous stage. %Note that this set is slightly different for Orthogonalized ALS because $\beta_t$ will go slightly to $\beta_t(1+1/k^{1+\epsilon})$ after each orthogonalization step, but one step of ALS will again bring $\beta_t$ back to its previous value at convergence. Hence it's important that the ALS step follow the orthogonalization step.
		\end{enumerate}
	\end{proof}
	%Hence $\beta_t\ \le 3\gamma \eta$ for some $t=O(\log \log d +\log k)$. By Lemma \ref{induction_global}, $|\hat{a}_{i,t}-c_{i,1}|\le 2\gamma\eta^2, i \ne 1$. Hence $|\hat{a}_{i,t} |\le 2\eta$. By Lemma \ref{error_converge}, the error at convergence satisfies $\normsq{A_1 - \hat{x}_{N}} \le 10\gamma^2 k\eta^2$ and the estimate of the weight $\bar{w}_1$ satisfies  $|w_1- \bar{w}_1| \le  O(\eta)$. 	Hence we have shown that if the first $(m-1)$ factors have converged to $X_{i} = A_i + \hat{\Delta}_i$ where $\norm{ \hat{\Delta}_i} \le O(\gamma k \eta^2), \; \forall i <m$ then the $m$th factor converges to $X_{m} = A_m + \hat{\Delta}_m$ where $\norm{ \hat{\Delta}_m} \le O(\gamma k \eta^2)$ in $O(\log k + \log \log d)$ steps with probability at least $1-\frac{\log^2k}{k^{1+0.5 \epsilon}}$. Hence we can now do a union bound to argue that each factor converges with $\ell_2$ error at most $ O(\gamma k \eta^2)$ in $O(\log k + \log \log d)$ with overall failure probability at most $\tilde{O}(1/k^{-\epsilon}), \epsilon>0$. 
	Therefore $\beta_t\ \le 3\gamma \eta$ for some $t=O(\log \log d +\log k)$. By Lemma \ref{induction_global}, $|\hat{a}_{i,t}-c_{i,1}|\le 18\gamma^2 \eta^2, i \ne 1$. Hence $|\hat{a}_{i,t} |\le 2\eta$. By Lemma \ref{error_converge}, the error at convergence satisfies $\normsq{A_m - \hat{A}_{m}} \le 10\gamma k\eta^2$ and the estimate of the weight $\bar{w}_m$ satisfies  $|1- \frac{\hat{w}_m}{w_m}| \le  O(\eta)$.
	
	Hence we have shown that if the first $(m-1)$ factors have converged to $X_{i} = A_i + \hat{\Delta}_i$ where $\norm{ \hat{\Delta}_i} \le 10\gamma k/d^2, \; \forall\; i <m$ then the $m$th factor converges to $X_{m} = A_m + \hat{\Delta}_m$ where $\norm{ \hat{\Delta}_m} \le 10\gamma k/d^2$ in $O(\log k + \log \log d)$ steps with probability at least $\Big(1-\frac{\log^5k}{k^{1+ \epsilon}}\Big)$. This proves the induction hypothesis.
	
	We can now do a union bound to argue that each factor converges with $\ell_2$ error at most $ O(\gamma k/d^2)$ in $O(\log k + \log \log d)$ with overall failure probability at most $\tilde{O}(1/k^{-\epsilon}), \epsilon>0$.
\end{proof}

\section{Proof of additional Lemmas}

In this section, we will prove the initialization condition which we used at several points in the proof of convergence of the tensor power method and Orth-ALS updates. We also provide the proof for a few Lemmas whose proofs were omitted earlier.

\goodstartwhp*
\begin{proof}
	 Without loss of generality, assume $\argmax_i |{w_i} a_{i,0}|=1$. We will first express all factors in terms of a particular choice of orthonormal basis vectors $\{v_i\}, i\in [k]$. $v_1=A_1$, and $v_i$ is unit vector along the projection of $A_i$ orthogonal to $\{A_j\}, j < i$. In terms of this basis, $A_1=(1,0,\dotsb,0)$, let $A_2=(x_{1,2},u_{2,2},0\dotsb, 0)$ and in general $A_i=(x_{i,1},x_{i,2},\dotsb, x_{i,i-1},u_{i,i},0, \dotsb, 0)$. We will show that $|x_{i,j}|\le O(c_{\max})$ for all valid $i,j$ i.e. for all $j<i, i \in [k]$. 
	 
	 We claim that $|x_{i,j}|\le c_{\max}(1+jc_{\max})$ for all valid $i,j$.
	 We prove this via induction on $j$. It is clear that $x_{i,1}\le c_{\max}(1+c_{\max})$ for all valid $i$ as  $\langle A_i, A_1 \rangle \le c_{\max}, i\ne 1$. The induction step is that $x_{i,j}\le c_{\max}(1+pc_{\max})$ for all valid $i$ and $j\le p$. We show that this implies that $x_{i,p+1}\le c_{\max}(1+(p+1)c_{\max})$ for all valid $i$. Note that $|\langle A_i, A_{p+1} \rangle |\le c_{\max} \; \forall \; i\ge p$ therefore, 
	 \begin{align}
	 |u_{p+1,p+1}x_{i,p+1}|&\le c_{\max}+\sum_{i=1}^{p}c_{\max}^2(1+ic_{\max})^2\nonumber\\
	 &\le c_{\max}+c_{\max}^2\sum_{i=1}^{p}(1+4ic_{\max})\nonumber\\
	 &\le c_{\max}+pc_{\max}^2+4p^2c_{\max}^3\nonumber
	 \end{align}
	 From the induction hypothesis, $|u_{p+1,p+1}|\ge 1-2kc_{\max}^2$. This is because $|x_{i,j}|\le c_{\max}(1+jc_{\max}) \le 2c_{\max} \implies \sum_{j}^{}x_{i,j}^2 \le 4 c_{\max}^2 \implies |u_{p+1,p+1}|\ge 1-2kc_{\max}^2$. Hence,
	 \begin{align}
	 |x_{i,p+1}| &\le (c_{\max}+pc_{\max}^2+4p^2c_{\max}^3)(1-2kc_{\max}^2)^{-1}\nonumber \\
	 &\le (c_{\max}+pc_{\max}^2+4p^2c_{\max}^3)(1+4kc_{\max}^2)\nonumber\\
	 &\le c_{\max}+pc_{\max}^2 +4k^2c_{\max}^3+4kc_{\max}^3+4k^2c_{\max}^4+16k^3c_{\max}^5\nonumber\\
	 &\le c_{\max}(1+(p+1)c_{\max})\nonumber
	 \end{align}
	 Therefore $|x_{i,j}|\le c_{\max}(1+jc_{\max})\le 2c_{\max}$ for all valid $i,j$. Let the random initialization be $(t_1/r, t_2/r, \dotsb, t_k/r)$ where $t_i\sim N(0,1/d)$ and $r=\sum_{i}^{}t_i^2$. Let $u_i=w_it_i$. By Lemma \ref{lem:random_ratio} with probability at least ${\Big(1-\frac{10\log^4 k}{k^{1+\epsilon}}\Big)}$, $\Big|\frac{{w_i} t_i}{{w_1} t_1} \Big| \le {1-10/k^{1+\epsilon}}, \epsilon>0 \; \forall \; i\ne 1$. We claim that $\Big|\frac{{w_i} a_{i,0}}{{w_1} a_{1,0}} \Big| \le {1-5/k^{1+\epsilon}}, \epsilon>0 \; \forall \; i\ne 1$ whenever $\Big|\frac{ t_{i}}{ t_{1}} \Big| \le {1-10/k^{1+\epsilon}}, \epsilon>0 \; \forall \; i\ne 1$. This follows because-
	 \begin{align}
	 \frac{w_ia_{i,0}}{w_1a_{1,0}}&=\frac{w_it_iu_{i,i}+\sum_{j<i}^{}x_{i,j}w_jt_j\frac{w_i}{w_j}}{w_1t_1}\nonumber\\
	 \implies \Big|\frac{{w_i} a_{i,0}}{{w_1} a_{1,0}} \Big| &\le \Big|\frac{ w_it_{i}}{ w_1t_{1}} \Big|+\sum_{j<i}2c_{\max}\Big|\frac{ w_it_{i}}{ w_1t_{1}} \Big|\gamma\nonumber\\
	 &\le 1-10/k^{1+\epsilon}+2\gamma kc_{\max}\nonumber\\
	 &\le 1-10/k^{1+\epsilon}+1/k^{1+\epsilon}\nonumber\\
	 &\le 1-5/k^{1+\epsilon}\nonumber
	 \end{align}
\end{proof}

\begin{lemma}\label{lem:random_ratio}
	Let $u_{i}\sim N(0,w_i^2), i \in [k]$ be independent Gaussian random variables. For $\log^4 k \le h \le k^2$, with probability at least $\Big(1-\frac{\log^{4} k}{\h}\Big)$, $\frac{|{w_i} u_{i}|}{\max_i|{w_i} u_{i}|} \le {1-1/h} $ for all $i\ne \argmax_i|w_ia_{i,0}|$.
\end{lemma}
\begin{proof}
	We refer to the pdf of $u_i$ by $f_i(x)$. Without loss of generality, assume $\argmax_i |{w_i} a_{i,0}|=1$. As we are only interested in the ratio of the absolute value of random variables $\{u_i\}$, we will assume without loss of generality that the standard deviations or the weights $w_i$ have been scaled such that $w_i\ge 1$. We will use the following tail bound on the standard Gaussian random variable $x$ (refer to \cite{duembgen2010bounding})-
	\begin{align}
	\frac{e^{-t^2/2}}{\sqrt{2\pi}}\frac{4}{\sqrt{4+t^2}+t}\le\Pr[|x|>t]\le \frac{2e^{-t^2/2}}{t\sqrt{2\pi}}\label{concentration}
	\end{align}
	Let $\kappa$ be some variable which satisfies the following relation--
	\begin{align}
	1-\frac{3\log k}{k} \le \frac{\sum_{i=1}^{k}\Pr[|u_i|\le \kappa]}{k} \le 1-\frac{2\log k}{k} \label{choose_max}
	\end{align}
	Let $m=\max_i u_i$. As the $u_i$ are independent,
	\begin{align}
	\Pr[m>\kappa] &= 1- \Pi_i \Pr[|u_i|\le \kappa] \nonumber
	\end{align}
	By the AM-GM inequality-
	\begin{align}
	\Pi_i \Pr[|u_i|\le \kappa] &\le \Big(\frac{\sum_{i=1}^{k}\Pr[|u_i|\le \kappa]}{k}\Big)^{k}\nonumber\\
	&\le \Big(1-\frac{2\log k}{k}\Big)^k\le \frac{1}{k^2}\nonumber
	\end{align}
	Hence with failure probability at most $1/k^2$ the maximum is at least $\kappa$.
	
	Instead of drawing $k$ samples from the $k$ distributions corresponding to the $k$ factors, we first draw the maximum $m$ from the distribution of the maximum of the $k$ samples, and then draw the remaining samples conditioned on the maximum being $m$. We have shown that $m>\kappa$ with high probability. We condition on the maximum $m>\kappa$. We now show that with high probability no sample lies in the range $[m(1-1/\h), m)]$, given that the maximum is at least $\kappa$. After drawing the maximum from its distribution, we will draw samples from the distributions corresponding to \emph{all} the $k$ factors even though one of the factors would already be the maximum $m$. Clearly this can only increase the probability of a sample lying in the interval $[m(1-1/\h), m]$, and as we only want an upper bound this is fine. Let the conditional pdf of the $i$th random variable $u_i$ conditioned on the maximum being $m$ be $g_{i|m}(x)$. Conditioned on the maximum being $m$, all remaining samples are at most $m$ and hence $g_{i|m}(x) = f_i(x)/\Pr[|u_i|\le m]$ for all $x\le m$ and is 0 otherwise. We will now upper bound $1/\Pr[|u_i|\le m]$. We rely on the following observation about the distribution of a standard Normal random variable $x$-
	\[ 
	\Pr[|x| \le t] \ge  \begin{cases}
		0.5t\quad &t \in [0,1]\\
		0.5\quad &t >1
	\end{cases}
	\]
	The bound for $t \in [0,1]$ follows from the concavity of the Gaussian cumulative distribution function for $t>0$, the bound for $t>1$ is easily verified. Using this, we can write 
	\begin{align}
	\Pr[|u_i|\le m] &\ge 0.5\min\Big\{\frac{m}{w_i},1\Big\}\nonumber
	\end{align}
	We will now find a upper bound on $f_i(m(1-1/k))$, the pdf of the samples at $m(1-1/\h)$. Let $t_i = \frac{m}{w_i}$. Using Eq. \ref{concentration}-
\begin{align}
f_i(m) &\le \frac{\Big(\sqrt{4+t_i^2}+t_i\Big){\Pr[|u_i|\ge m]}}{4w_i}\nonumber\\
\implies g_{i|m}(m) &\le\frac{(2t_i+2){\Pr[|u_i|\ge \kappa]}}{4w_i\Pr[|u_i|\le m] }\nonumber\\
&\le \frac{(2t_i+2){\Pr[|u_i|\ge \kappa]}}{2w_i \min\Big\{\frac{m}{w_i},1\Big\}}\nonumber\\
&\le \frac{(t_i+1){\Pr[|u_i|\ge \kappa]}}{\min\{m, w_i\}}\nonumber\\
\implies \sum_{i: t_i \le \log^{} k}  g_{i|m}(m) &\le \sum_{i: t_i \le \log^{} k} \frac{(t_i+1){\Pr[|u_i|\ge \kappa]}}{m/\log k}\nonumber\\
	&\le  \frac{2\log^{2} k\sum_i\Pr[|u_i|\ge \kappa]}{m} \nonumber \\
	&\le 6 \log^{3}k/m\label{eq:init1}
\end{align}
%where the second to last step follows because $\Pr[|u_i|\ge l] \le e^{-t_i^2/2}$ (using Eq. \ref{concentration}), therefore $\sum_{i: t_i > \log k} {(t_i+1){\Pr[|u_i|\ge l]}} \le \sum_{i: t_i > \log k} {(t_i+1) e^{-t_i^2/2}}\le 1$. We now relate $g_{i|m}(l)$ and ${g_{i|m}(l(1-1/\h))}$. 
where we used Eq. \ref{choose_max} in the last step. We will now relate $g_{i|m}(m)$ and ${g_{i|m}(m(1-1/\h))}$. We can write,
\begin{align}
g_{i|m}(m(1-1/\h))&=  \frac{\exp\Big(-\frac{m^2}{2w_i^2}(1-2/\h+1/k^{2+2\epsilon})\Big)}{\sqrt{2\pi} w_i\Pr[|u_i|\le m]}\nonumber\\
&\le \frac{\exp\Big(-{t_i^2}(1-2/\h)/2\Big)}{\sqrt{2\pi} w_i\Pr[|u_i|\le m]}\nonumber\\
&= \Big(\frac{e^{-t_i^2/2}}{\sqrt{2\pi} w_i\Pr[|u_i|\le m]}\Big)e^{t_i^2/h}\nonumber\\
&= g_{i|m}(m)e^{t_i^2/\h}\nonumber\\
\implies \sum_{i: t_i \le \log^{} k} g_{i|m}(m(1-1/\h))&\le  \sum_{i: t_i \le \log^{} k}g_{i|m}(m)e^{t_i^2/\h}\nonumber\\
&\le 20 \log^{3}k/m \nonumber
\nonumber
\end{align}
where we used Eq. \ref{eq:init1} in the last step. For all $i: t_i > \log^{} k$, we can write,
\begin{align}
g_{i|m}(m(1-1/\h)) &\le \frac{e^{-t_i^2(1/2-1/h)}}{\sqrt{2\pi} w_i\Pr[|u_i|\le m]}\nonumber\\
&\le \frac{2e^{-t_i^2/3}}{\sqrt{2\pi} \min\{m, w_i\}}\nonumber\\
&\le \frac{2e^{-t_i^2/3}}{\sqrt{2\pi} m/t_i}\nonumber\\
&\le \frac{2t_ie^{-t_i^2/3}}{\sqrt{2\pi} m}\le 1/(k^2m)\nonumber
\end{align}
Therefore,
\begin{align}
  \sum_{i: t_i > \log^{} k} g_{i|m}(m(1-1/\h)) &\le 1/(km)\nonumber\\
\implies \sum_{i} g_{i|m}(m(1-1/\h)) &\le 25 \log^{3}k/m\nonumber
\end{align}

%Note that $M>\kappa$ with probability at least $1-1/k^2$ and as $g_{i|m}(x)$ is a decreasing function of $x$, $\sum_i g_{i|m}(m(1-1/\h)) \le 20\log k$ whenever $m>l$. 
Hence the probability of a sample lying in the interval $[m(1-1/\h), m]$ can be bounded by-
\begin{align}
\Pr[\cup (u_i \in [m(1-1/\h), m])] &\le \sum_{i}^{}\Pr[ u_i \in [m(1-1/\h), m]] \nonumber\\
&\le \sum_i g_{i|m}(m(1-1/\h))\frac{m}{\h}\nonumber\\
&= \frac{m}{\h}\sum_i g_{i|m}(m(1-1/\h))\nonumber\\
&\le \frac{25\log^3 k}{\h}\nonumber
%&\le \frac{50\gamma \log^{3} k}{\h}
\end{align}
%where the last step follows as $m<\gamma\sqrt{5\log k}$ with probability at least $1-1/k^2$. 
Hence with probability at least $(1-\frac{1}{k^2})(1-\frac{25 \log^{3} k}{\h})=1-\frac{\log^{4} k}{\h}$ the maximum is greater than $\kappa$ and there are no samples in the interval $[m(1-1/\h), m]$.
\end{proof}

\eulergraph*
\begin{proof}
	Consider any decomposition of $G$ into a disjoint set of $p$ cycles $C_1\cup C_2\dotsb \cup C_p$. We will consider the number of unique nodes in $C_1\cup C_2\dotsb \cup C_t$ for $t\le p$. Let $N(C_{1}\cup C_2\dotsb \cup C_t)$ be the number of unique nodes in $C_1\cup C_2\dotsb \cup C_t$. Similarly, let $M(C_{1}\cup C_2\dotsb \cup C_t)$ be the number of edges in $C_1\cup C_2\dotsb \cup C_t$. Note that $N(C_1)\le M(C_1)$. There must be a cycle in $C_2\dotsb \cup C_p$ with at least one node common to $C_1$ as $G$ is connected. Assume $C_2$ has this property. Then $N(C_1\cup C_2)\le M(C_1)+M(C_2)-1$ as $C_1$ and $C_2$ have one node in common. Repeating this argument, at any stage when we have selected $t$ cycles, there must be a cycle which has not been selected yet but has a node common to the selected cycles. Hence $N(C_{1}\cup C_2\dotsb \cup C_p)\le \sum_{1}^{p}M(C_i)-p+1$.\\
	
	To prove the second part of the Lemma, we claim that if $N=M-p+1$ for some decomposition $\mathcal{C}$ of $G$ into a disjoint set of $p$ cycles, then there cannot be more than two disjoint paths between any pair of nodes. By our previous argument, if $N=M-p+1$ then for any union $\mathcal{S}$ of connected cycles, any cycle not in $S$ can have at most one node common with the nodes in $\mathcal{S}$. Note that the number of disjoint set of paths between any pair $u$ and $v$ must be even as the graph is Eulerian. Assume for the sake of contradiction that there are at least four disjoint paths between two nodes $u$ and $v$. Consider any set of cycles $\mathcal{S}$ in $\mathcal{C}$ which cover any two of the disjoint paths. Say that $P$ is some path which is not covered by $\mathcal{S}$. Then, there must exist nodes $s$ and $t$ such that $s$ and $t$ are present in $\mathcal{C}$ but a segment of the path $P$ is not present in $\mathcal{C}$. We claim that this implies that for some union $\mathcal{S}'$ of cycles such that there is some cycle having two nodes common with $\mathcal{S}'$. To verify this, we simply add cycles to $\mathcal{S}$ to grow our subgraph from node $s$ till it reaches node $t$. At some point, there must be a cycle with two nodes common to the cycles already selected, because we have to reach the node $t$ which has already been included.
	
	%We claim that the path $P$ can be decomposed into ordered set of segments of edges $P_1, P_2, P_n$ such that $P_1 \in C_{i_1}$ for some $C_{i_1}\in R$, $P_2 \in C_{i_2}$ for some $C_{i_2}\in R$ and in general $P_j \in C_{i_j}$ for some $C_{i_j}\in \mathcal{C}$ such that $i_j\ne i_k$ for $j\ne k$. Assume for contradiction that for $j\ne k$ $P_j \in C$ and $P_k \in C$ for $C\in \mathcal{C}$. Then, there exist nodes $s,t \in C$ such that there is a path disjoint from $C$ which connects $s$ and $t$.  
	
	%To prove the second part of the Lemma, we claim that if there are four unique paths between a pair of nodes, then for some stage $l$ when we have selected $l$ cycles, there must be a cycle which has not been selected yet but has two node common to the selected cycles. Additionally, at every stage $t$ when we have selected $t$ cycles, there must be a cycle which has not been selected yet but has a node common to the selected cycles. It is clear that $N\le M-p$ follows from the claim. The second part of the claim is identical to what we showed in the first part of Lemma \ref{lem:euler_graph}. We prove the first part of the claim by contradiction. 
\end{proof}

\errorconverge*
\begin{proof}
	Consider any step $\tau$ of the power iterations at the end of which $|\hat{a}_{i,\tau}|\le 2\eta\;\forall\; i\ne 1$. Let the first (largest) factor have true correlation $a_{1,\tau}$ with the iterate at this time step. Consider the next tensor power method update. From the update formula, the result ${Z}_{\tau+1}$ is-
	\begin{align}
	{Z}_{\tau+1} = \frac{\sum_{i=1}^{k}w_1\hat{w}_ia_{1,\tau}^2\hat{a}_{i,\tau}^2A_{i}}{\norm{\sum_{i=1}^{k}w_1\hat{w}_ia_{1,\tau}^2\hat{a}_{i,\tau}^2A_{i}}}\nonumber
	\end{align}
	Let $\kappa = \frac{w_1 a_{1,\tau}^2}{\norm{\sum_{i=1}^{k}w_1\hat{w}_ia_{1,\tau}^2\hat{a}_{i,\tau}^2A_{i}}}$. Hence the estimate at the end of the $m$th iteration is-
	\begin{align}
	{X}_{\tau+1} &= \kappa\sum_{i=1}^{k}\hat{w}_i\hat{a}_{i,\tau}^2A_{i}\nonumber\\
	&= \kappa(A_1+\sum_{i\ne 1}^{}\hat{w}_i\hat{a}_{i,\tau}^2A_{i})\nonumber
	\end{align}
	Denote $\sum_{i\ne 1}^{}\hat{w}_i\hat{a}_{i,\tau}^2A_{i} = \hat{\Delta}_1$. Note that  $\norm{\hat{\Delta}_1} \le {4\gamma k \eta^2}$ as $|\hat{a}_{i,\tau}| \le 2 \eta \implies \hat{w}_i\hat{a}_{i,\tau}^2 \le 4\gamma \eta^2 $ and the factors $A_i$ have unit norm. As $\norm{{X}_{\tau+1}}=1, \kappa = 1/\norm{A_1+\hat{\Delta}_1}$. From the triangle inequality,
	\begin{align}
	& 1-\norm{\hat{\Delta}_1}\le \norm{A_1+\hat{\Delta}_1} \le 1+ \norm{\hat{\Delta}_1}\nonumber\\
	\implies & 1-{4\gamma k \eta^2}\le \norm{A_1+\hat{\Delta}_1} \le 1+ {4\gamma k \eta^2}\nonumber\\
	\implies & \frac{1}{1+{4\gamma k \eta^2}} \le \kappa \le  \frac{1}{1-{4\gamma k \eta^2}} \nonumber\\
	\implies & {1-{4\gamma k \eta^2}} \le \kappa \le  {1+{5\gamma k \eta^2}} \nonumber
	\end{align}
	
	We can now write the error  $\norm{A_1-{X}_{\tau+1}} $ as-
	\begin{align}
	\norm{A_1-{X}_{\tau+1}} &=\norm{A_1- \kappa(A_1 + \hat{\Delta}_1)} \nonumber\\
	&= \norm{A_1(1-\kappa) + \kappa\hat{\Delta}_1}\nonumber\\
	&\le \norm{A_1(1-\kappa)} + \kappa\norm{\hat{\Delta}_1}\nonumber\\
	&\le {5\gamma k \eta^2}+ {4\gamma k \eta^2}(1+{5\gamma k \eta^2})\nonumber\\
	&\le {10\gamma k \eta^2}\nonumber
	\end{align}
	We also show that the error in estimating the weight $w_1$ of factor $A_1$ is small once we have good estimate of the factor.
	\begin{align}
	\bar{w}_1 &= \sum_{i}^{}w_i \langle\hat{A}_1, A_i\rangle^3\nonumber\\
	&= \sum_{i}^{}w_i \langle A_1 + \hat{\Delta}_1, A_i\rangle^3\nonumber\\
	\implies |w_1- \bar{w}_1| &\le \Big|w_1 \langle A_1 + \hat{\Delta}_1, A_1\rangle^3 -w_1\Big| +\Big|\sum_{i\ne 1}^{}w_i \langle A_1 + \hat{\Delta}_1, A_i\rangle^3\Big|\nonumber\\
	&\le 3w_1 \eta+ 8\sum_{i}^{}w_i \eta^3 \label{weight}\\
	&\le 3w_1 \eta+ 8w_i k\eta^3\le 4w_1 \eta\nonumber\\
	\implies \Big|1- \frac{\bar{w}_1}{w_1}\Big| &\le  O(\eta)\nonumber
	\end{align}
	where Eq. \ref{weight} follows as $|A_i^T \hat{\Delta}_1 |\le \norm{\hat{\Delta}_1} \le \eta$.
\end{proof}
\end{document}